\documentclass[letterpaper, 10 pt, conference]{ieeeconf}  

\IEEEoverridecommandlockouts                              
\usepackage{times}
\usepackage[usenames,dvips]{color}
\usepackage{amsmath}
\usepackage{amssymb}
\usepackage{amsfonts}
\usepackage[font={small}]{caption}
\usepackage{subcaption}
\usepackage{graphicx}
\usepackage{multicol}
\usepackage{multirow}
\usepackage{bm}

\def\vec#1{{\bf #1}}
\def\S{{\mathbb{S}}}
\def\S2{{\mathbb{S}^2}}
\def\S3{{\mathbb{S}^3}}

\def\bingham{{\rm Bingham}}

\usepackage{amsthm}

\theoremstyle{definition}
\newtheorem{claim}{Claim}

\title{\LARGE \bf
Bingham Procrustean Alignment for Object Detection in Clutter
}

\author{Jared Glover and Sanja Popovic
\thanks{Computer Science and Artificial Intelligence Laboratory,
Massachusetts Institute of Technology, Cambridge, MA 02139
\tt\small\{jglov,sanja\}@mit.edu}}

\renewcommand{\baselinestretch}{.98}

\begin{document}

\maketitle
\thispagestyle{empty}
\pagestyle{empty}

\begin{abstract}

A new system for object detection in cluttered RGB-D images is presented.
Our main contribution is a new method called Bingham Procrustean Alignment (BPA)
to align models with the scene.  BPA uses point correspondences between oriented features to
derive a probability distribution over possible model poses.  The orientation component of
this distribution, conditioned on the position, is shown to be a Bingham distribution.
This result also applies to the classic problem of least-squares alignment of point sets,
when point features are orientation-less, and gives a principled, probabilistic way to measure
pose uncertainty in the rigid alignment problem.  Our detection system leverages BPA to achieve
more reliable object detections in clutter.

\end{abstract}

\section{Introduction}

Detecting known, rigid objects in RGB-D images relies on being able to align 3-D object models with an observed scene.
If alignments are inconsistent or inaccurate, detection rates will suffer.  In noisy and cluttered scenes (such as shown
in figure~\ref{fig:scope_samples1}), good alignments must rely on multiple cues, such as 3-D point positions, surface
normals, curvature directions, edges, and image features.  Yet there is no existing alignment method (other than brute
force optimization) that can fuse all of this information together in a meaningful way.

The Bingham distribution\footnote{See the appendix for a brief overview.} has recently been shown to be useful for fusing
orientation information for 3-D object detection~\cite{Glover11}.  In this paper, we derive a surprising result connecting
the Bingham distribution to the classical least-squares alignment problem, which allows our new system to easily fuse
information from both position and orientation information in a principled, Bayesian alignment system which we call
Bingham Procrustean Alignment (BPA).

\subsection{Background}

Rigid alignment 
of two 3-D point sets $X$ and $Y$ is a well-studied
problem---one seeks an optimal (quaternion) rotation $\vec{q}$ and
translation vector $\vec{t}$ to minimize an alignment cost function, such as sum of squared errors between
corresponding points on $X$ and $Y$.  Given known correspondences, $\vec{t}$ and $\vec{q}$ can be found in closed form with Horn's method~\cite{horn1987closed}.
If correspondences are unknown, the alignment cost function can be specified in terms of sum-of-squared distances between nearest-neighbor
points on $X$ and $Y$, and iterative algorithms like ICP (Iterative Closest Point) are guaranteed to reach a local minimum of the cost
function~\cite{besl_icp_1992}.  During each iteration of ICP, Horn's method is used to solve for an optimal $\vec{t}$ and $\vec{q}$
given a current set of correspondences, and then the correspondences are updated using nearest neighbors given the new pose.

\begin{figure}[t!]
  \centering
  \includegraphics[height=1.2in]{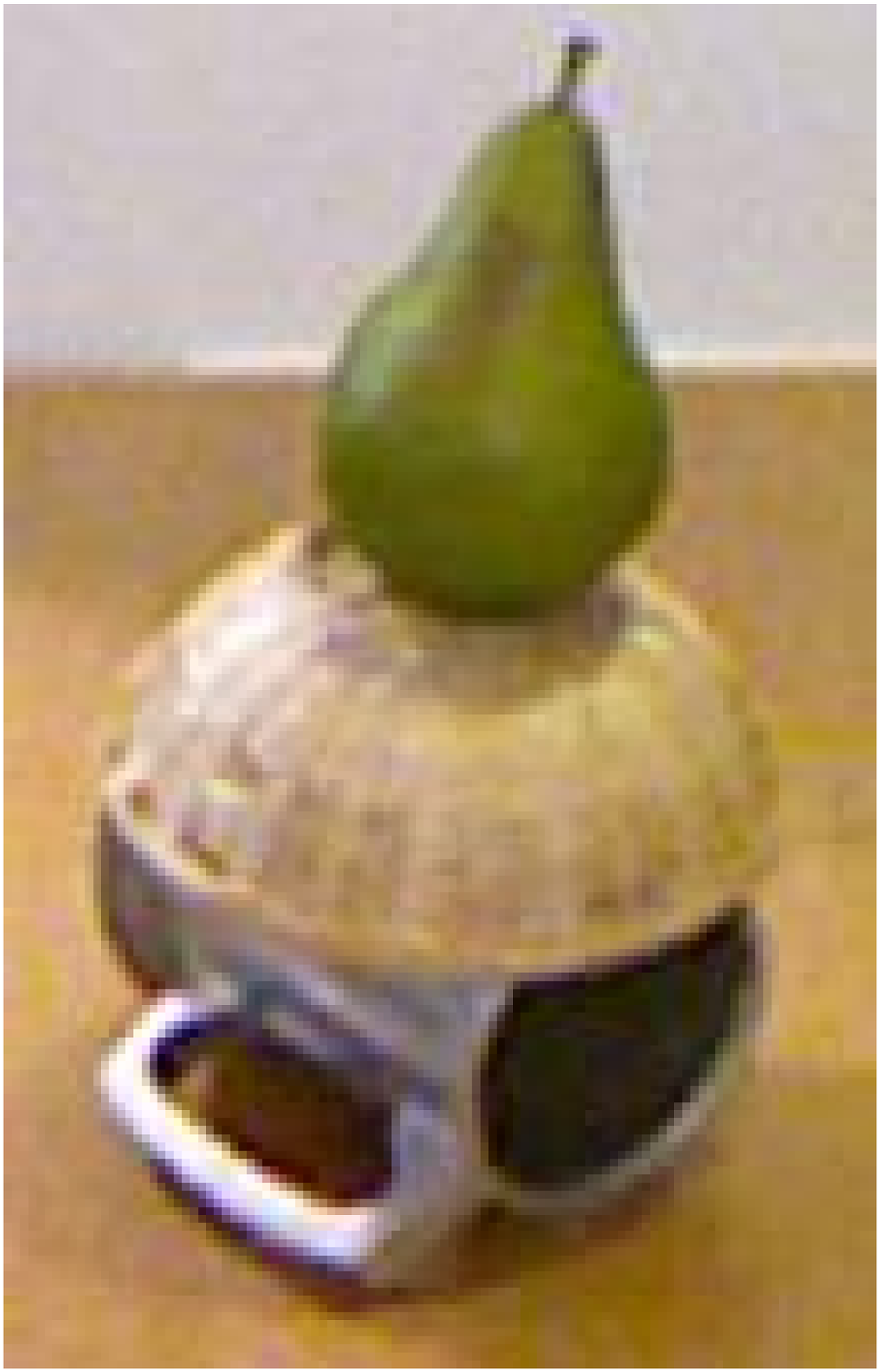}%
  \hfill
  \includegraphics[height=1.2in]{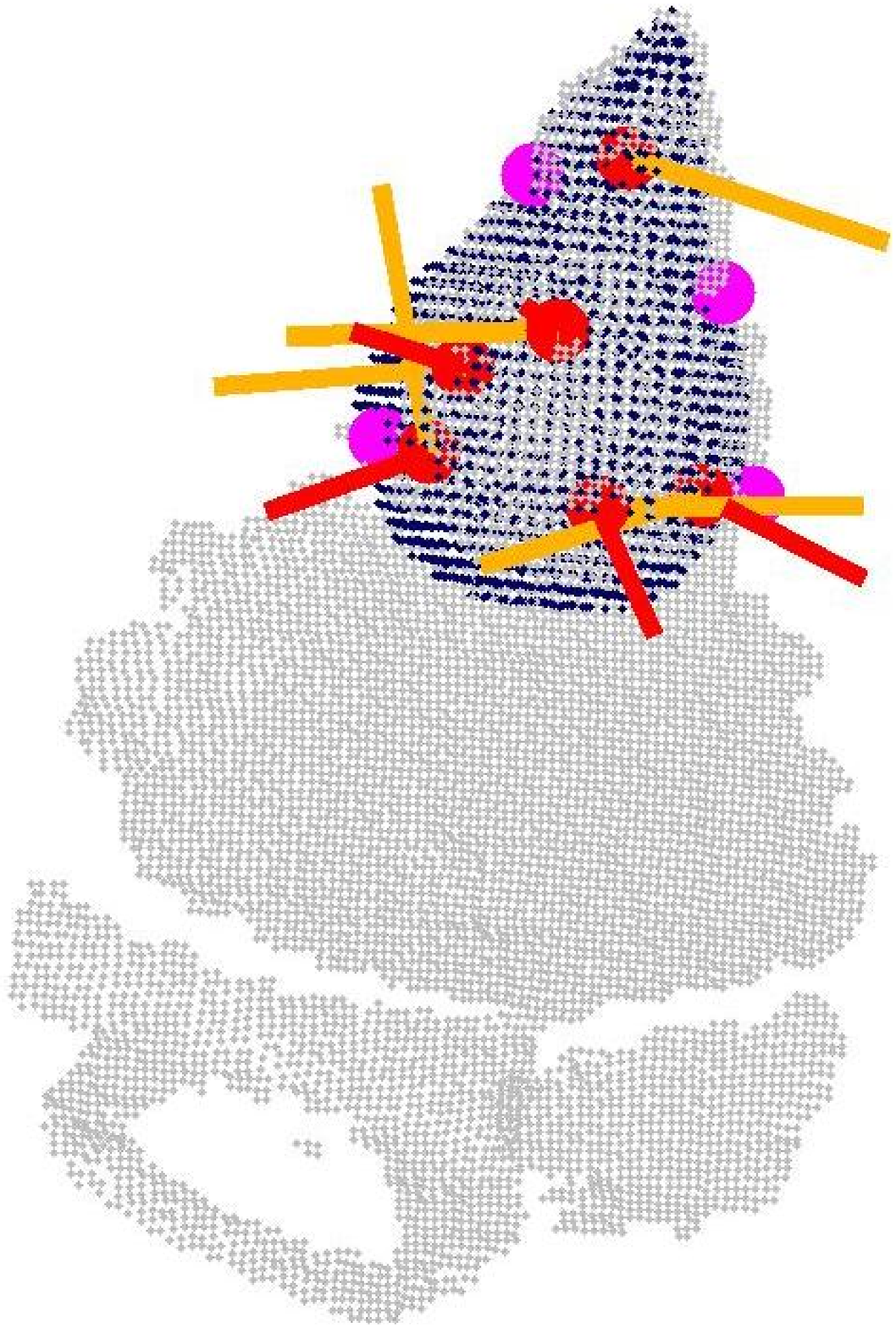}%
  \hfill
  \includegraphics[height=1.2in]{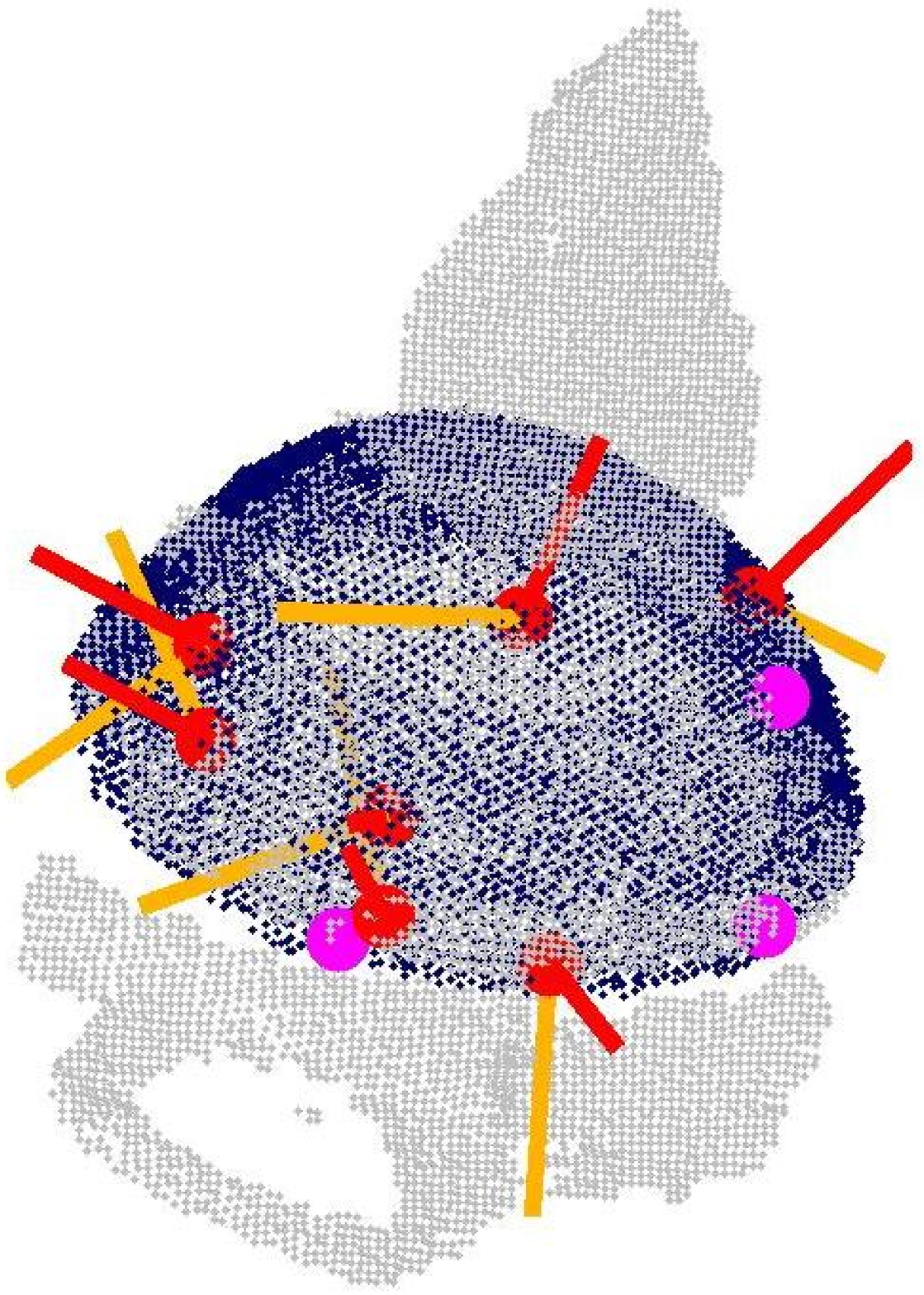}%
  \hfill
  \includegraphics[height=1.2in]{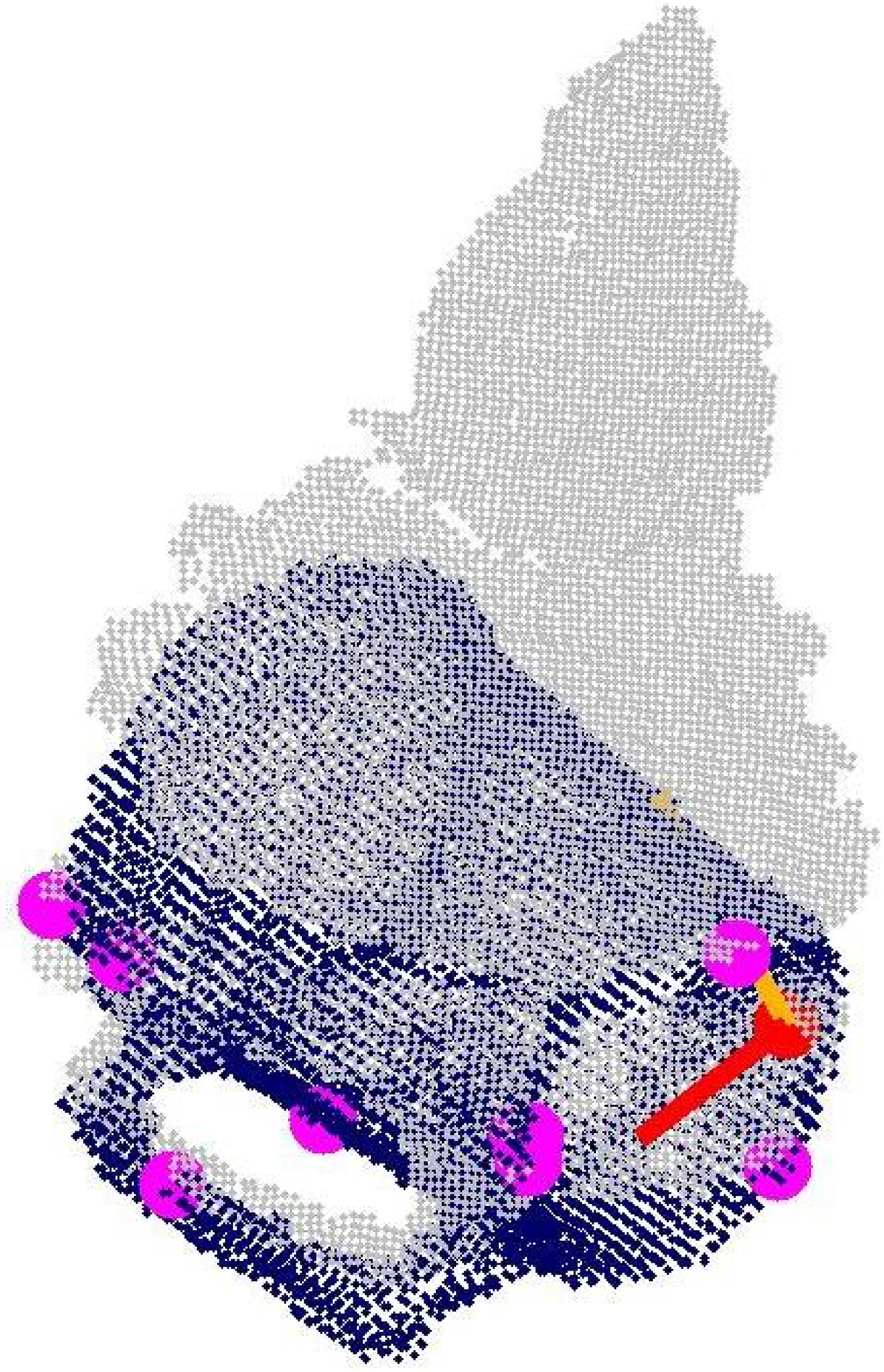}
  \caption{Object detections found with our system, along with the feature
    correspondences that BPA used to align the model.  Surface features are indicated by
    red points, with lines sticking out of them to indicate orientations (red for normals,
    orange for principal curvatures).  Edge features (which are orientation-less) are shown
    by magenta points.}
  \label{fig:scope_samples1}
\end{figure}

ICP can be slow, because it needs to find dense correspondences between the two point sets at
each iteration.  Sub-sampling the point sets can improve speed, but only at the cost of accuracy when the data is noisy.
Another drawback is its sensitivity to outliers---for example when it is applied to a cluttered scene with segmentation error.

Particularly because of the clutter problem, many modern approaches to alignment use sparse point sets, where one only uses points computed at
especially unique keypoints in the scene.  These keypoint features can be computed from either 2-D (image) or 3-D (geometry) information,
and often include not only positions, but also orientations derived from image gradients, surface normals, principal curvatures, etc.
However, these orientations are typically only used in the feature matching and pose clustering stages, and are ignored during the
alignment step.

Another limitation is that the resulting alignments are often based on just a few features, with noisy position measurements, and yet
there is very little work on estimating confidence intervals on the resulting alignments.  This is especially difficult when the features
have different noise models---for example, a feature found on a flat surface will have a good estimate of its surface normal, but a high variance
principal curvature direction, while a feature on an object edge may have a noisy normal, but precise principal curvature.  Ideally, we would
like to have a posterior distribution over the space of possible alignments, given the data, and we would like that distribution to include
information from feature positions and orientation measurements, given varying noise models.

As we will see in the next section, a full joint distribution on $\vec{t}$ and $\vec{q}$ is difficult to obtain.  However, in the original least-squares formulation,
it is possible to solve for the optimal $\vec{t}^*$ independently of $\vec{q}^*$, simply by taking $\vec{t}^*$ to be the translation which aligns
the centroids of $X$ and $Y$.  Given a fixed $\vec{t}^*$, solving for the optimal $\vec{q}^*$ then becomes tractable.
In a Bayesian analysis of the least-squares alignment problem, we seek a full distribution on $\vec{q}$ given $\vec{t}$, not just
the optimal value, $\vec{q}^*$.  That way we can assess the confidence of our orientation estimates, and fuse $p(\vec{q} | \vec{t})$
with other sources of orientation information, such as from surface normals.

Remarkably, given the common assumption of independent, isotropic Gaussian noise on position measurements (which is implicit in the
classical least-squares formulation), we can show that $p(\vec{q} | \vec{t})$ is a Bingham distribution.  This result makes it
easy to combine the least-squares distribution on $\vec{q} | \vec{t}$ with other Bingham distributions from feature orientations (or prior distributions),
since the Bingham is a common distribution for encoding uncertainty on 3-D rotations represented as unit quaternions~\cite{bingham_antipodally_1974, Glover11, antone2001robust}.


The mode of the least-squares Bingham distribution on $\vec{q} | \vec{t}$ will be exactly the same as the optimal
orientation $\vec{q}^*$ from Horn's method.  When other sources of orientation information are available, they may bias the distribution
away from $\vec{q}^*$.  Thus, it is important that the concentration (inverse variance) parameters of the Bingham distributions are accurately
estimated for each source of orientation information, so that this bias is proportional
to confidence in the measurements. (See the appendix for an example.)

We use our new alignment method, BPA, to build an object detection system for known, rigid objects in cluttered RGB-D images.
Our system combines information from surface and edge feature correspondences to improve
object alignments in cluttered scenes (as shown in figure~\ref{fig:scope_samples1}), and acheives
state-of-the-art recognition performance on both an existing Kinect data set~\cite{aldoma2012global}, and on a new data set
containing far more clutter and pose variability than any existing data set\footnote{Most existing data
sets for 3-D cluttered object detection have very limited object pose variability (most of the objects are upright),
and objects are often easily separable and supported by the same flat surface.}.

\begin{figure}[t]
  \centering
  \includegraphics[width=.45\textwidth]{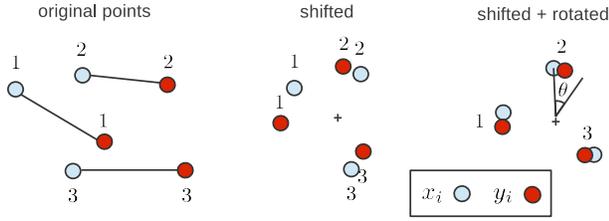}
  \caption{Rigid alignment of two point sets.}
  \label{fig:alignment_of_point_sets}
\end{figure}

\section{Bingham Procrustean Alignment}

Given two 3-D point sets $X$ and $Y$ in one-to-one correspondence, we seek a distribution over the set
of rigid transformations of $X$, parameterized by a (quaternion) rotation $\vec{q}$ and a translation
vector $\vec{t}$.  Assuming independent Gaussian noise models on deviations between corresponding
points on $Y$ and (transformed) $X$, the conditional distribution $p(\vec{q} | \vec{t},X,Y)$ is
proportional to  $p(X,Y|\vec{q},\vec{t}) p(\vec{q} | \vec{t})$, where
\begin{align}
p(X,Y|\vec{q},\vec{t}) &= \prod_i p(\vec{x_i},\vec{y_i}|\vec{q},\vec{t}) \\
&= \prod_i N(Q (\vec{x_i} + \vec{t}) - \vec{y_i} ; \vec{0}, \Sigma_i)  \label{eq:measurement_model}
\end{align}
given that $Q$ is $\vec{q}$'s rotation matrix, and covariances $\Sigma_i$.


Given isotropic noise models\footnote{This is the implicit assumption in the least-squares formulation.}
on point deviations (so that $\Sigma_i$ is a scalar times the identity matrix),
$p(\vec{x_i},\vec{y_i}|\vec{q},\vec{t})$ reduces to a 1-D Gaussian PDF on the distance
between $\vec{y_i}$ and $Q (\vec{x_i} + \vec{t})$, yielding
\begin{align*}
p(\vec{x_i},\vec{y_i}|\vec{q},\vec{t}) &= N(\| Q (\vec{x_i} + \vec{t}) - \vec{y_i} \| ; 0, \sigma_i) \\
&= N(d_i ; 0, \sigma_i)
\end{align*}
where $d_i$ depends on $\vec{q}$ and $\vec{t}$.

Now consider the triangle formed by the origin (center of rotation), $Q(\vec{x_i}+\vec{t})$ and $\vec{y_i}$,
as shown on the left of figure~\ref{fig:two_points_triangle}.  By the law of cosines, the squared-distance between
$Q(\vec{x_i}+\vec{t})$, and $\vec{y_i}$ is $d^2 = a^2 + b^2 - ab\cos(\theta)$, which only depends on
$\vec{q}$ via the angle $\theta$ between the vectors $Q(\vec{x_i}+\vec{t})$ and $\vec{y_i}$.
(We drop the $i$-subscripts on $d$, $a$, $b$, and $\theta$ for brevity.)  We can thus replace $p(\vec{x_i},\vec{y_i}|\vec{q},\vec{t})$ with
\begin{equation}
p(\vec{x_i},\vec{y_i}|\theta,\vec{t}) = \frac{1}{Z} \exp \left\{ \frac{ab\cos(\theta)}{\sigma^2} \right\} \label{eq:vonmises}
\end{equation}
which has the form of a Von-Mises distribution on $\theta$.

\begin{figure}[h]
  \centering
  \includegraphics[width=.45\textwidth]{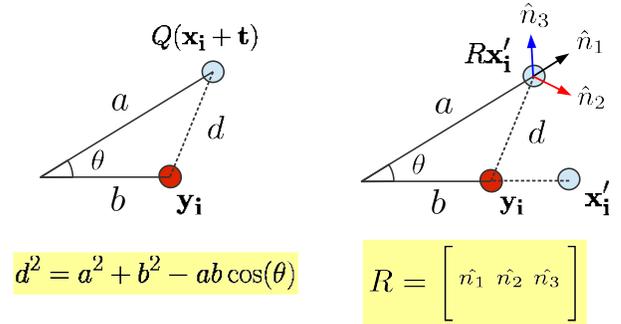}
  \caption{Distance between corresponding points as a function of orientation.}
\label{fig:two_points_triangle}
\end{figure}

Next, we need to demonstrate how $\theta$ depends on $\vec{q}$.  Without loss of generality, assume that $\vec{y_i}$ points along the axis $(1,0,0)$.
When this is not the case, the Bingham distribution over $\vec{q}$ which we derive below can be post-rotated by any quaternion which takes $(1,0,0)$
to $\vec{y_i} / \| \vec{y_i} \|$.

Clearly, there will be a family of $\vec{q}$'s which rotate
$\vec{x_i}+\vec{t}$ to form an angle of $\theta$ with $\vec{y_i}$, since we can compose $\vec{q}$ with any rotation about $\vec{x_i}+\vec{t}$
and the resulting angle with $\vec{y_i}$ will still be $\theta$.  To demonstrate what this family is, we first let $\vec{s}$ be a unit quaternion which
rotates $\vec{x_i}+\vec{t}$ onto $\vec{y_i}$'s axis, and let $\vec{x'_i} = S(\vec{x_i}+\vec{t})$, where $S$ is $\vec{s}$'s rotation matrix.
Then, let $\vec{r}$ (with rotation matrix $R$) be a quaternion that rotates $\vec{x'_i}$ to $Q(\vec{x_i}+\vec{t})$, so that
$\vec{q} = \vec{r} \circ \vec{s}$.  Because $\vec{y_i}$ and $\vec{x'_i}$ point along the axis $(1,0,0)$, the first column of $R$, $\hat{n_1}$,
will point in the direction of $Q(\vec{x_i}+\vec{t})$, and form an angle of $\theta$ with $\vec{y_i}$, as shown on the right side of
figure~\ref{fig:two_points_triangle}.  Thus, $\hat{n_1} \cdot (1,0,0) = \hat{n_{11}} = \cos \theta$.

The rotation matrix of quaternion $\vec{r} = (r_1,r_2,r_3,r_4)$ is
\begin{equation*}
  R = \biggl[
    \begin{smallmatrix} r_1^2+r_2^2-r_3^2-r_4^2 &\;\;& 2r_2r_3-2r_1r_4 &\;\;& 2r_2r_4+2r_1r_3 \\
      2r_2r_3+2r_1r_4 && r_1^2-r_2^2+r_3^2-r_4^2 && 2r_3r_4-2r_1r_2 \\
      2r_2r_4-2r_1r_3 && 2r_3r_4+2r_1r_2 && r_1^2-r_2^2-r_3^2+r_4^2
    \end{smallmatrix} \biggr]
\end{equation*}
Therefore, $\cos \theta = \hat{n_{11}} = r_1^2+r_2^2-r_3^2-r_4^2 = 1 - 2r_3^2 - 2r_4^2$.

We can now make the following claim about $p(\vec{x_i},\vec{y_i} | \vec{q},\vec{t})$:

\begin{claim}
Given that $\vec{y_i}$ lies along the $(1,0,0)$ axis, 
then the probability density $p(\vec{x_i},\vec{y_i} | \vec{q},\vec{t})$ is proportional
to a Bingham density\footnote{See the appendix for an overview of the Bingham distribution.} on $\vec{q}$ with parameters
\[ \Lambda = ( -\frac{2ab}{\sigma^2}, -\frac{2ab}{\sigma^2}, 0 ) \;\;\;\; \textrm{and} \;\;\;\;
V = \biggl[ \begin{smallmatrix} 0&0&0 \\ 0&0&1 \\ 1&0&0 \\ 0&1&0  \end{smallmatrix} \biggr] \circ \vec{s} = W \circ \vec{s} \;\;\;, \]
where ``$\circ$'' indicates column-wise quaternion multiplication.
\vspace{.01in}
\label{claim:BPA1}
\end{claim}

\begin{proof}
The Bingham density in claim~\ref{claim:BPA1} is given by
\begin{align}
p(\vec{q} | \Lambda, V) &= \frac{1}{F} \exp \sum_{j=1}^3 \lambda_j ((\vec{w_j} \circ \vec{s})^T \vec{q})^2 \\
&= \frac{1}{F} \exp \left\{ -\frac{2ab}{\sigma^2}r_3^2 -\frac{2ab}{\sigma^2}r_4^2 \right\} \\
&= \frac{1}{F'} \exp \left\{ \frac{ab \cos \theta}{\sigma^2} \right\} \label{eq:BPA_proof}
\end{align}
since $(\vec{w_j} \circ \vec{s})^T \vec{q} = \vec{w_j}^T (\vec{q} \circ \vec{s^{-1}}) = \vec{w_j}^T \vec{r}$,
and $\cos \theta = 1 - 2r_3^2 - 2r_4^2$.  Since (\ref{eq:BPA_proof}) is proportional to (\ref{eq:vonmises}), we conclude that $p(\vec{q} | \Lambda, V) \propto p(\vec{x_i},\vec{y_i} | \vec{q},\vec{t})$,
as claimed.

\end{proof}

\begin{claim}
Let $\vec{s'}$ be a quaternion that rotates $(1,0,0)$ onto the axis of $\vec{y_i}$ (for arbitrary $\vec{y_i}$).
Then the probability density $p(\vec{x_i},\vec{y_i} | \vec{q},\vec{t})$ is proportional
to a Bingham density on $\vec{q}$ with parameters
\[ \Lambda = ( -\frac{2ab}{\sigma^2}, -\frac{2ab}{\sigma^2}, 0 ) \;\;\;\; \textrm{and} \;\;\;\;
V = \vec{s'} \circ \biggl[ \begin{smallmatrix} 0&0&0 \\ 0&0&1 \\ 1&0&0 \\ 0&1&0  \end{smallmatrix} \biggr] \circ \vec{s} \;\;\;, \]
where ``$\circ$'' indicates column-wise quaternion multiplication.
\vspace{.01in}
\label{claim:BPA2}
\end{claim}

As explained above, the distribution on $\vec{q}$ from claim~\ref{claim:BPA1}
must simply be post-rotated by $\vec{s'}$ when $\vec{y_i}$ is not aligned with the $(1,0,0)$ axis.
The proof is left to the reader.  Putting it all together, we find that
\begin{align}
p(\vec{q} | \vec{t},X,Y) &\propto \prod_i \bingham(\vec{q}; \Lambda_i, V_i) \cdot p(\vec{q} | \vec{t}) \\
&= \bingham(\vec{q}; \tilde{\Lambda}, \tilde{V}) \cdot p(\vec{q} | \vec{t}) \label{eq:BPA_LS}
\end{align}
where $\Lambda_i$ and $V_i$ are taken from claim~\ref{claim:BPA2}, and where $\tilde{\Lambda}$ and $\tilde{V}$
are computed from the formula for multiplication of Bingham PDFs, which is given in the appendix.

Equation~\ref{eq:BPA_LS} tells us that, in order to update a prior on $\vec{q}$ given $\vec{t}$
after data points $X$ and $Y$ are observed, one must simply multiply the prior by an appropriate Bingham
term.  Therefore, assuming a Bingham prior over $\vec{q}$ given $\vec{t}$ (which includes the uniform distribution),
the conditional posterior, $p(\vec{q} | \vec{t},X,Y)$ is the PDF of a Bingham distribution.

To demonstrate this fact,
we relied only upon the assumption of independent isotropic Gaussian noise on position measurements, which is exactly the same
assumption made implicitly in the least-squares formulation of the optimal alignment problem.  This illustrates
a deep and hitherto unknown connection between least-squares alignment and the Bingham distribution,
and paves the way for the fusion of orientation and position measurements in a wide variety of applications.

\subsection{Incorporating Orientation Measurements}

Now that we have shown how the orientation information from the least-squares alignment of two point sets $X$ and $Y$
is encoded as a Bingham distribution, it becomes trivial to incorporate independent orientation measurements
at some or all of the points, provided that the orientation noise model is Bingham.  Given orientation
measurements $(O_X,O_Y)$,
\begin{align*}
p(\vec{q} | \vec{t},X,&Y,O_X,O_Y) \\
&\propto p(X,Y,O_X,O_Y | \vec{q},\vec{t}) \cdot p(\vec{q} | \vec{t}) \\
&= p(X,Y | \vec{q},\vec{t}) \cdot p(O_X,O_Y | \vec{q},\vec{t}) \cdot p(\vec{q} | \vec{t}) \;\;.
\end{align*}
Similarly as in equation~\ref{eq:BPA_LS}, $p(O_X,O_Y | \vec{q},\vec{t})$ is the product of
Bingham distributions from corresponding orientation measurements in $(O_X,O_Y)$, and so
the entire posterior $p(\vec{q} | \vec{t},X,Y,O_X,O_Y)$ is Bingham (provided as before that
the prior $p(\vec{q} | \vec{t})$ is Bingham).

\subsection{The Alignment Algorithm}

To incorporate our Bayesian model into an iterative ICP-like alignment algorithm, one could
solve for the \textit{maximum a posteriori} (MAP) position and orientation by maximizing
$p(\vec{q},\vec{t} | X,Y,\ldots)$ with respect to $\vec{q}$ and $\vec{t}$.  However, for probabilistic
completeness, it is often more desirable to draw samples from this posterior distribution.

The joint posterior distribution $p(\vec{q},\vec{t} | Z)$---where $Z$ contains
all the measurements ($X,Y,O_X,O_Y,\ldots$)---can be broken up into
$p(\vec{q} | \vec{t},Z) p(\vec{t} | Z)$.  Unfortunately, writing down a closed-form
distribution for $p(\vec{t} | Z)$ is difficult.  But sampling from the joint distribution is
easy with an importance sampler, by first sampling $\vec{t}$ from a proposal distribution---for
example, a Gaussian centered on the optimal least-squares translation (that aligns the centroids
of $X$ and $Y$)---then sampling $\vec{q}$ from $p(\vec{q} | \vec{t},Z)$, and then weighting the
samples by the ratio of the true posterior (from equation~\ref{eq:measurement_model}) and the proposal
distribution (e.g., Gaussian times Bingham).

We call this sampling algorithm Bingham Procrustean Alignment (BPA). It takes as input a set of (possibly
oriented) features in one-to-one correspondence, and returns samples from the distribution over possible alignments.
In section~\ref{sec:scope}, we will show how BPA can be incorporated into an iterative alignment algorithm
that re-computes feature correspondences at each step and uses BPA to propose a new alignment given the correspondences.

\section{Building Noise-Aware 3-D Object Models}

\begin{figure}[t]
  \centering
  \includegraphics[width=.4\textwidth]{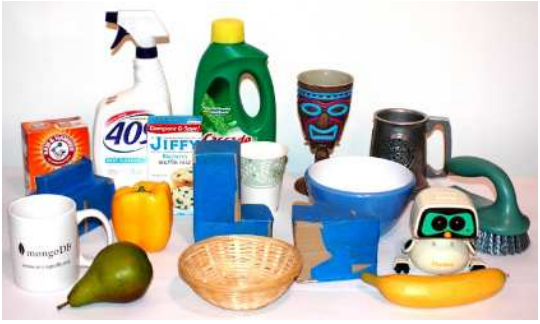}
  \caption{The 18 objects in our new \textit{Clutter} data set.}
\label{fig:kinect_scans}
\end{figure}

Our first step in building a system to detect known, rigid objects---such as the ones in figure~\ref{fig:kinect_scans}---is to
build complete 3-D models of each object.  However, the end goal of model building is not just to estimate an object's
geometry correctly.  Rather, we seek to predict what an RGB-D sensor would see, from every possible viewing angle of the object.
To generate such a predictive model, we will estimate both the most likely observations from each viewing angle, and also
the degree of noise predicted in those measurements.  That way, our detection system will realize that depth measurements
near object boundaries, on reflective surfaces, or on surfaces at a high oblique angle with respect to the camera, are less
reliable than front-on measurements of non-reflective, interior surface points.

In our model-building system, we place each object on a servo-controlled turntable in 2-3 resting positions and
collect RGB-D images from a stationary Kinect sensor at $10^\circ$ turntable increments, for a total of 60-90 views.
We then find the turntable plane in the depth images (using RANSAC), and separate object point clouds (on top of the turntable) from
the background.  Next we align each set of 30 scans (taken of the object in a single resting position) by optimizing for
the 2-D position of the turntable's center of rotation, with respect to an alignment cost function that measures the
sum-of-squared nearest-neighbor distances from each object scan to every other scan.  We then use another optimization
to solve for the 6-dof translation + rotation that aligns the 2-3 sets of scans together into one, global frame of reference.

After the object scans are aligned, we compute their surface normals, principal curvatures, and FPFH features~\cite{rusu_fpfh_2009},
and we use the the ratio of principal curvatures to estimate the (Bingham) uncertainty on the quaternion orientation
defined by normals and principal curvature directions at each point\footnote{The idea is to capture the orientation uncertainty
on the principal curvature direction by measuring the ``flatness'' of the observed surface patch; see the appendix for details.}.
We then use ray-tracing to build a 3-D occupancy grid model, where in addition to the typical probability of
occupancy, we also store each 3-D grid cell's mean position and normal, and variance on the normals in that cell\footnote{
In fact, we store two ``view-buckets'' per cell, each containing an occupancy probability, a position, a normal, and a
normal variance, since on thin objects like cups and bowls, there may be points on two different surfaces which fall in
the same grid cell.}.  We then threshold the occupancy grid at an occupancy probability of $0.5$, and remove interior
cells (which cannot be seen from any viewing angle) to obtain a full model point cloud, with associated normals and
normal variance estimates.  We also compute a distance transform of this model point cloud, by computing the distance
from the center of each cell in the occupancy grid to the nearest model point (or zero if the cell contains a model point).

Next, for a fixed set of 66 viewing angles across the view-sphere, we estimate range edges---points on the model where there
is a depth discontinuity in the predicted range image seen from that view angle.  We also store the minimum distance
from each model point to a range edge for each of the 66 viewing angles.  Using these view-dependent edge distances, along
with the angles between surface normals and viewpoints, we fit sigmoid models across the whole data set to estimate the
expected noise on range measurements and normal estimates as functions of (1) edge distance, and (2) surface angle, as shown in
figure~\ref{fig:noise_models}.

\begin{figure}[]
  \centering
  \includegraphics[width=.35\textwidth]{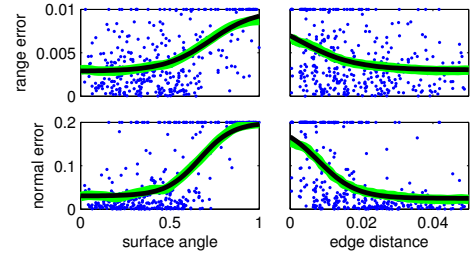}
  \caption{Our noise models predict range and normal errors (standard deviations) as functions of
  surface angle and edge distance (both with respect to the viewpoint).}
\label{fig:noise_models}
\end{figure}

\section{Learning Discriminative Feature Models for Detection}

Similarly to other recent object detection systems, our system computes a set of feature model placement score functions,
in order to evaluate how well a given model placement hypothesis fits the scene according to different features, such as
depth measurements, surface normals, edge locations, etc.  In our early experiments with object detection using the
generative object models in the previous section, the system was prone to make mis-classification errors, because some objects
scored consistently higher on certain feature scores (presumably due to training set bias).  Because of this problem,
we trained discriminative, logistic regression models on each of the score components using the turntable scans with
true model placements as positive training examples and a combination of correct object / wrong pose and wrong object / aligned
pose as negative examples.  Alignments of wrong objects were found by running the full object detection system (from the next
section) with the wrong object on the turntable scans.  By adding an (independent) discriminative layer to each of the
feature score types, we were able to boost the classification accuracy of our system considerably.

\begin{figure*}[t!]
  \includegraphics[width=.99\textwidth]{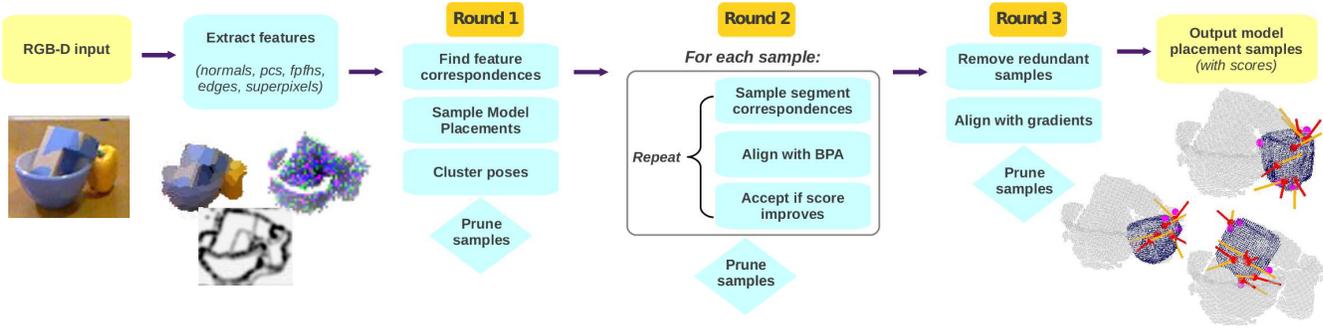}
  \caption{Single object detection pipeline.}
  \label{fig:flow_chart}
\end{figure*}

\section{Detecting Single Objects in Clutter} \label{sec:scope}

The first stages of our object detection pipeline are very similar to many other state-of-the-art systems
for 3-D object detection, with the exception that we rely more heavily on edge information.
We are given as input an RGB-D image, such as from a Kinect.  If environmental information is available,
the image may be pre-processed by another routine to crop the image to an area of interest, and to label
background pixels (e.g., belonging to a supporting surface).

As illustrated in figure~\ref{fig:flow_chart}, our algorithm starts by estimating a dense set of surface normals on the
3-D point cloud derived from the RGB-D image.
From these surface normals, it estimates principal curvatures and FPFH features.  In addition, it finds and labels three
types of edges: range edges, image edges, and curvature edges---points in
the RGB-D image where there is a depth discontinuity, an image intensity discontinuity\footnote{We use the Canny
edge detector to find image edges.}, or high negative curvature.  This edge information is converted into
an edge image, which is formed from a spatially-blurred, weighted average of the three edge pixel masks.
Intuitively, this edge image is intended to capture the relative likelihood that each point in the image is part
of an object boundary.  Then, the algorithm uses k-means to over-segment the point cloud based on positions,
normals, and spatially-blurred colors (in CIELAB space) into a set of 3-D super-pixels.

\begin{figure}[h]
  \includegraphics[width=.11\textwidth]{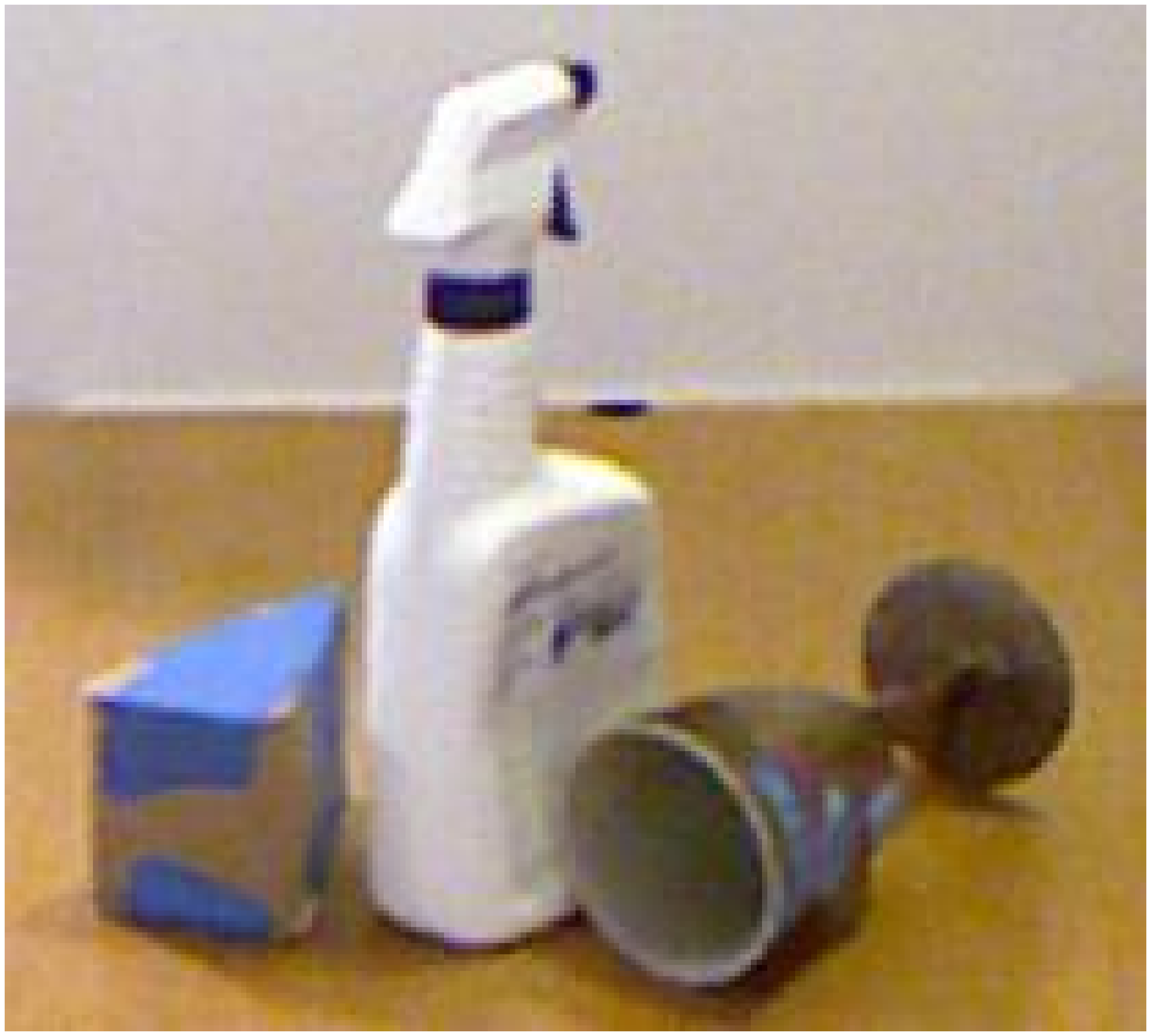}%
  \hfill
  \includegraphics[width=.125\textwidth]{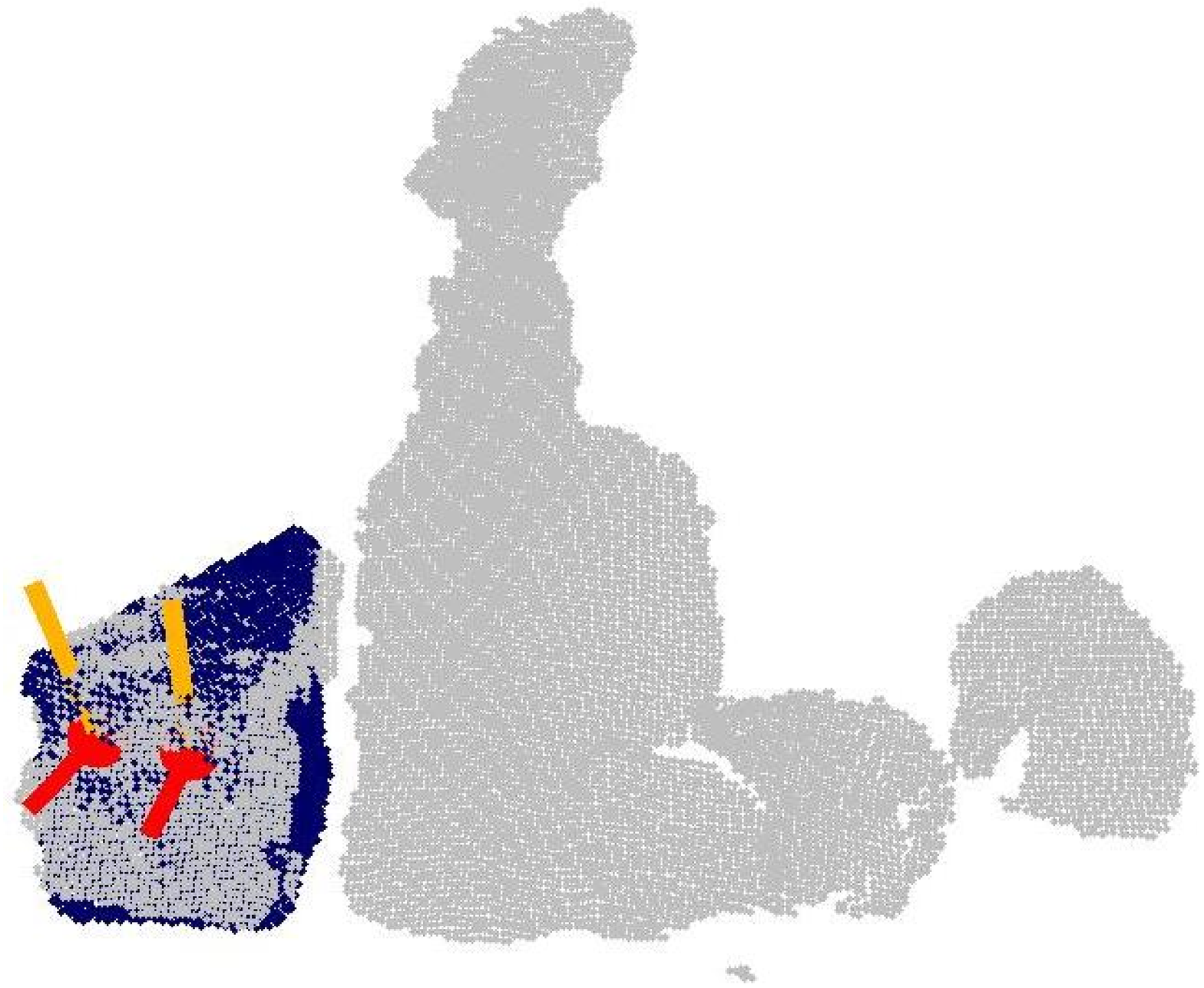}
  \includegraphics[width=.125\textwidth]{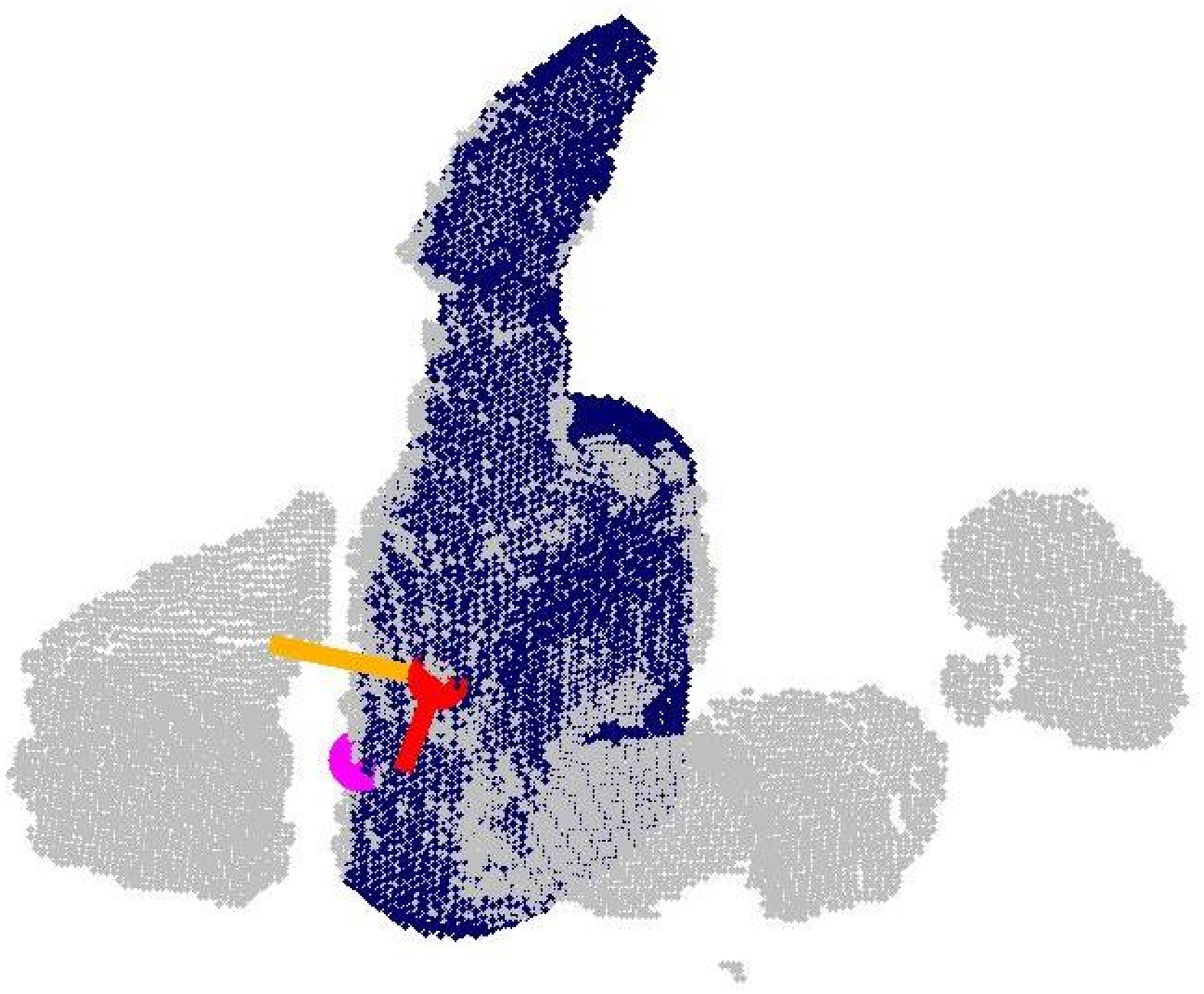}%
  \includegraphics[width=.125\textwidth]{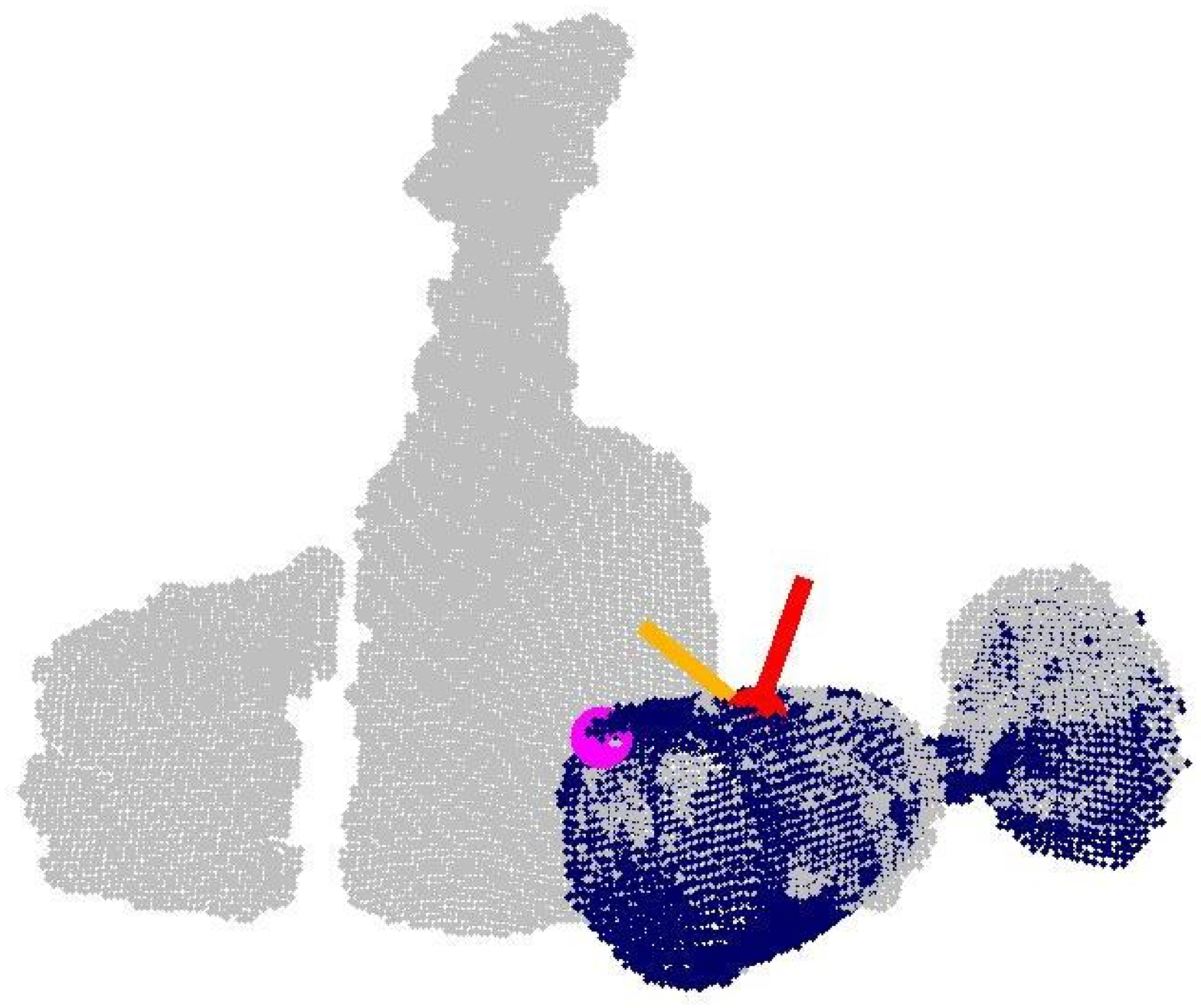}
  \caption{Examples of objects correctly aligned by BPA with only two correspondences.}
  \label{fig:two_correspondences}
\end{figure}

Next, the algorithm samples possible oriented feature correspondences from the scene to the model\footnote{We currently use
only FPFH correspondences in the first stage of detection as we did not find the addition of other feature
types, such as SIFT~\cite{lowe1999object} or SHOT~\cite{tombari2010unique}, to make any difference in our
detection rates.}.  Then, for each correspondence, a candidate object pose is sampled using BPA.
Given a set of sampled model poses from single correspondences, we then reject samples for which more than $20\%$ of a
subset of $500$ randomly-selected model points project into free space---places where the difference between observed range image
depth and predicted model depth is above $5cm$.  Next, we run a pose clustering stage, where we group correspondences
together whose sampled object poses are within $2.5cm$ and $\pi/16$ radians of one another.
After pose clustering, we reject any sample with less than two correspondences, then re-sample object poses with BPA.

At this stage, we have a set of possible model placement hypotheses, with at least two features correspondences each.
Because BPA uses additional orientation information, two correspondences is often all it takes to lock down a very
precise estimate of an object's pose when the correspondences are correct (Figure~\ref{fig:two_correspondences}).

We proceed with a second round of model placement validation and rejection, this time using a scoring function that includes
(1) range and normal differences, which are computed by projecting a new subset of $500$ randomly-selected model points
into the observed range image, (2) visibility---the ratio of model points in the subset that are unoccluded,
(3) edge likelihoods, computed by projecting the model's edge points from the closest stored viewpoint into the observed
edge image, and (4) edge visibility---the ratio of edge points that are unoccluded.  Each of the feature score components
is computed as a truncated (so as not to over-penalize outliers), average log-likelihood of observed features given model
feature distributions.  For score components (1) and (3), we weight the average log-likelihood by visibility probabilities,
which are equal to 1 if predicted depth $<$ observed depth, and $N(\Delta \textrm{depth} ; 0, \sigma_{vis}) / N(0;0,1)$
otherwise\footnote{We use $\sigma_{vis} = 1cm$ in all of our experiments.}.

After rejecting low-scoring samples in round 2, we then refine alignments by repeating the following three steps:
\begin{enumerate}
\item Assign observed super-pixel segments to the model.
\item Align model to the segments with BPA.
\item Accept the new alignment if the round 2 model placement score has improved.
\end{enumerate}
In step (1), we sample a set of assigned segments according to the probability that each segment belongs
to the model, which we compute as the ratio of segment points (sampled uniformly from the segment)
that are within $1cm$ in position and $\pi/16$ radians in normal orientation from the closest model
point.  In step (2), we randomly extract a subset of 10 segment points from the set of assigned segments,
find nearest neighbor correspondences from the keypoints to the model using the model distance transform,
and then use BPA to align the model to the 10 segment points.
Segment points are of two types---surface points and edge points.  We only assign segment edge points
to model edge points (as predicted from the given viewpoint), and surface points to surface points.
Figures~\ref{fig:scope_samples1} and~\ref{fig:scope_samples2} show examples of object alignments found
after segment alignment, where red points (with red normal vectors and orange principal curvature
vectors sticking out of them) indicate surface point correspondences, and magenta points (with no
orientations) are the edge point correspondences\footnote{In future work, we plan to incorporate
edge orientations as well.}.

After round 2 alignments, the system removes redundant samples (with the same or similar poses),
and then rejects low scoring samples using the scores found at the end of the segment alignment
process.  Then, it performs a final, gradient-based alignment, which
optimizes the model poses with a local hill-climbing search to directly maximize model placement
scores.  Since this alignment step is by far the slowest, it is critical that the system has
performed as much alignment with BPA and has rejected as many low-scoring samples as possible,
to reduce the computational burden.

Finally, the system performs a third round of model placement evaluation, then sorts the pose samples
by score and returns them.  This third round of scoring includes several additional feature score
components:
\begin{itemize}
\item Random walk score---starting from an observed point corresponding to the model,
take a random walk in the edge image (to stay within predicted object boundaries),
then measure the distance from the new observed point to the closest model point.
\item Occlusion edge score---evaluate how well model occlusion edges (where the model surface changes
from visible to occluded) fits the observed edge image.
\item FPFH score---computes how well observed and model FPFH features match.
\item Segment score---computes distances from segment points to nearest model points.
\item Segment affinity score---measures how consistent the set of assigned segments
is with respect to predicted object boundaries (as measured by the observed edge image, and by
differences in segment positions and normals).
\end{itemize}

\section{Detecting Multiple Objects in Clutter}

To detect multiple objects in a scene, we run the individual object detector from the
previous section to obtain the 50 best model placements for each model, along with their
individual scores.  Then, following Aldoma et. al~\cite{aldoma2012global}, we use simulated
annealing to optimize the subset of model placements (out of $50 \times N$ for $N$ models)
according to a multi-object-placement score, which we compute as a weighted sum of the following
score components: (1) the average of single object scores, weighted by the number of observed points each object
explains, (2) the ratio of explained / total observed points, and (3) a small penalty for the total number of detected objects.
We also keep track of the top $100$ multi-object-placement samples found during optimization, so 
we can return a set of possible scene interpretations to the user (in the spirit of interpretation tree
methods~\cite{grimson1987localizing}).  This is particularly useful for robot vision systems because they can use
tracking, prior knowledge, or other sensory input (like touch) to provide additional validation of model placements,
and we don't want detections from a single RGB-D image to filter out possible model placements prematurely.


\begin{figure*}[t!]
    \includegraphics[height=.65in]{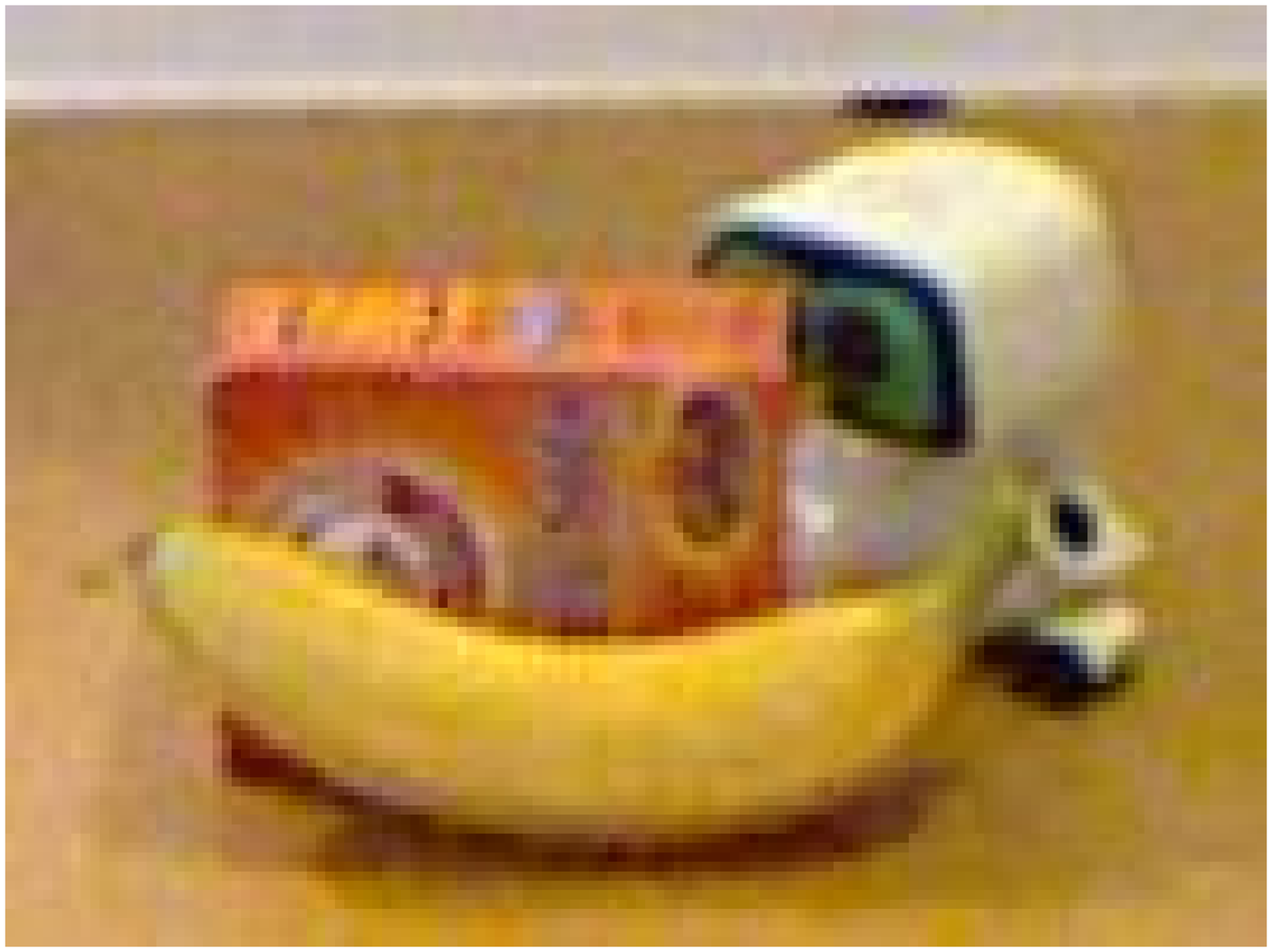}%
    \includegraphics[height=.65in]{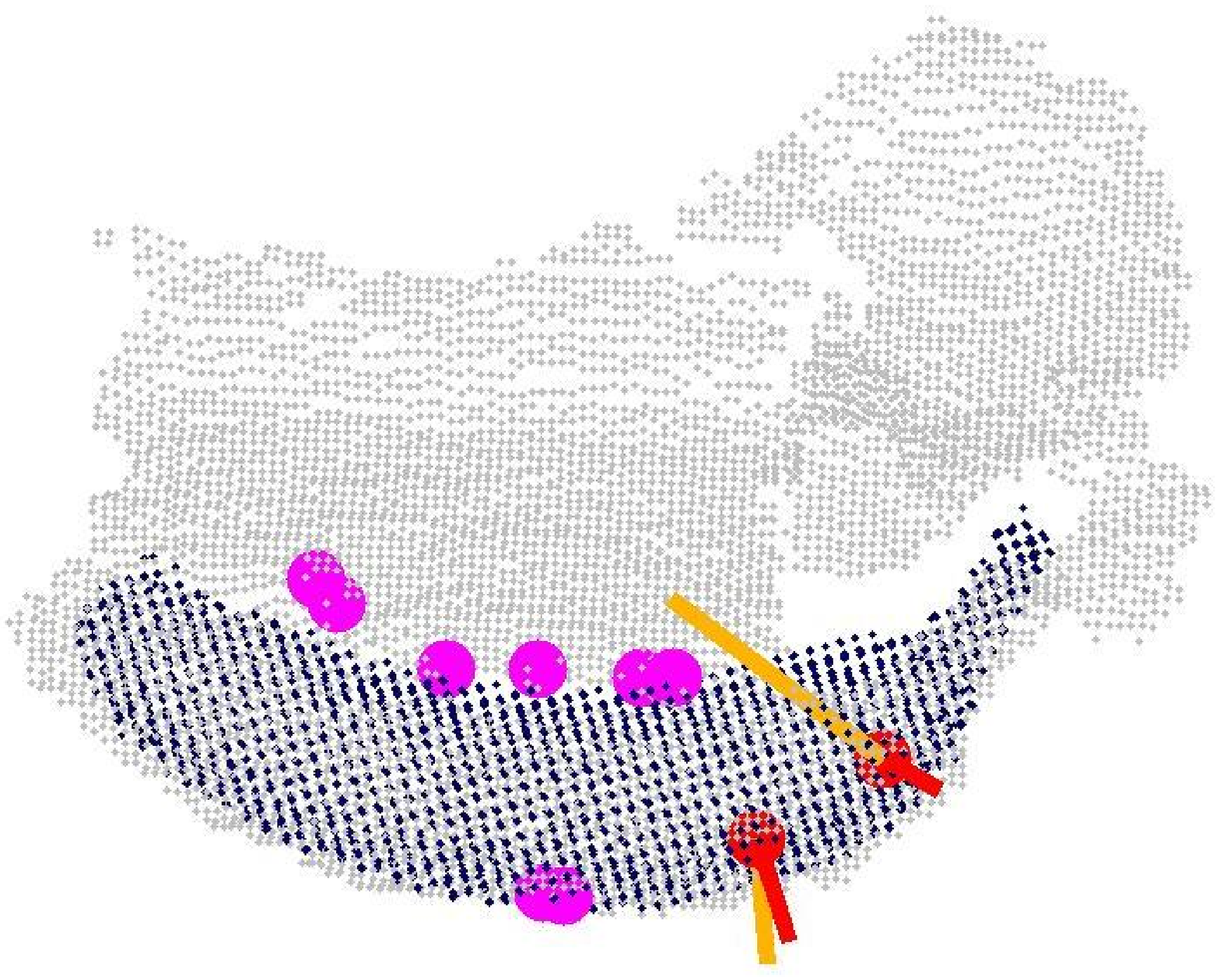}%
    \includegraphics[height=.65in]{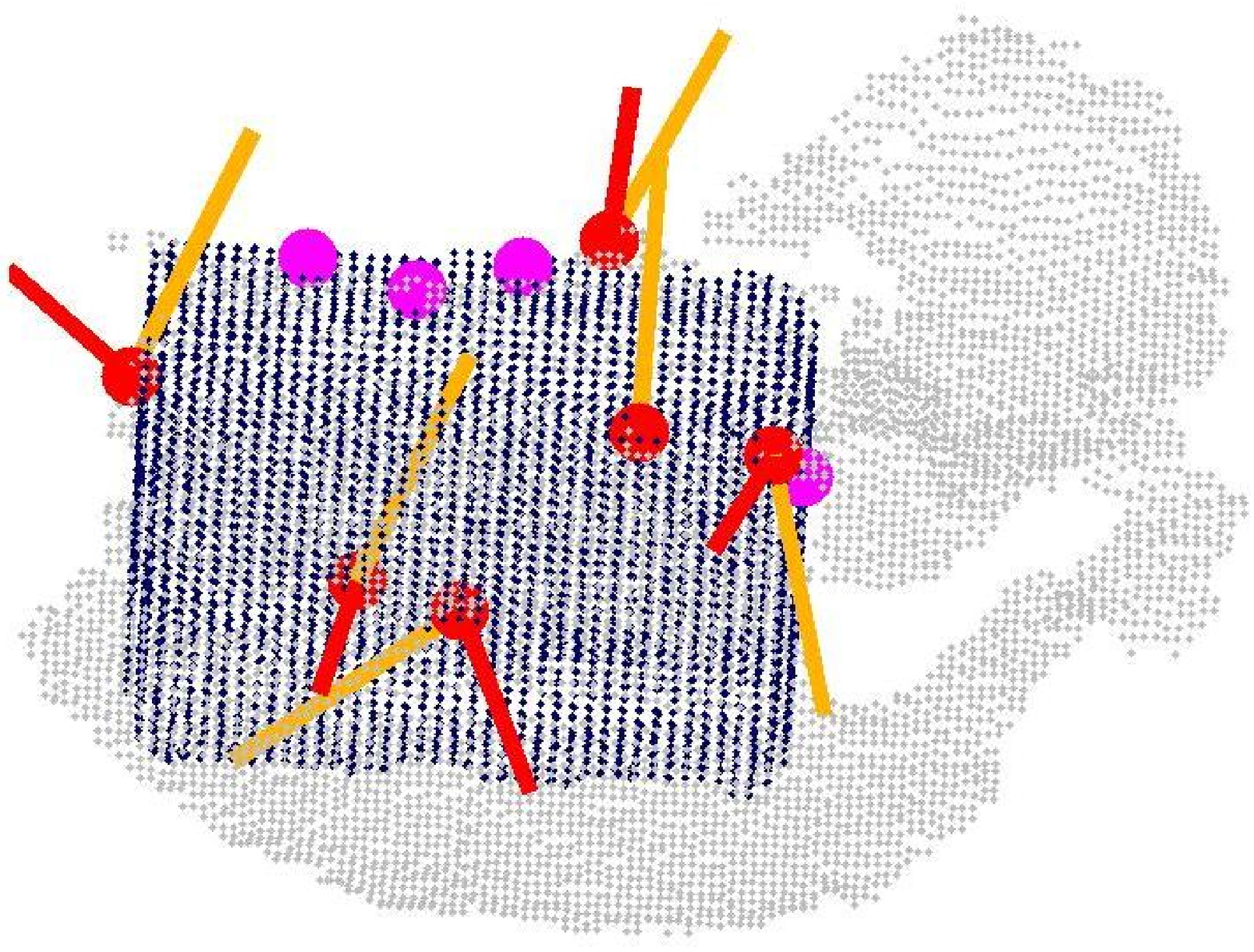}%
    \includegraphics[height=.65in]{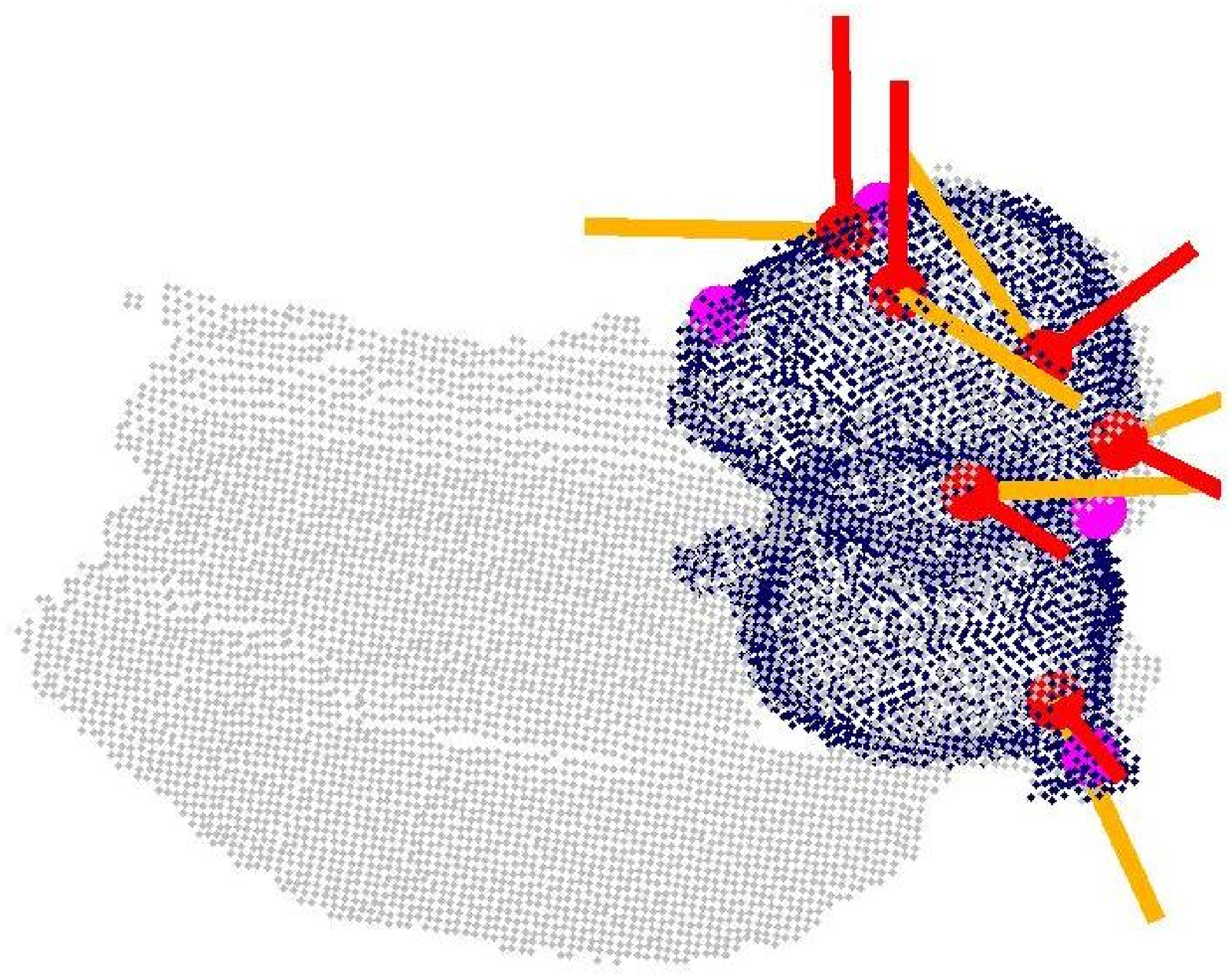}%
  \hfill
    \includegraphics[height=.65in]{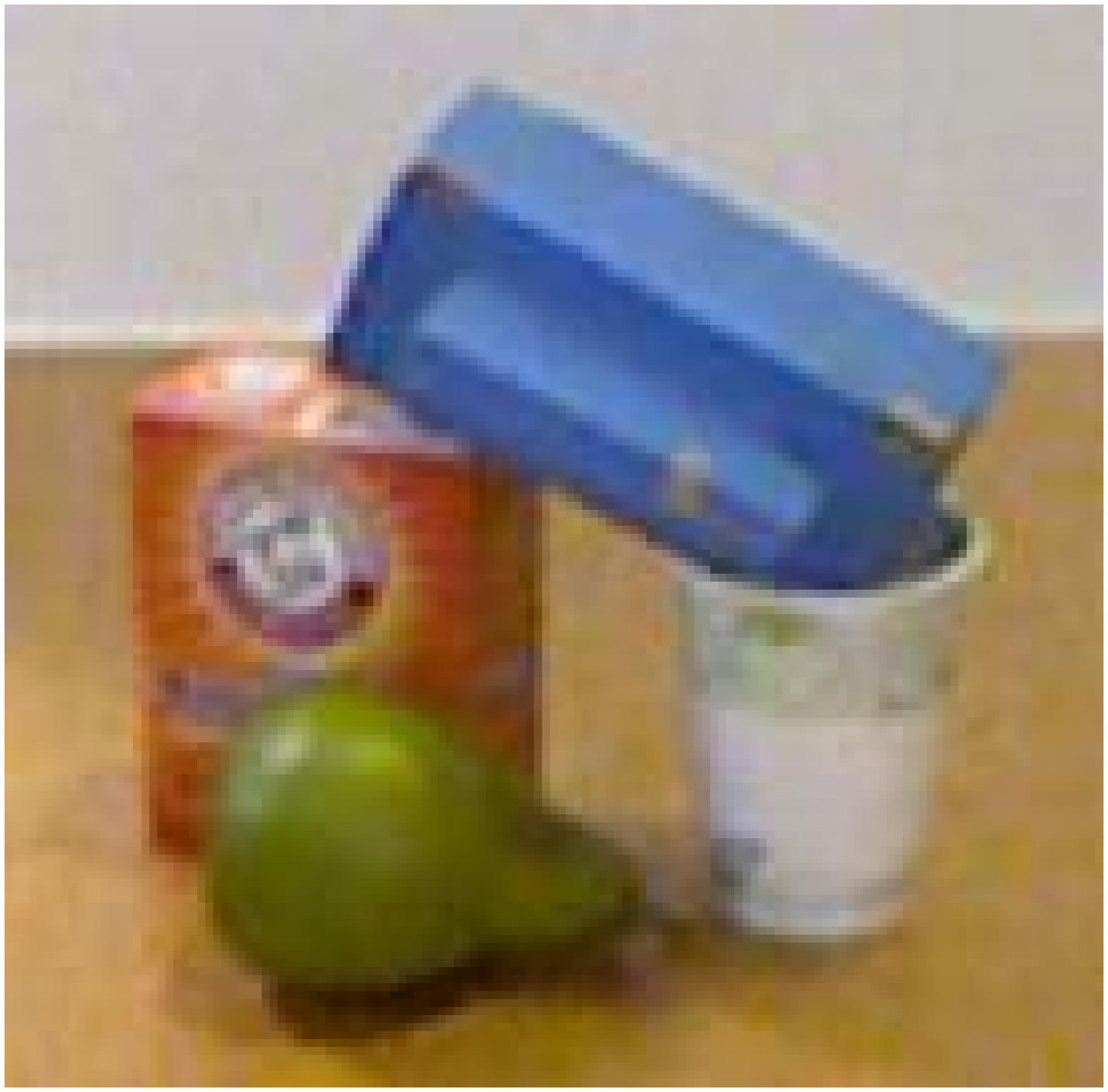}%
    \includegraphics[height=.65in]{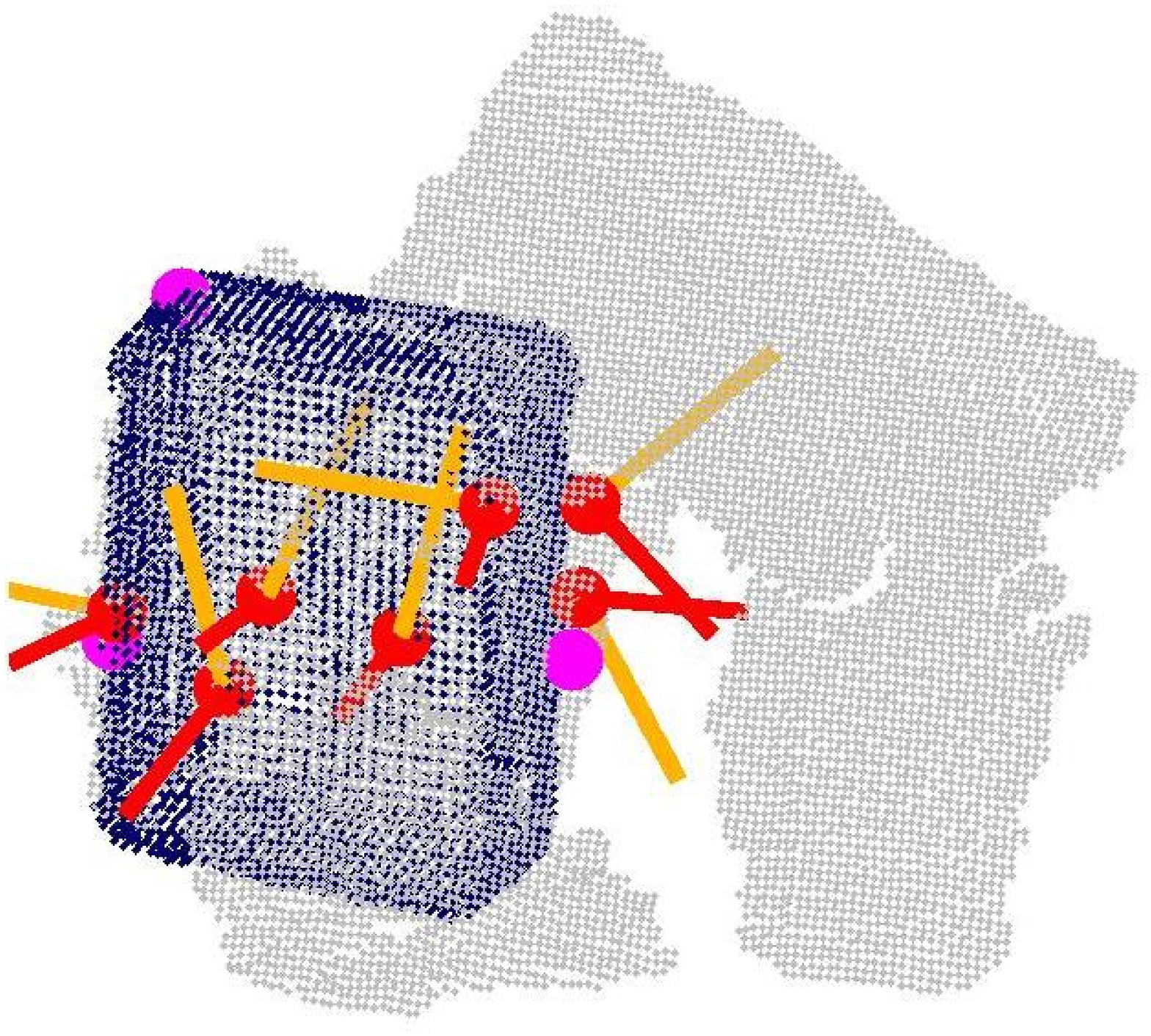}%
    \includegraphics[height=.65in]{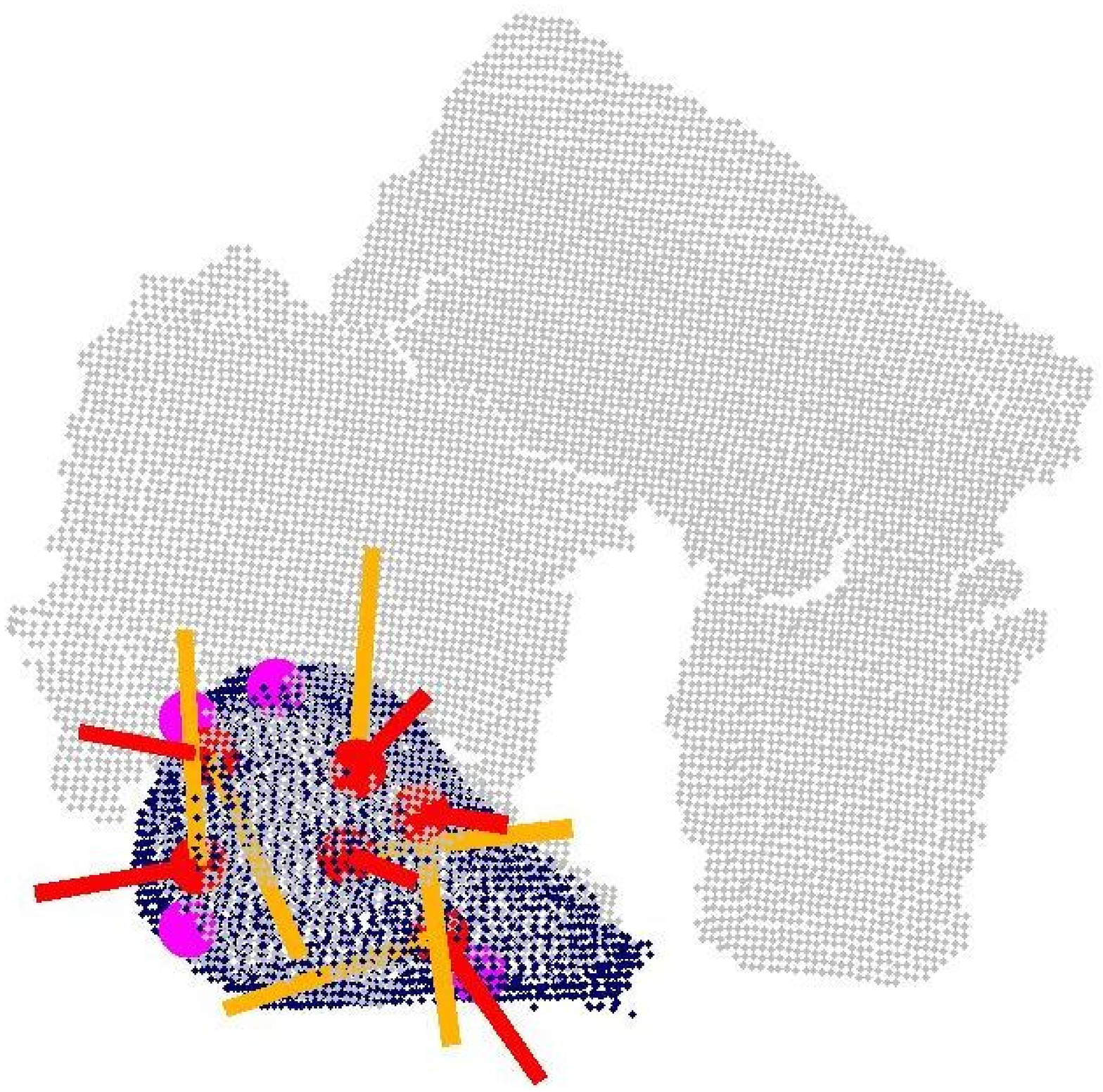}%
    \includegraphics[height=.65in]{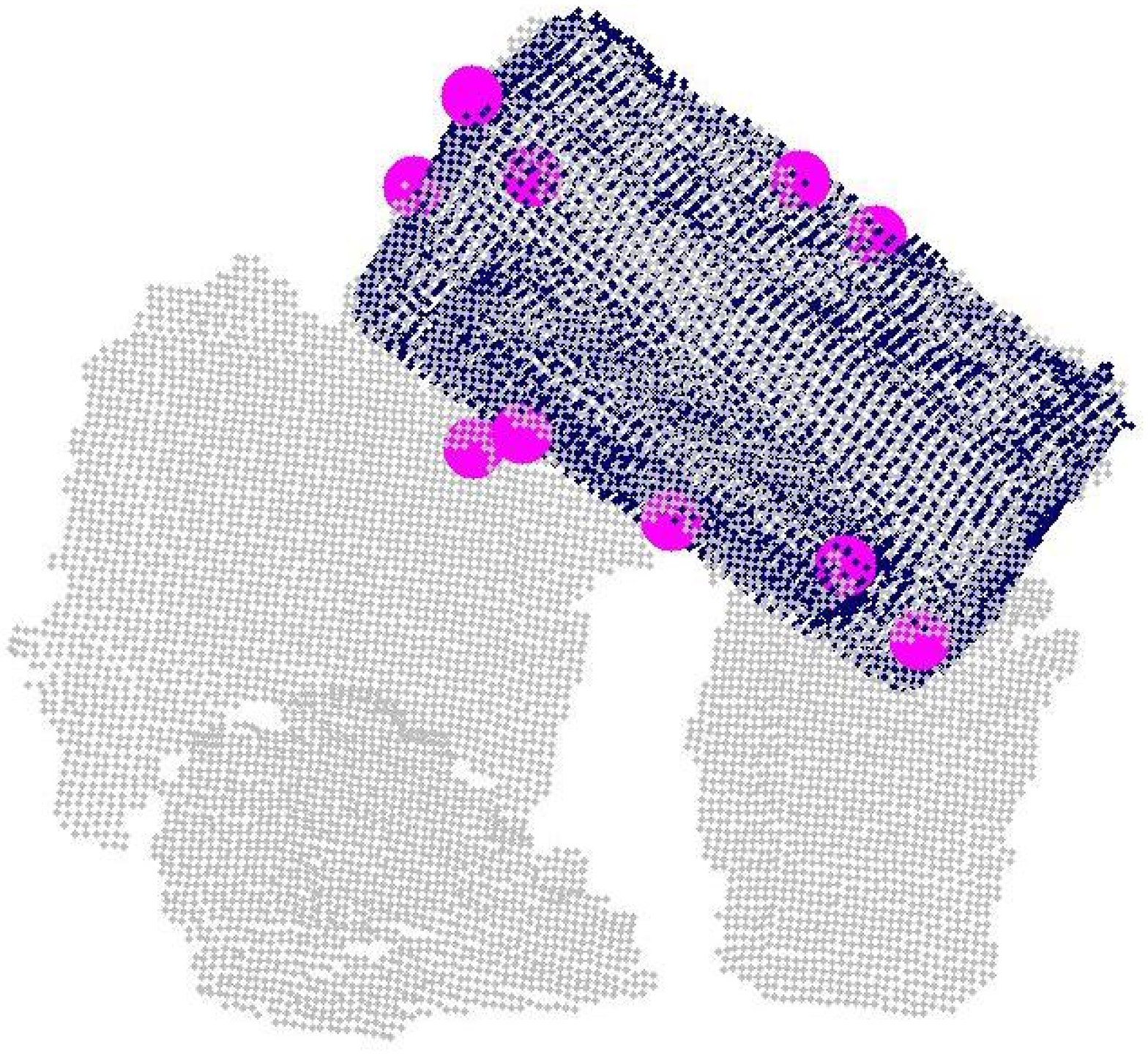}%
    \includegraphics[height=.65in]{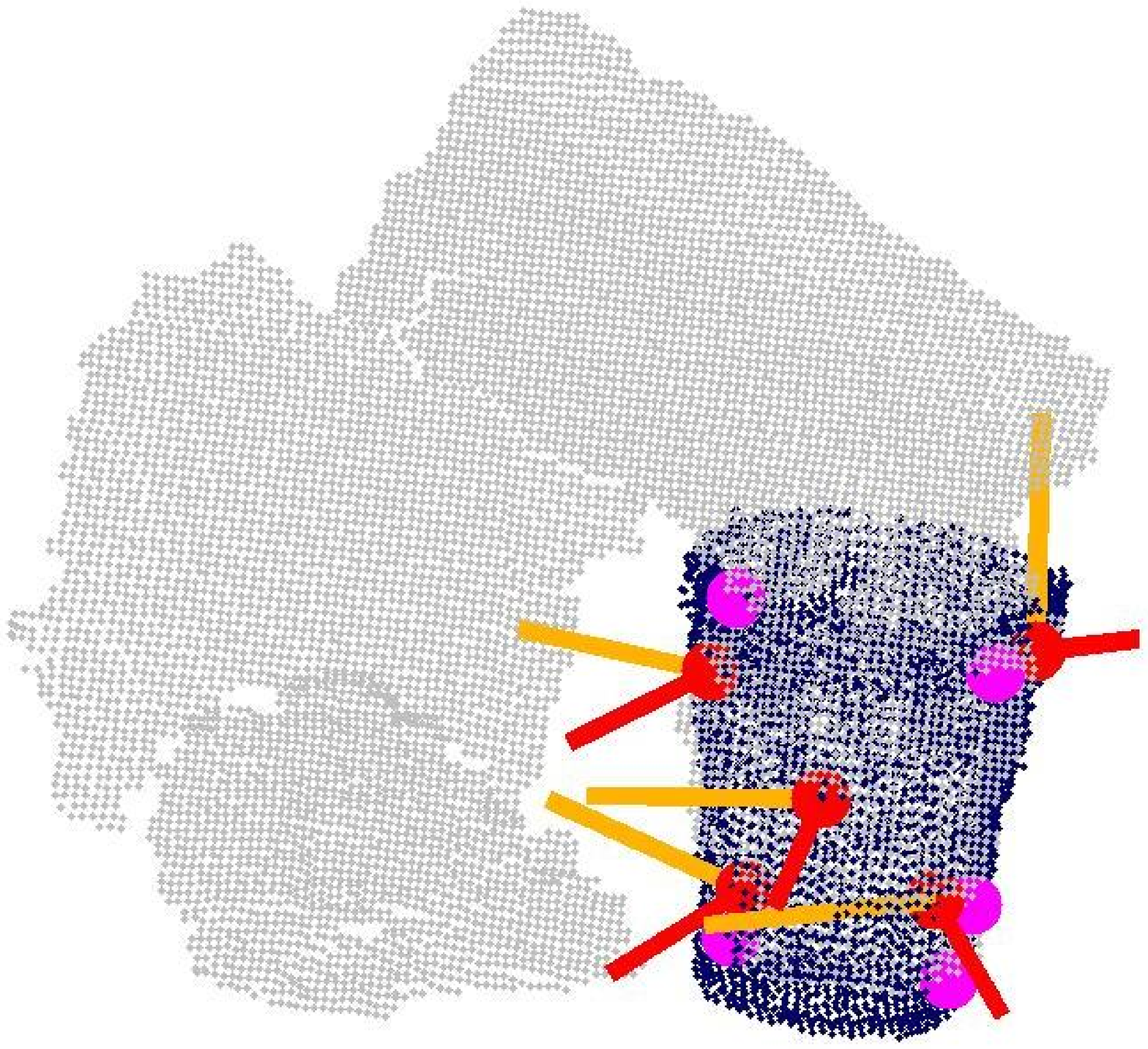} \\ \\
    \includegraphics[height=.65in]{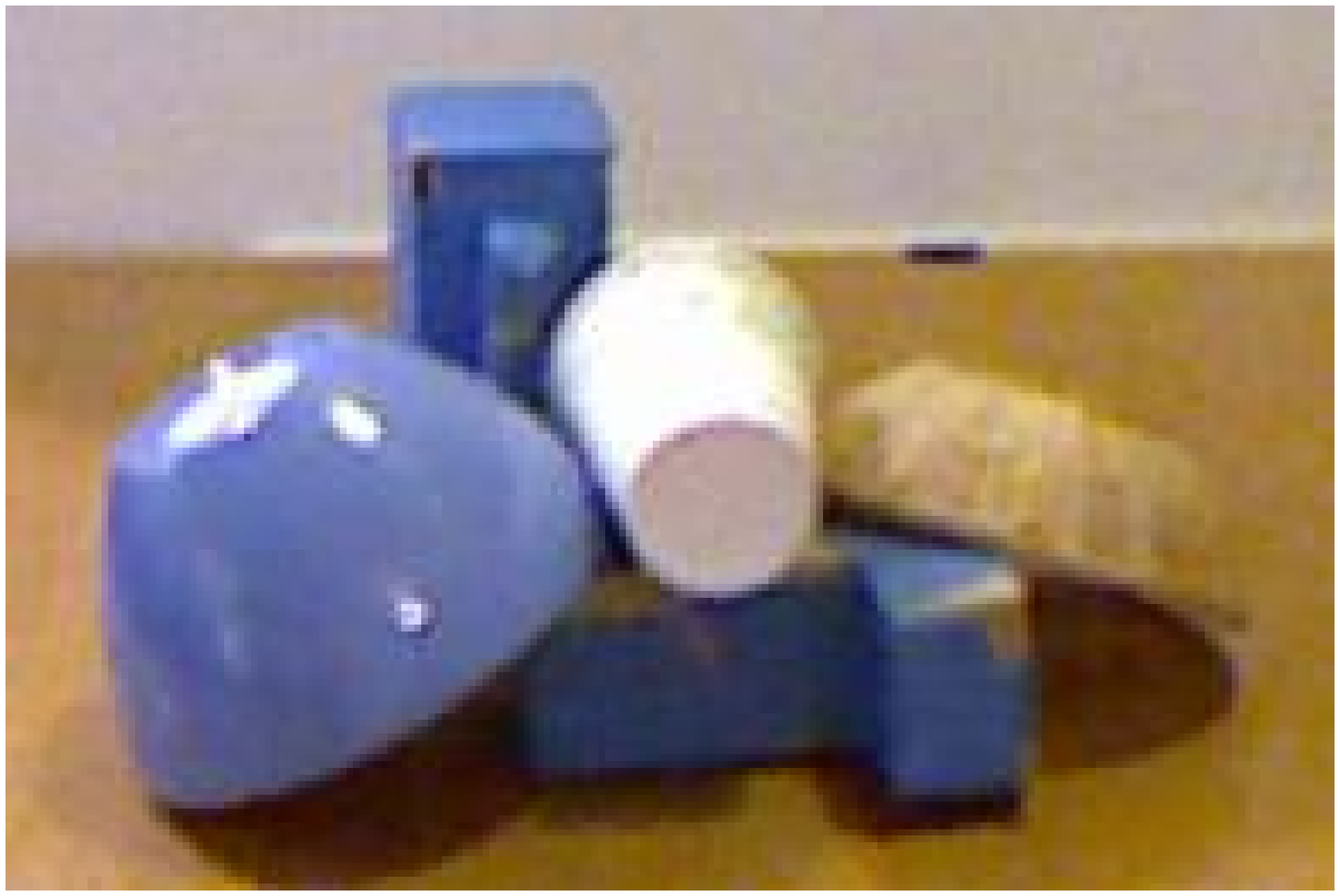}%
    \includegraphics[height=.65in]{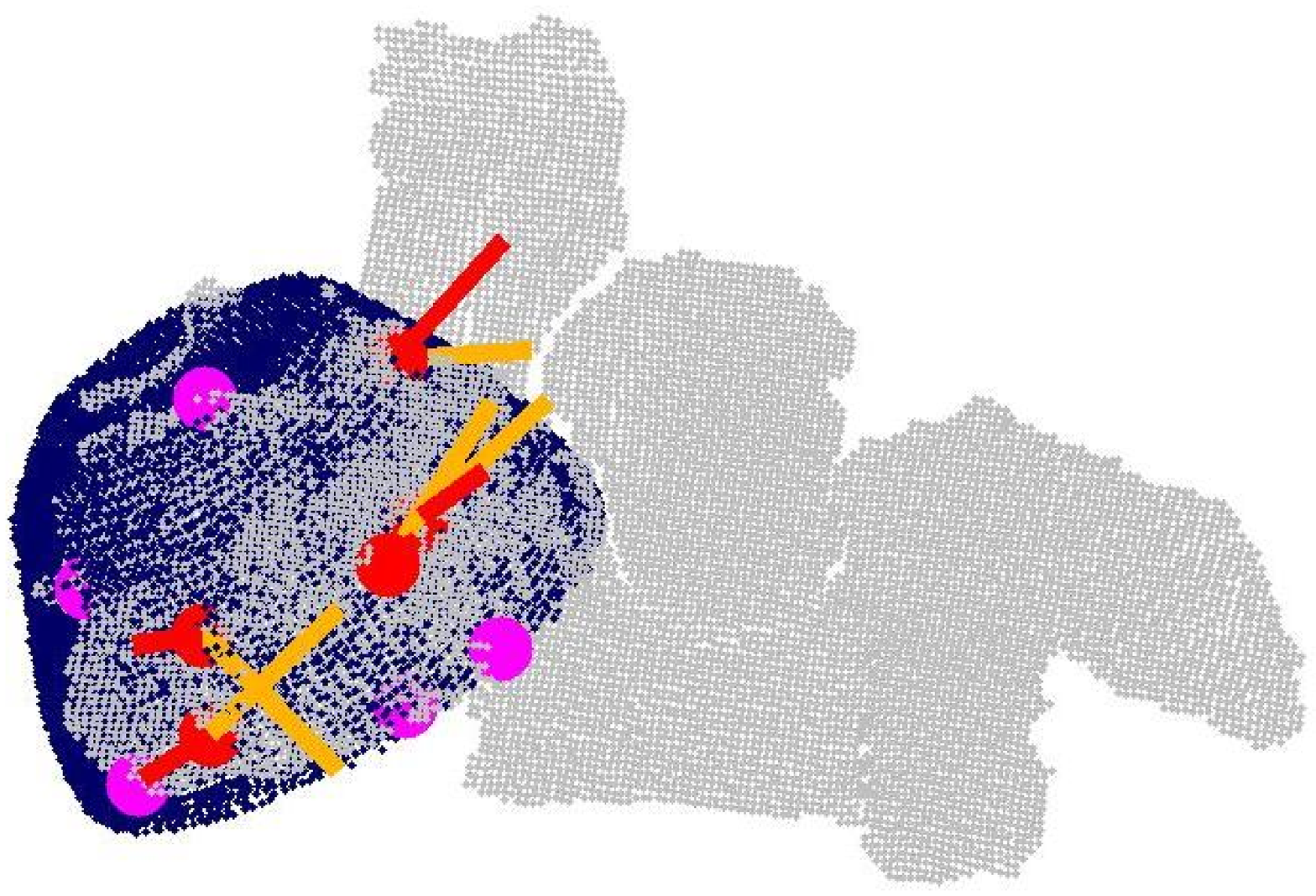}%
    \includegraphics[height=.65in]{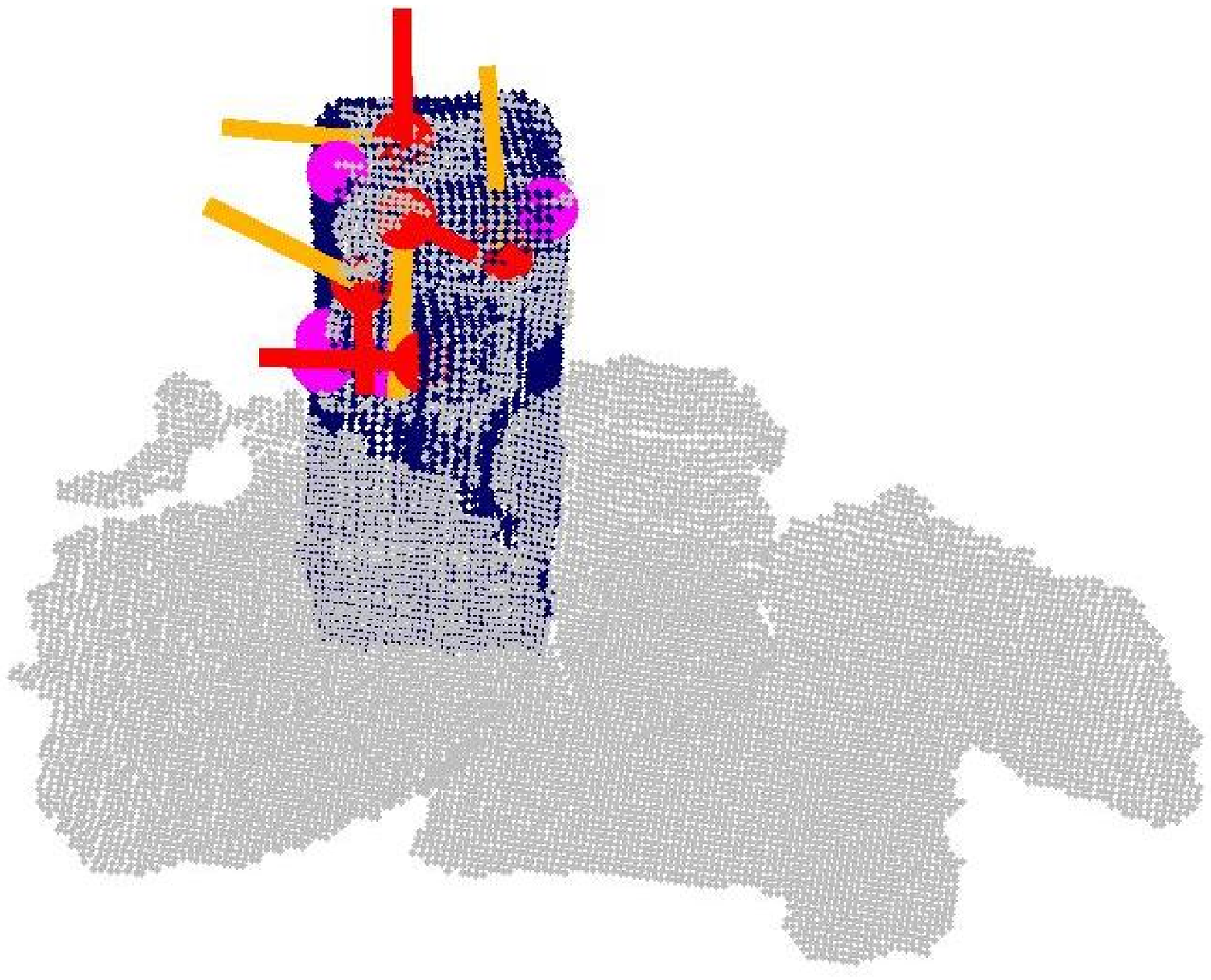}%
    \includegraphics[height=.65in]{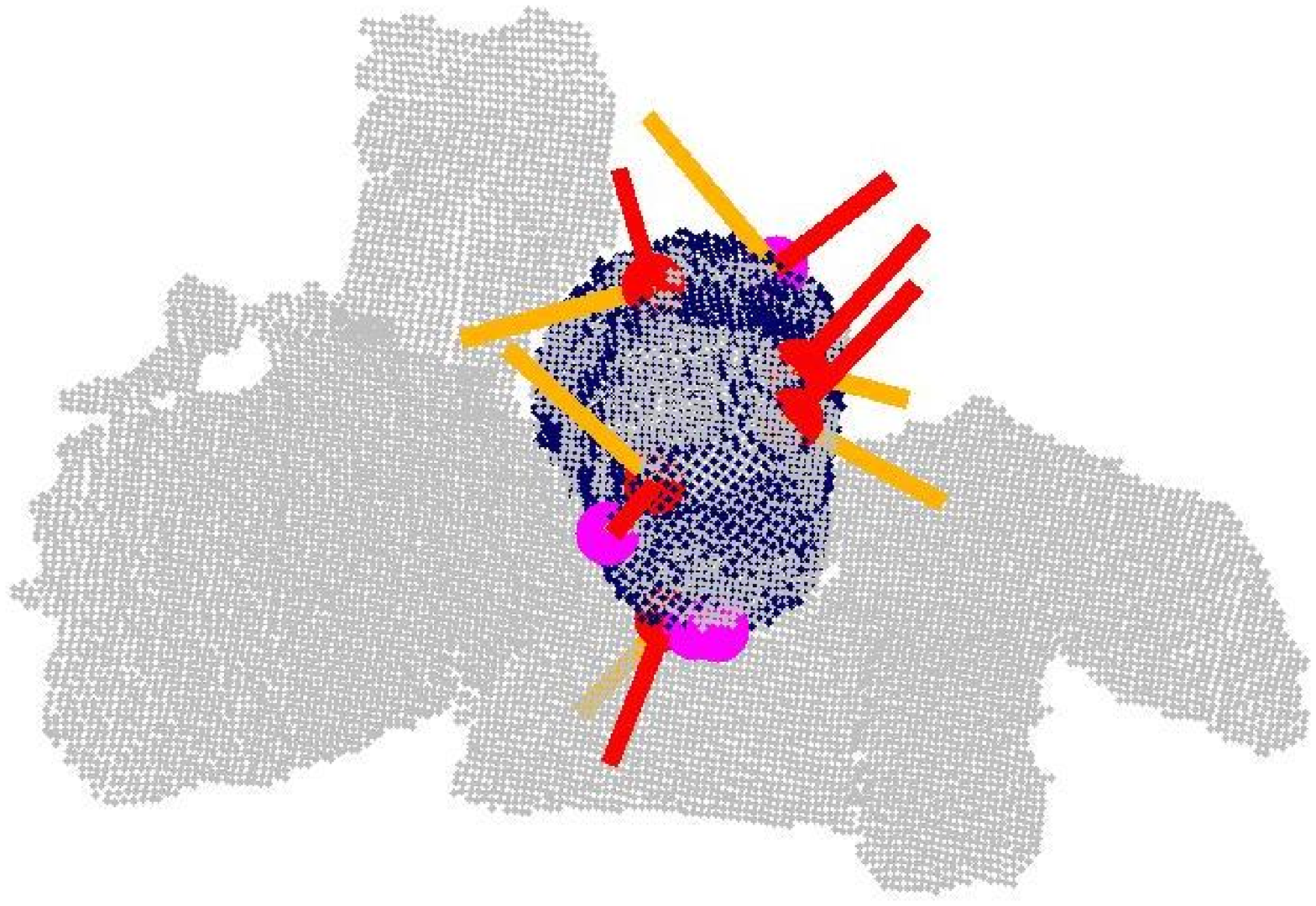}%
    \includegraphics[height=.65in]{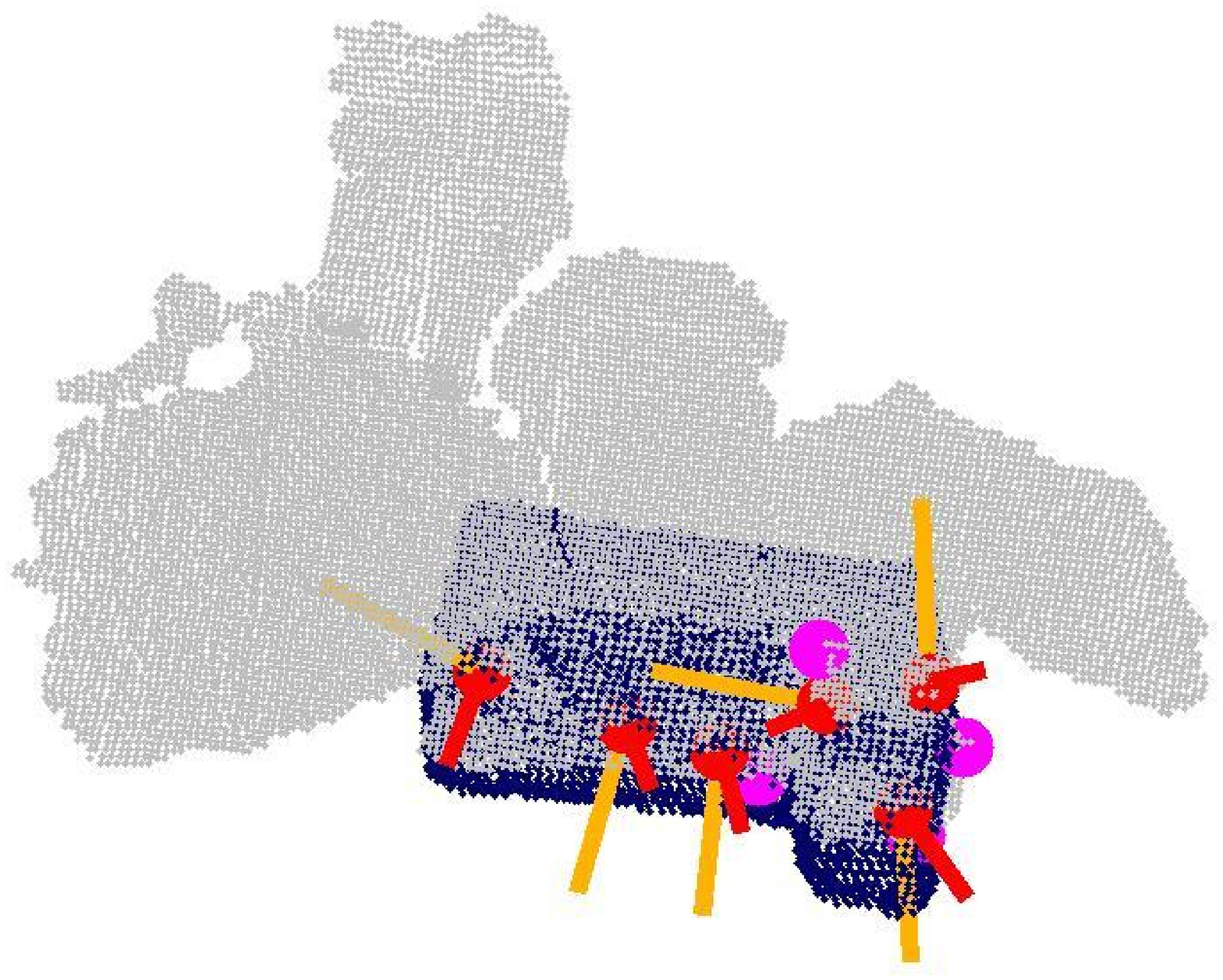}%
    \includegraphics[height=.65in]{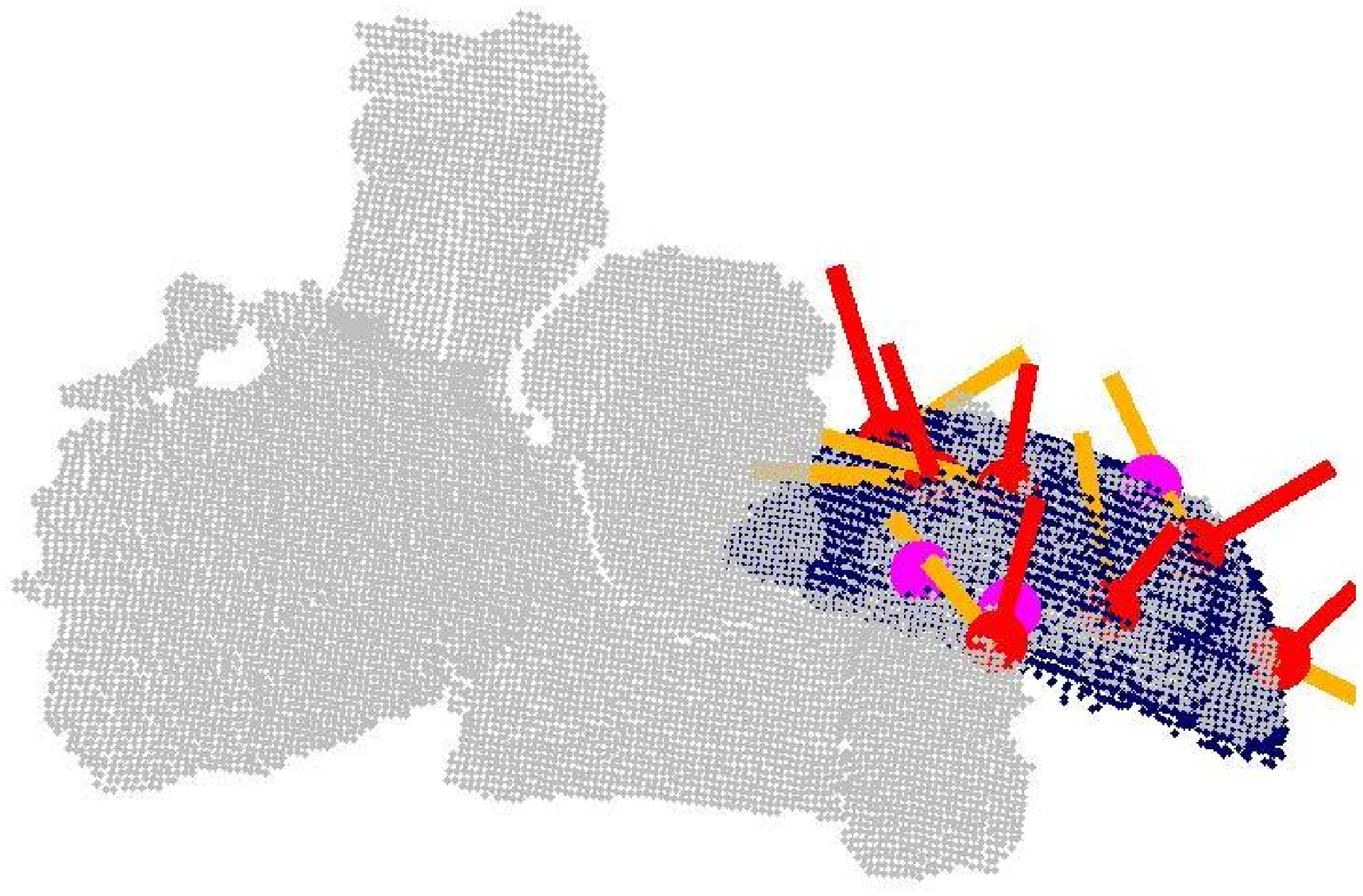} \\ \\
    \includegraphics[height=.65in]{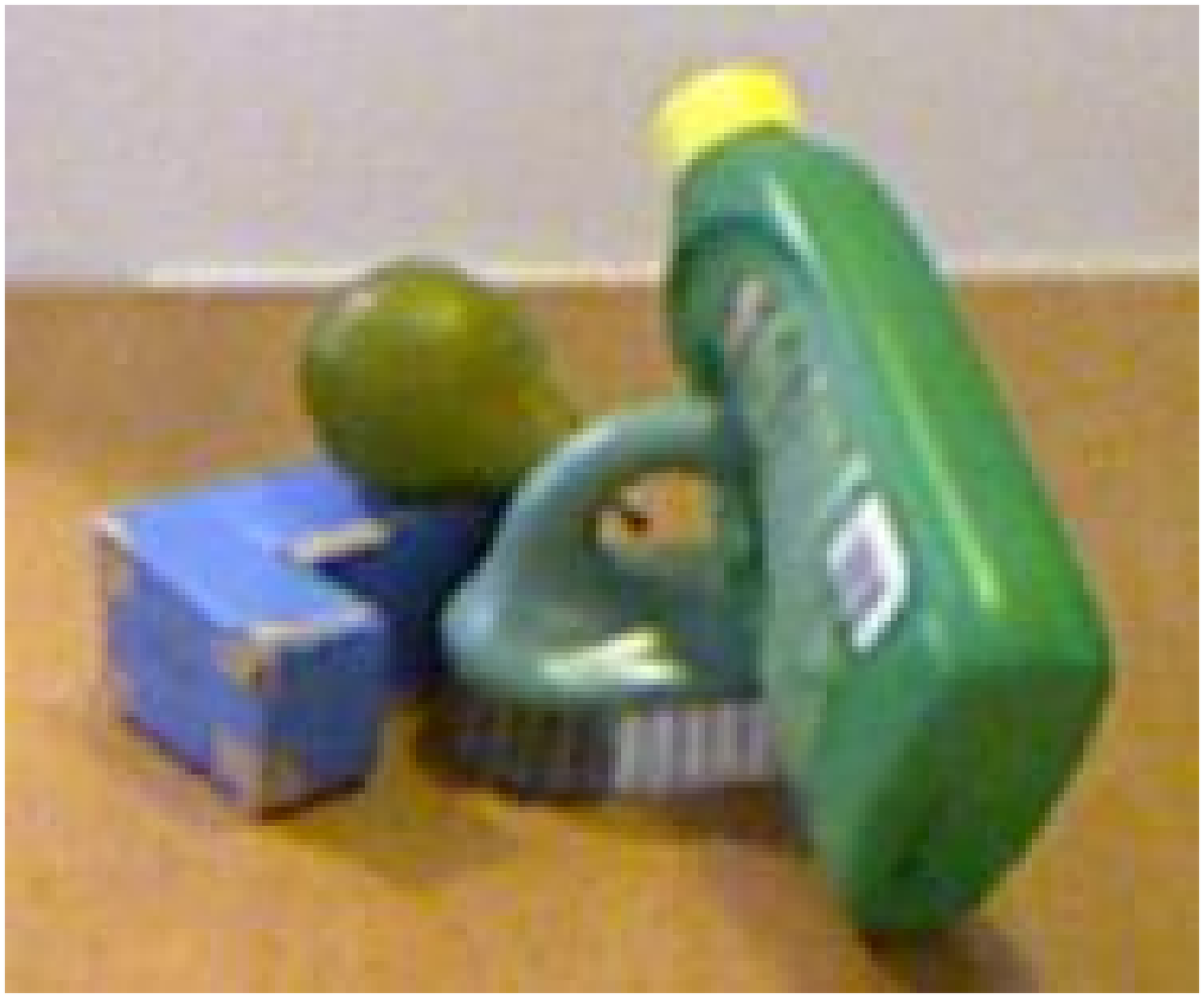}%
    \includegraphics[height=.65in]{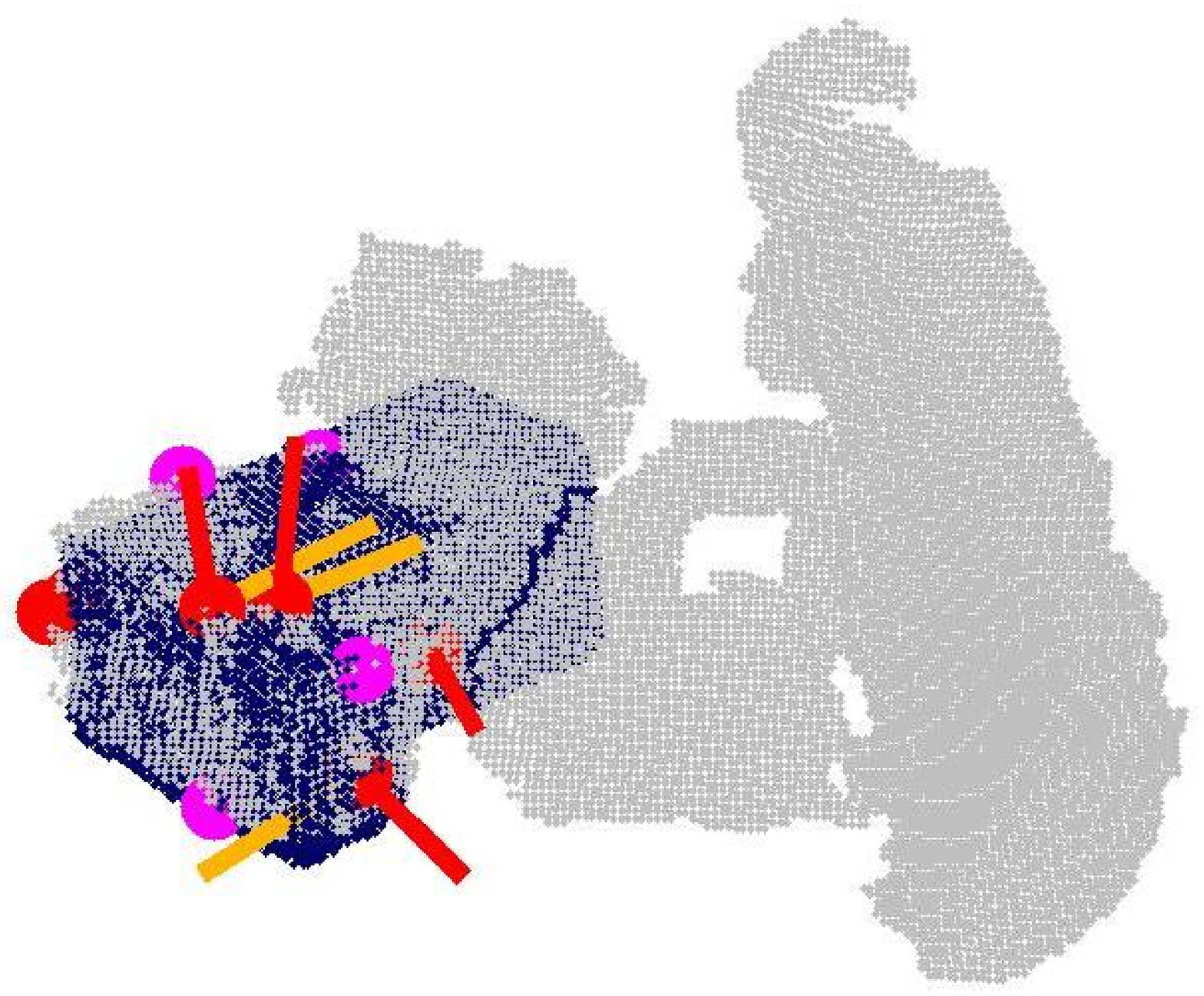}%
    \includegraphics[height=.65in]{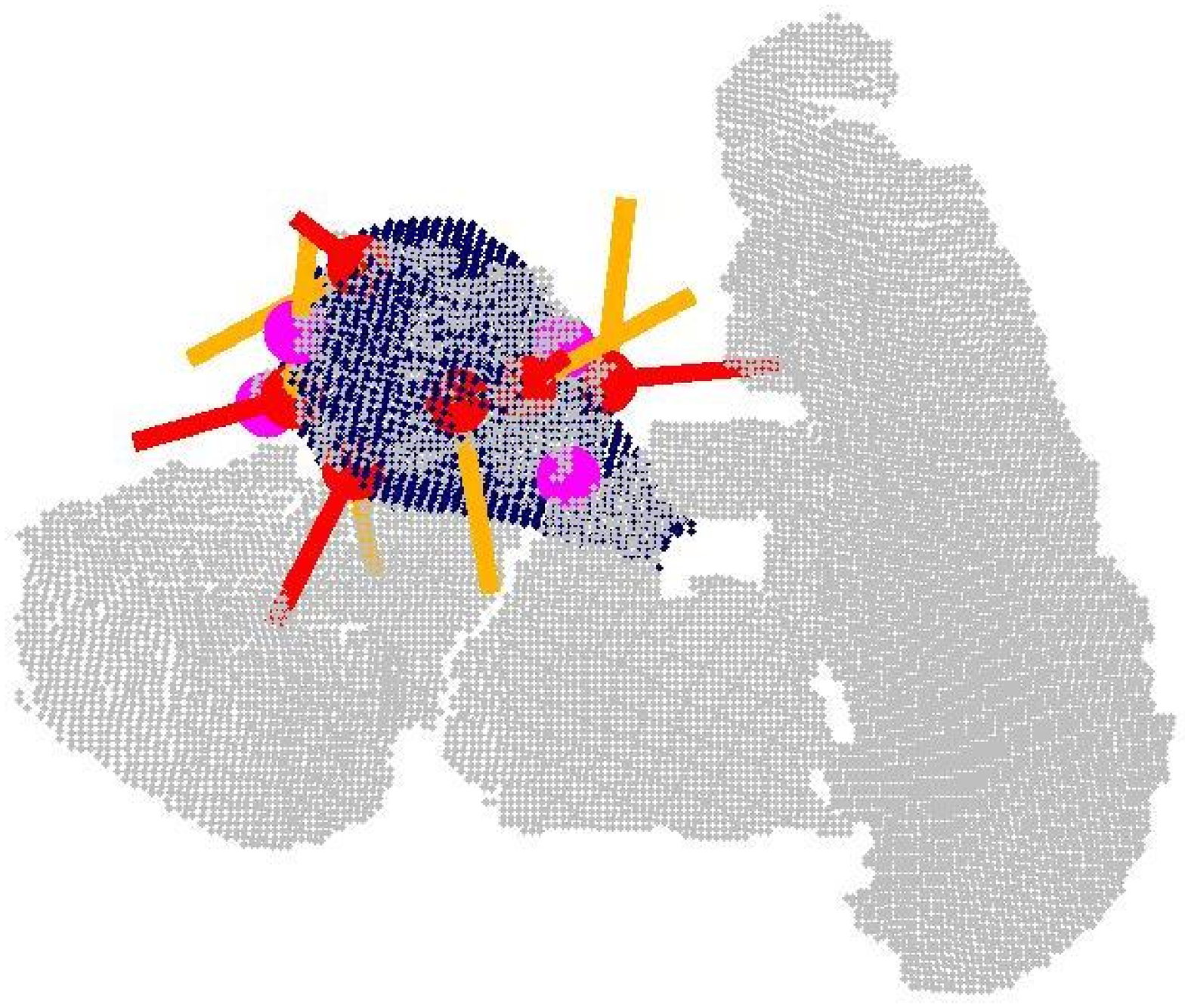}%
    \includegraphics[height=.65in]{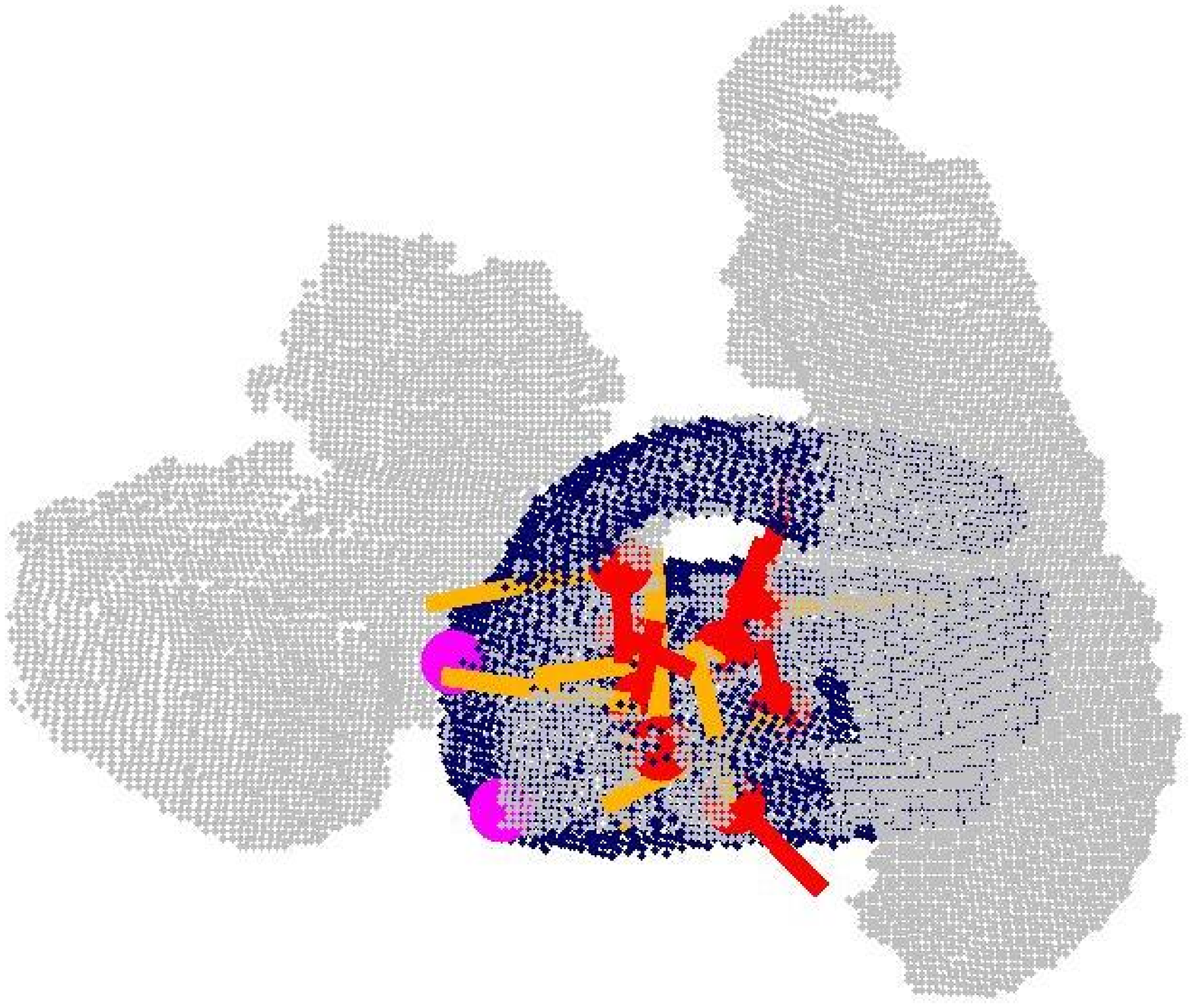}%
    \includegraphics[height=.65in]{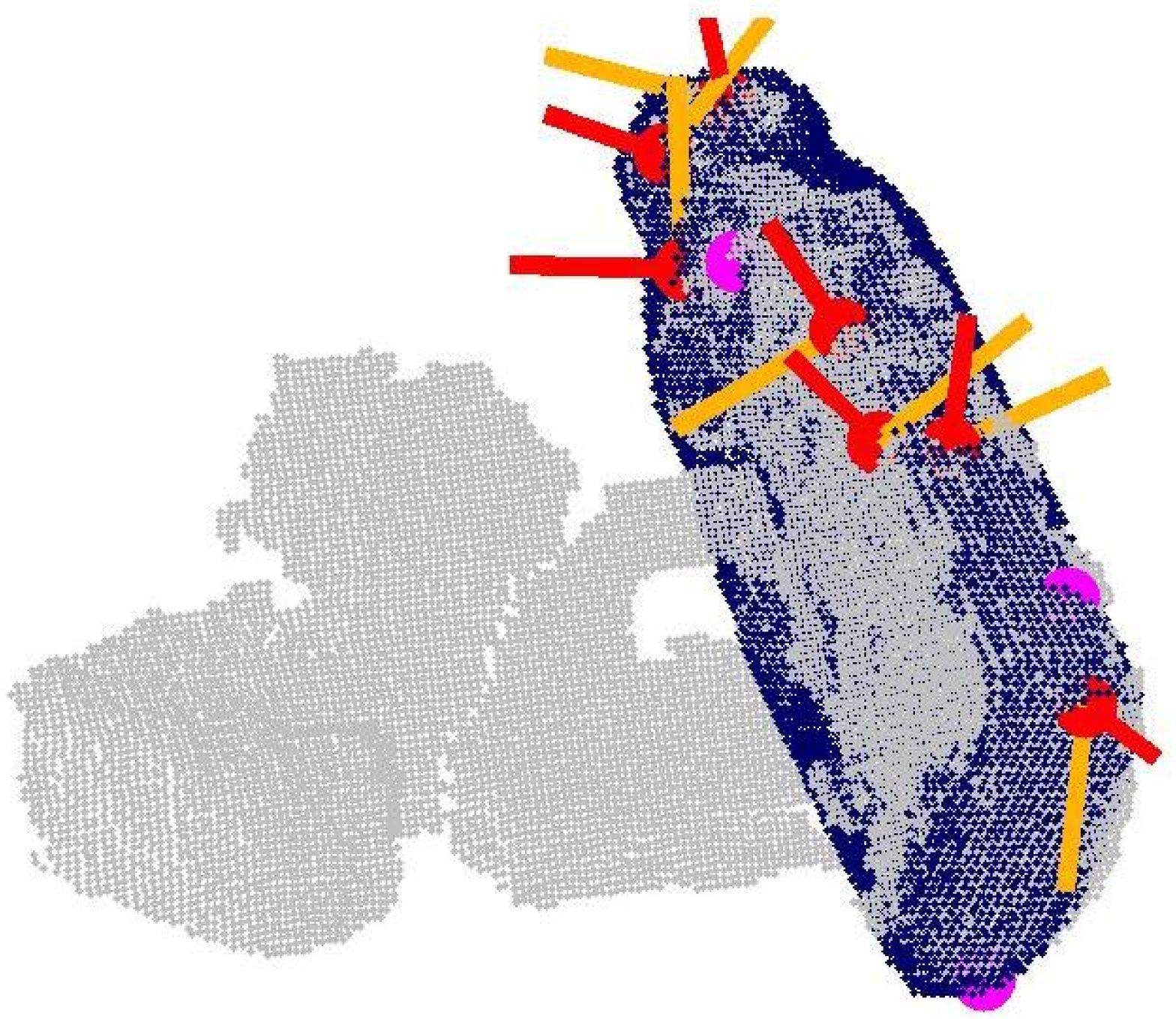}
    \caption{Object detections found with our system, along with the feature
      correspondences that BPA used to align the model.  Surface features are indicated by
      red points, with lines sticking out of them to indicate orientations (red for normals,
      orange for principal curvatures).  Edge features (which are orientation-less) are shown
      by magenta points.}
  \label{fig:scope_samples2}
\end{figure*}

\begin{figure}[t!]
  \centering
  \includegraphics[width=.049\textwidth]{figs/sope_samples/cloud_3_1.ps}%
  \includegraphics[width=.049\textwidth]{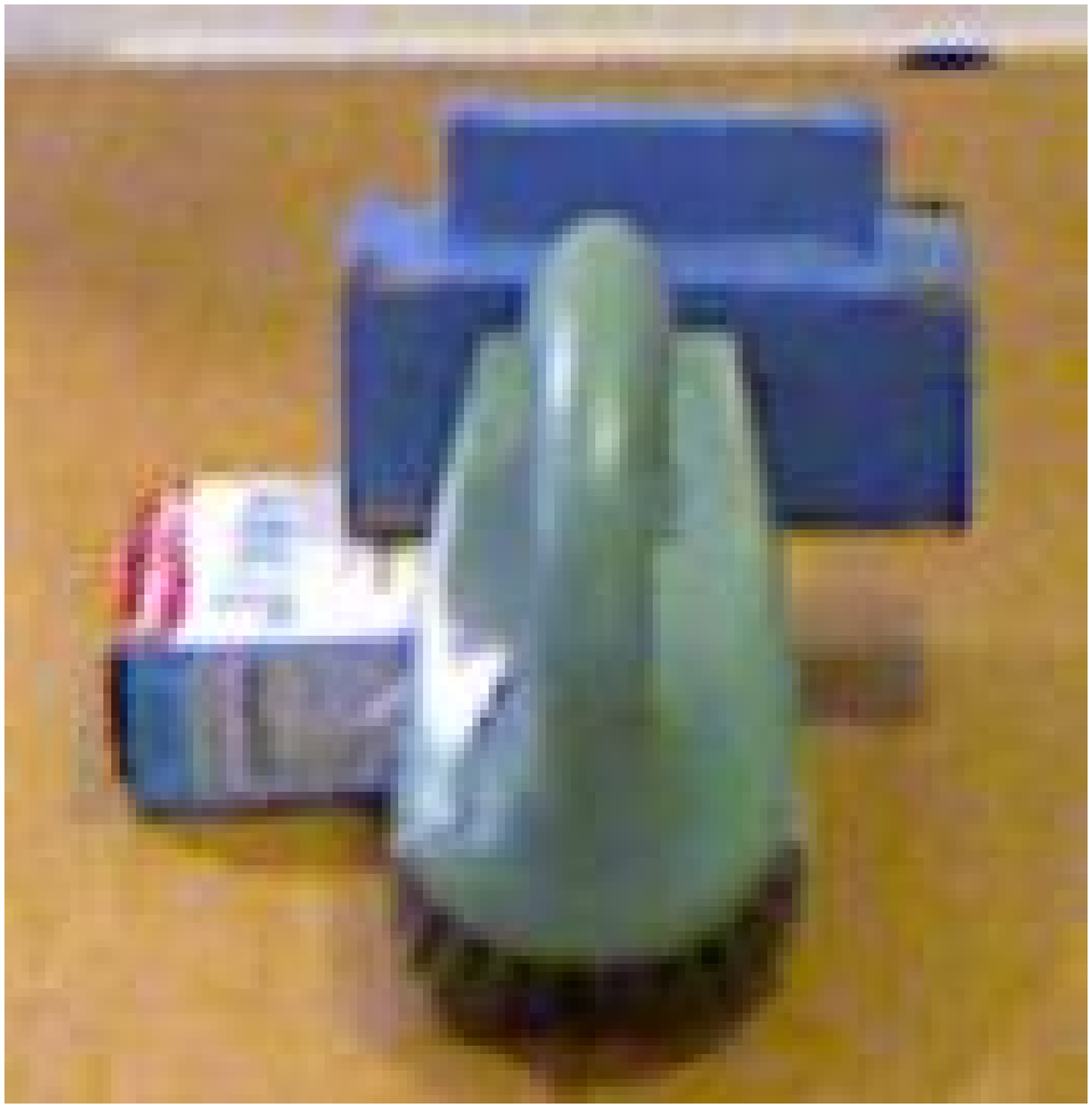}%
  \includegraphics[width=.049\textwidth]{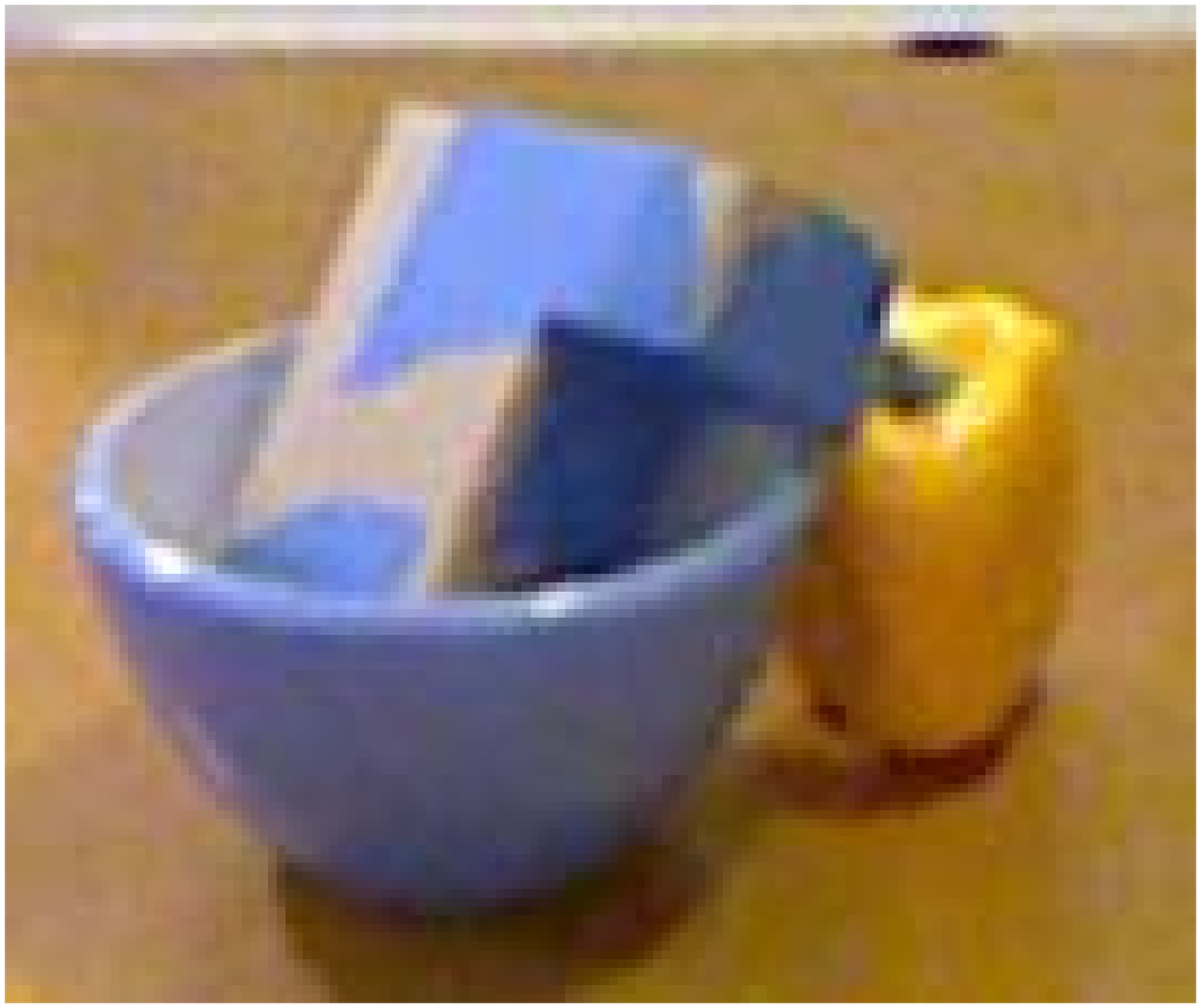}%
  \includegraphics[width=.035\textwidth]{figs/sope_samples/cloud_3_4.ps}%
  \includegraphics[width=.049\textwidth]{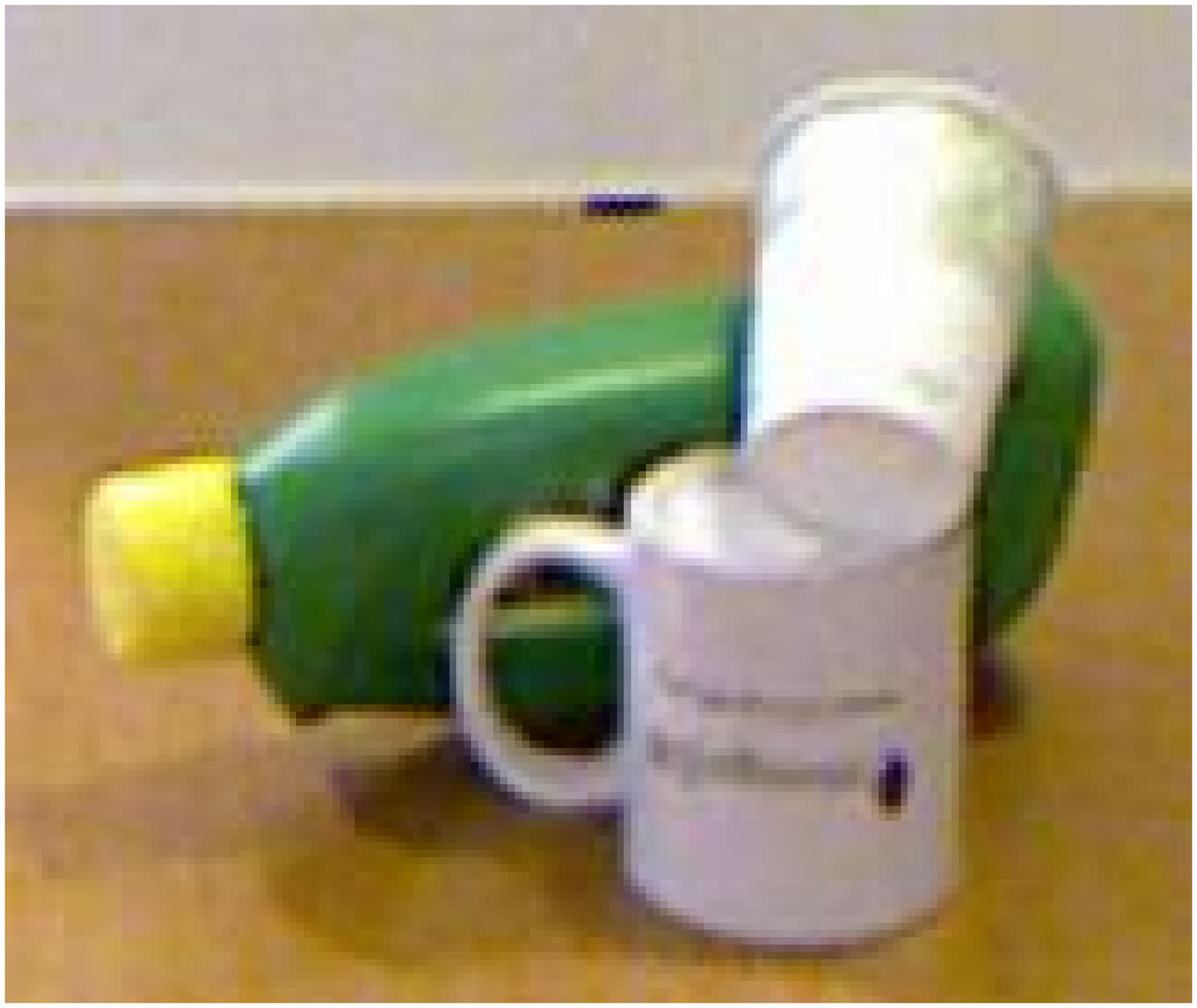}%
  \includegraphics[width=.049\textwidth]{figs/sope_samples/cloud_3_6.ps}%
  \includegraphics[width=.049\textwidth]{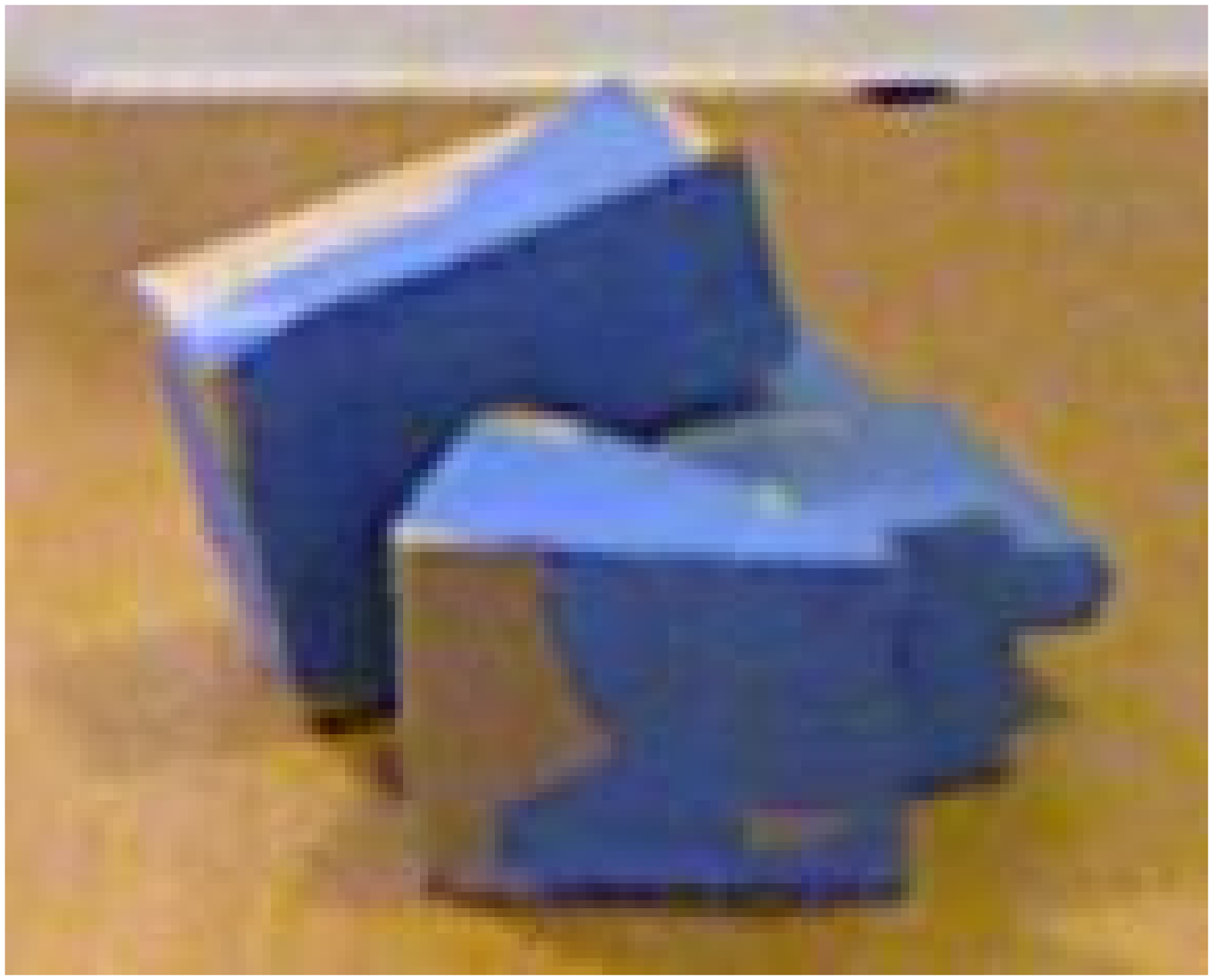}%
  \includegraphics[width=.049\textwidth]{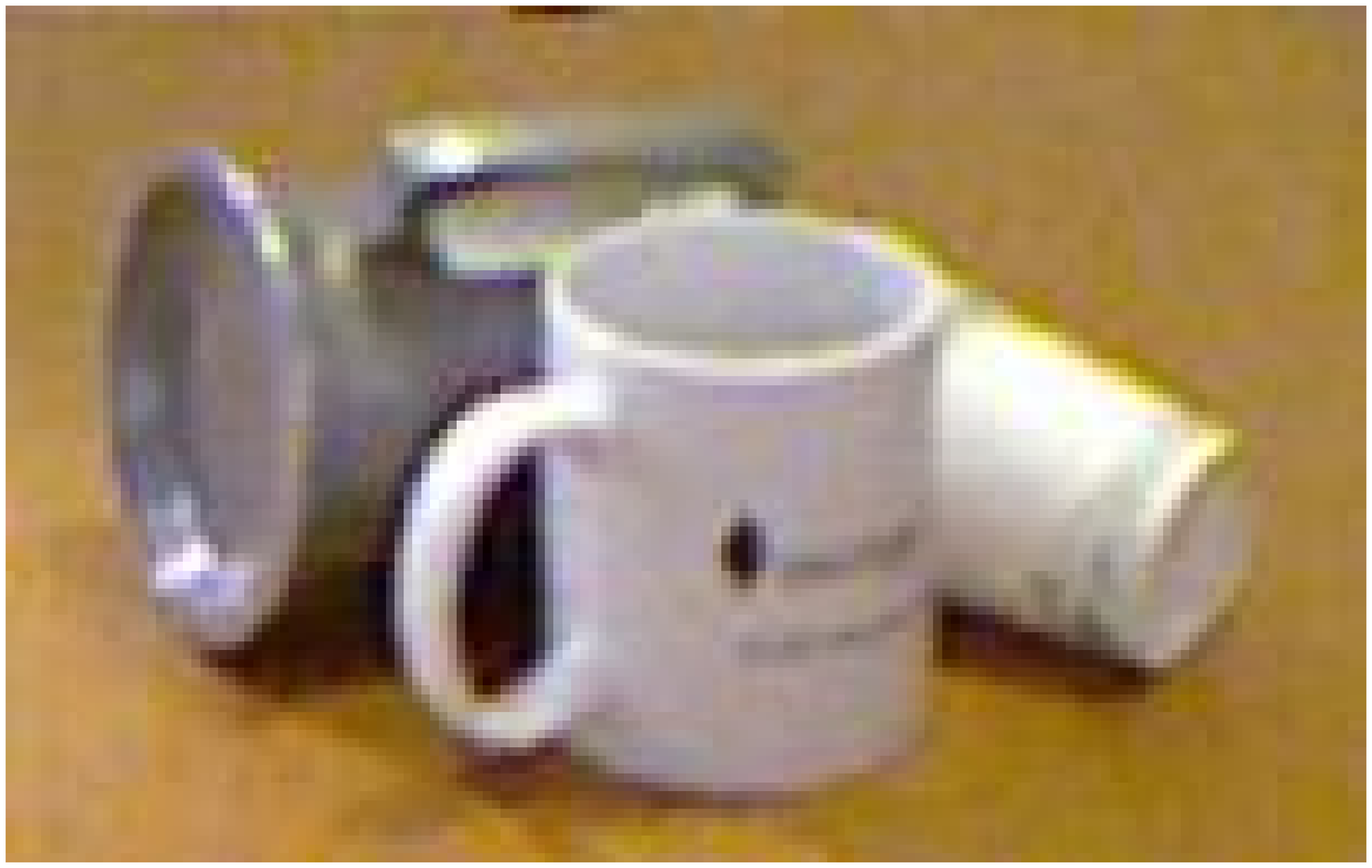}%
  \includegraphics[width=.049\textwidth]{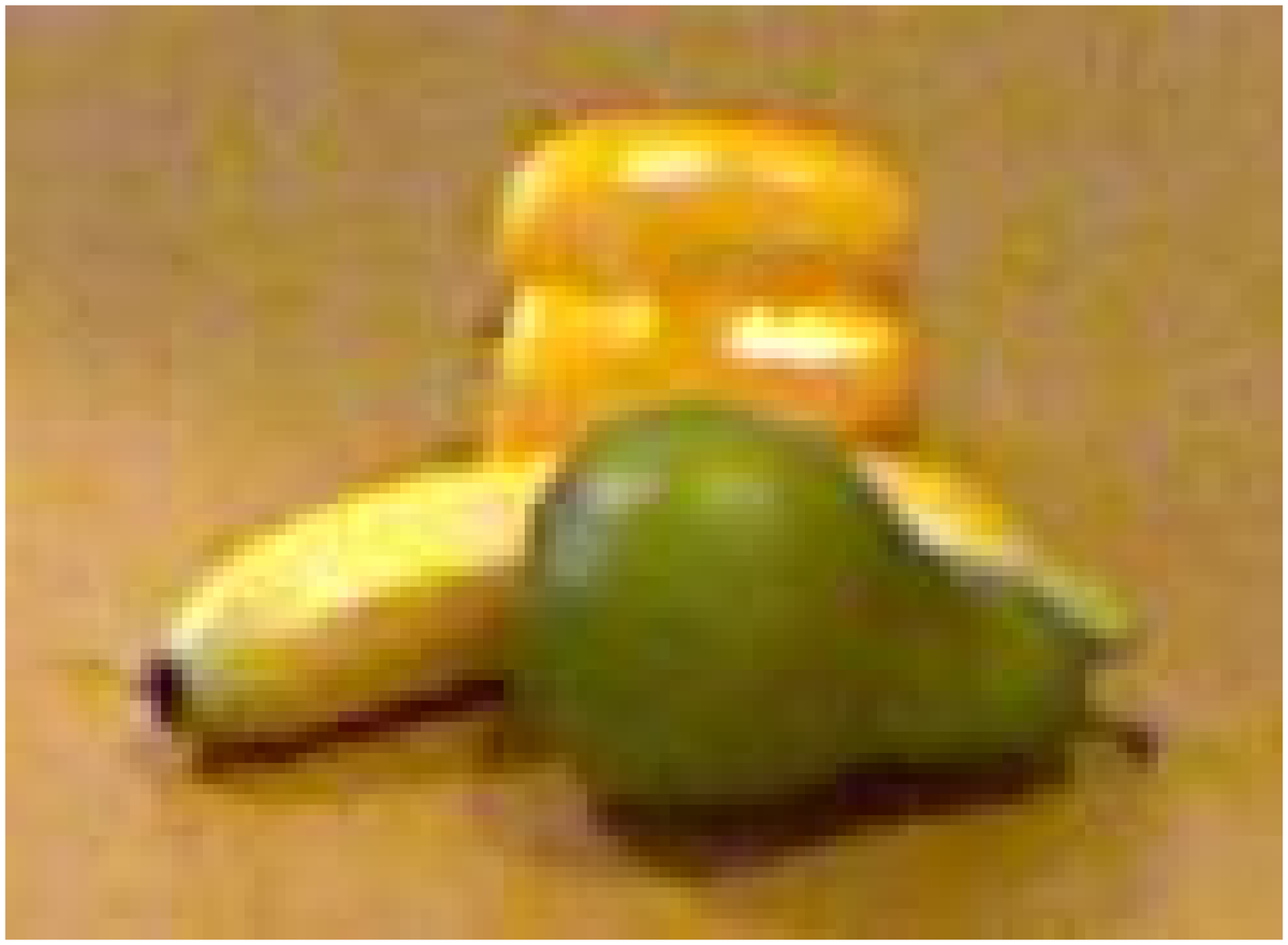}%
  \includegraphics[width=.049\textwidth]{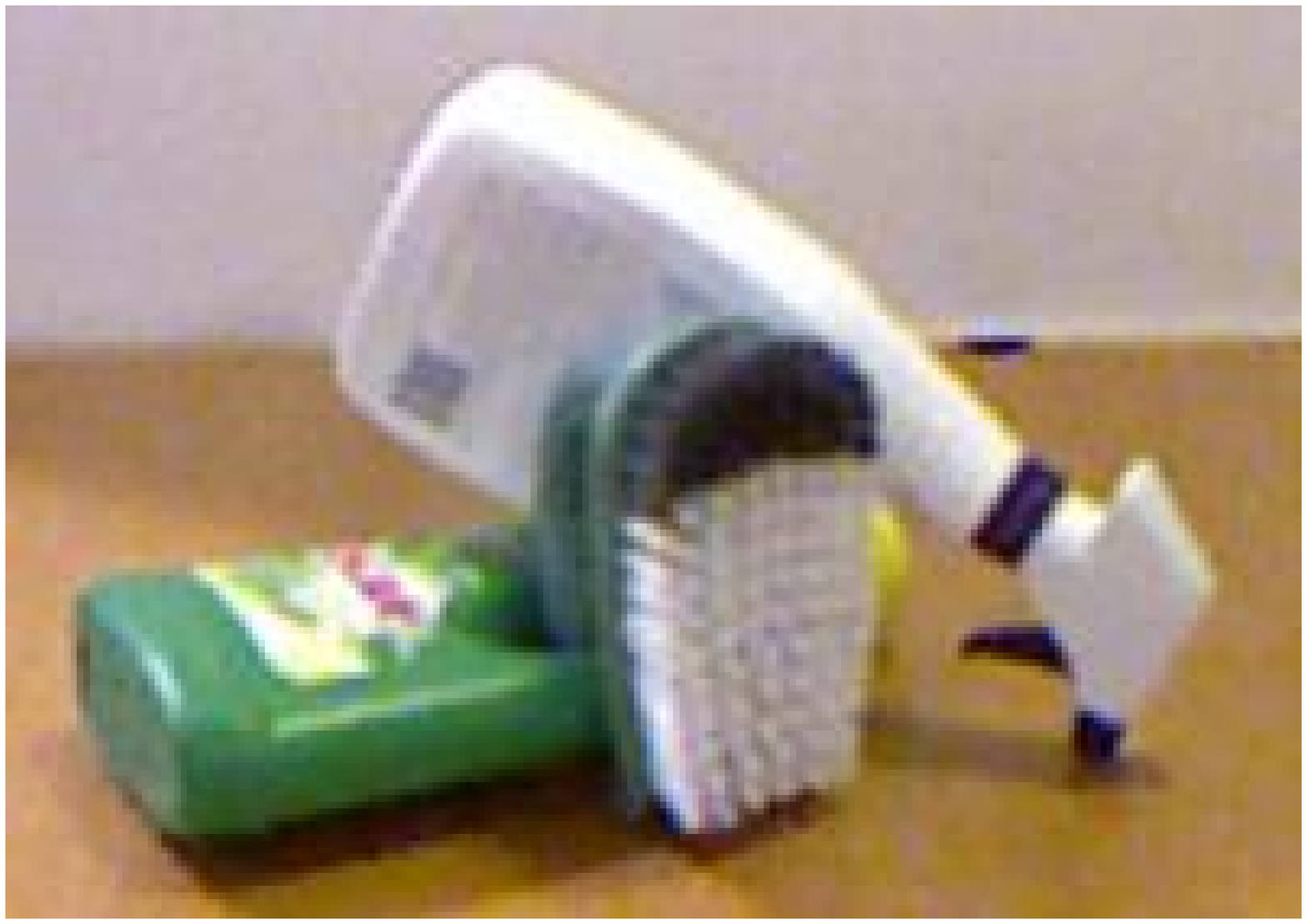}
  \\
  \includegraphics[width=.049\textwidth]{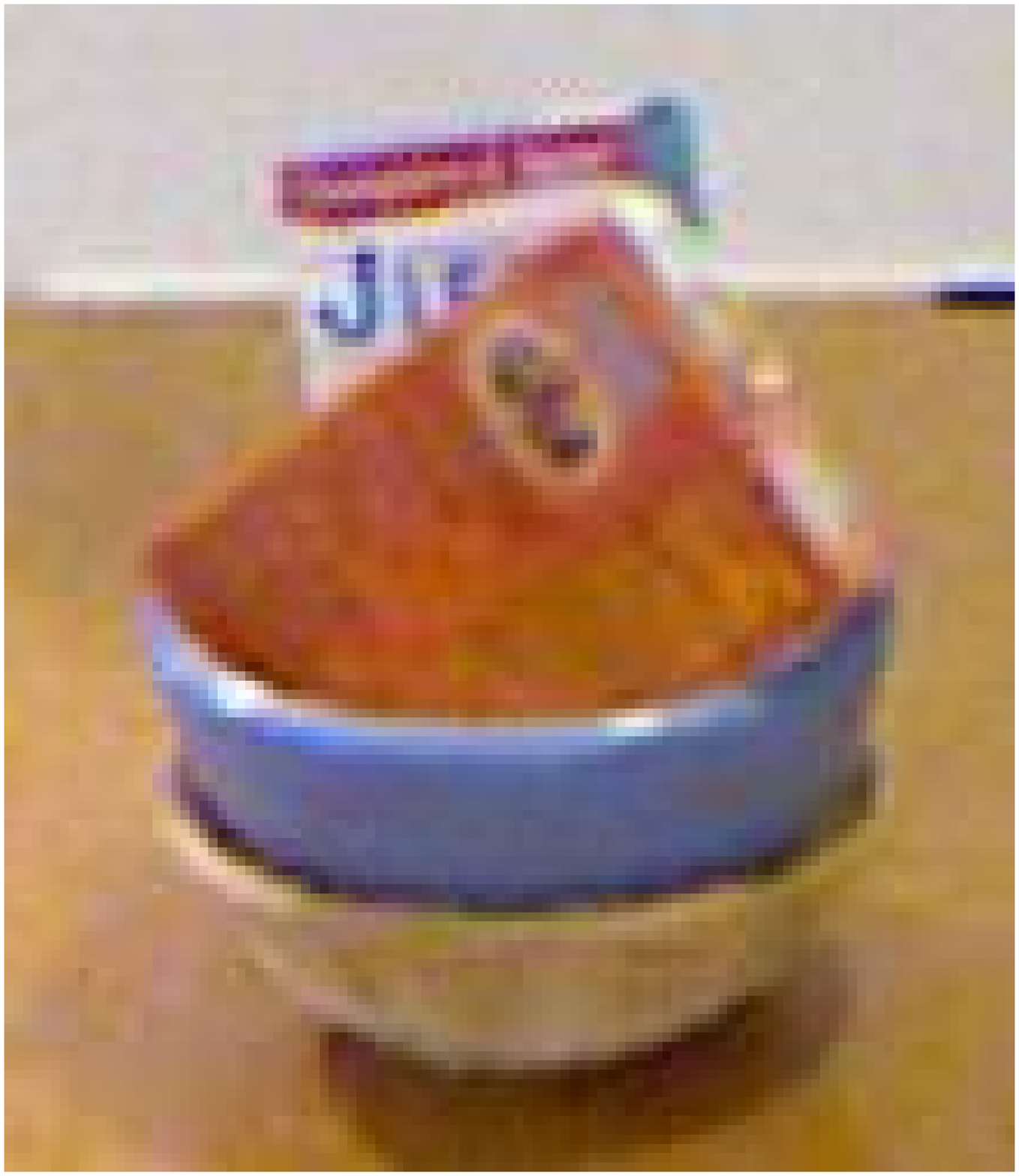}%
  \includegraphics[width=.06\textwidth]{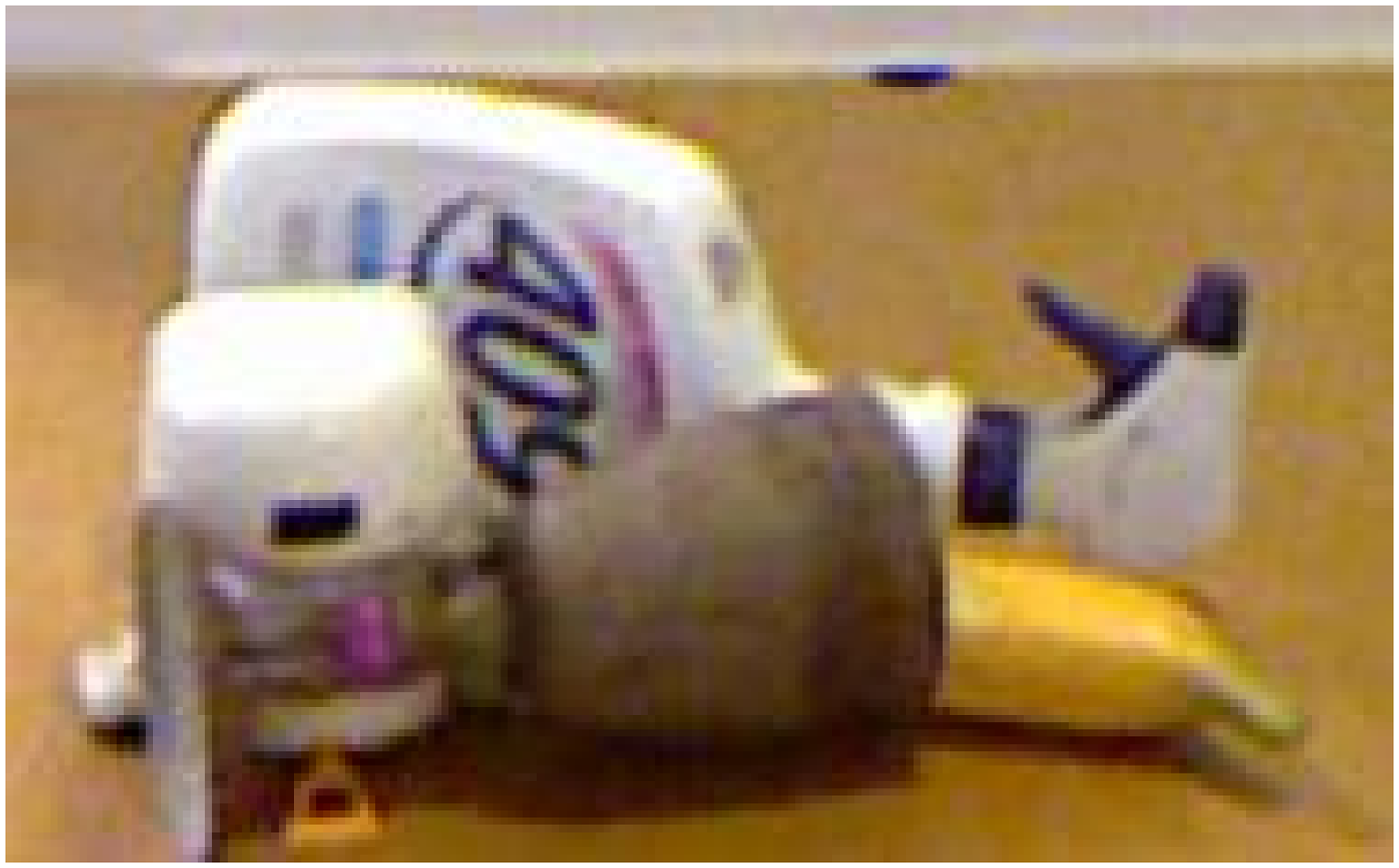}%
  \includegraphics[width=.049\textwidth]{figs/sope_samples/cloud_4_3.ps}%
  \includegraphics[width=.025\textwidth]{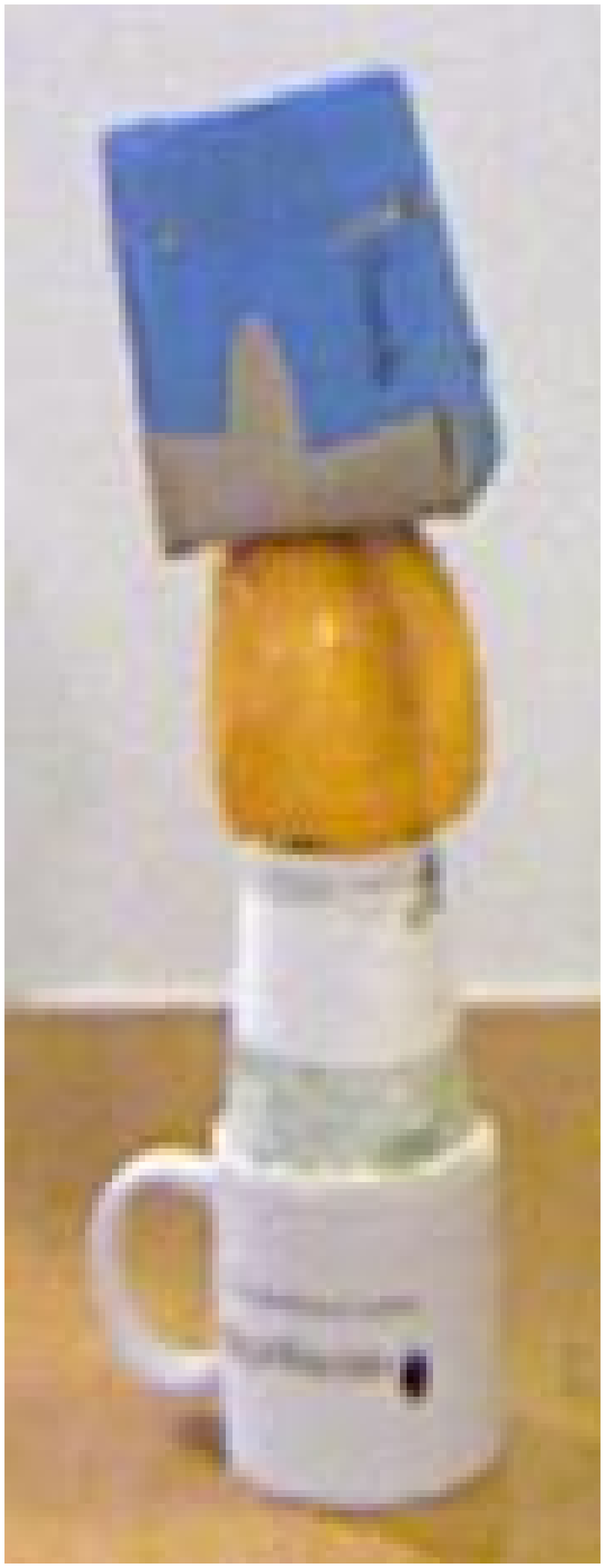}%
  \includegraphics[width=.03\textwidth]{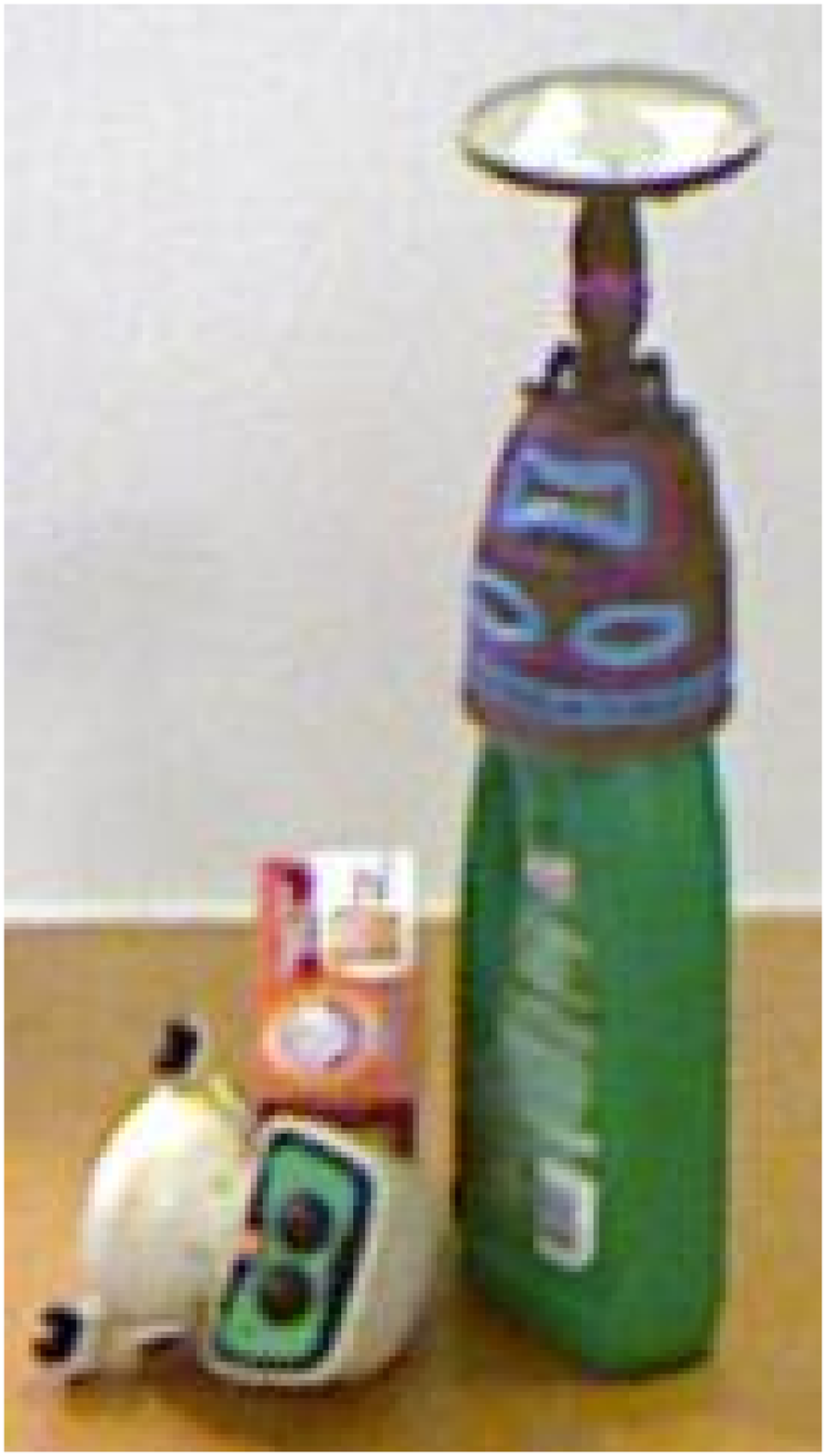}%
  \includegraphics[width=.049\textwidth]{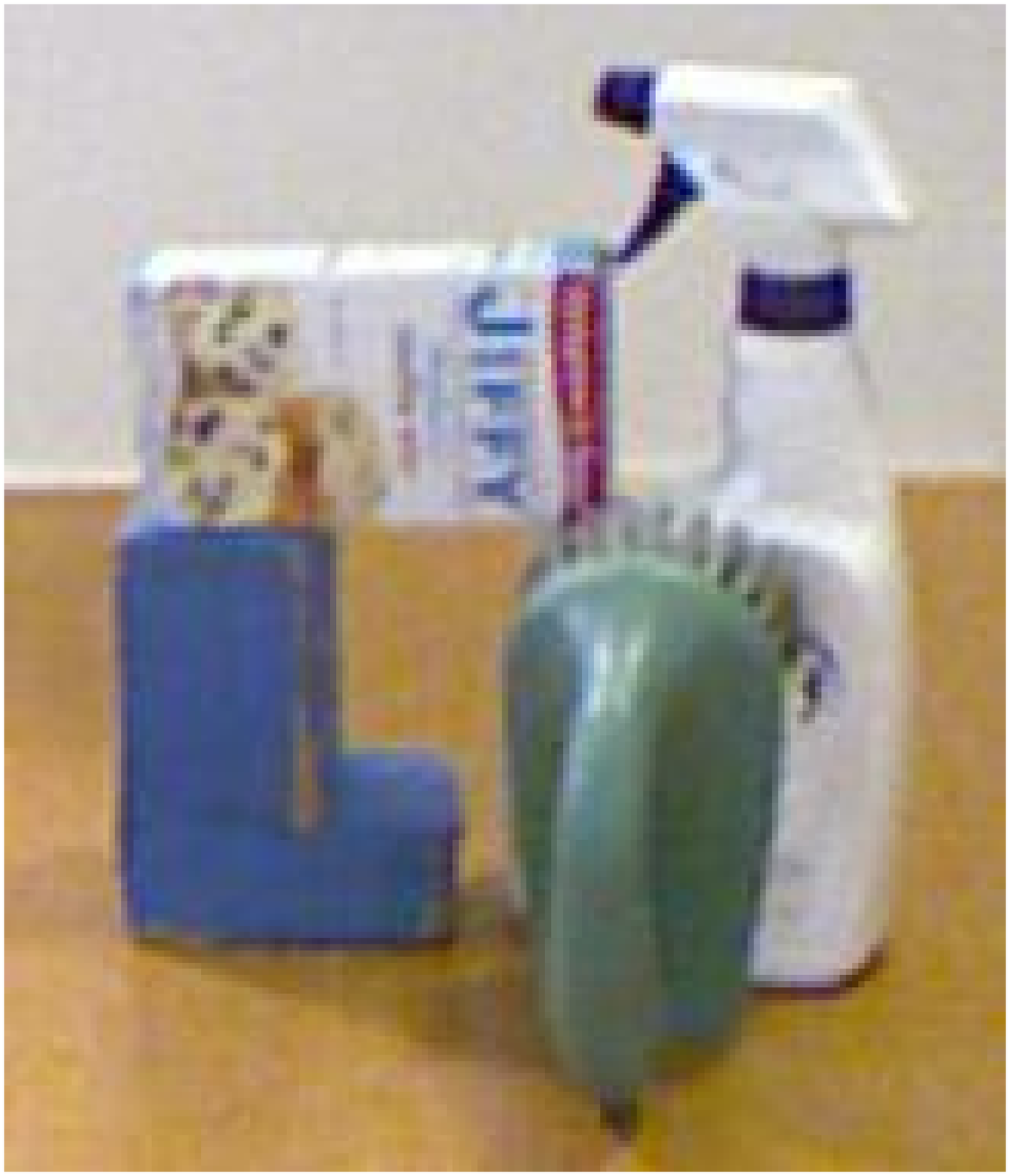}%
  \includegraphics[width=.049\textwidth]{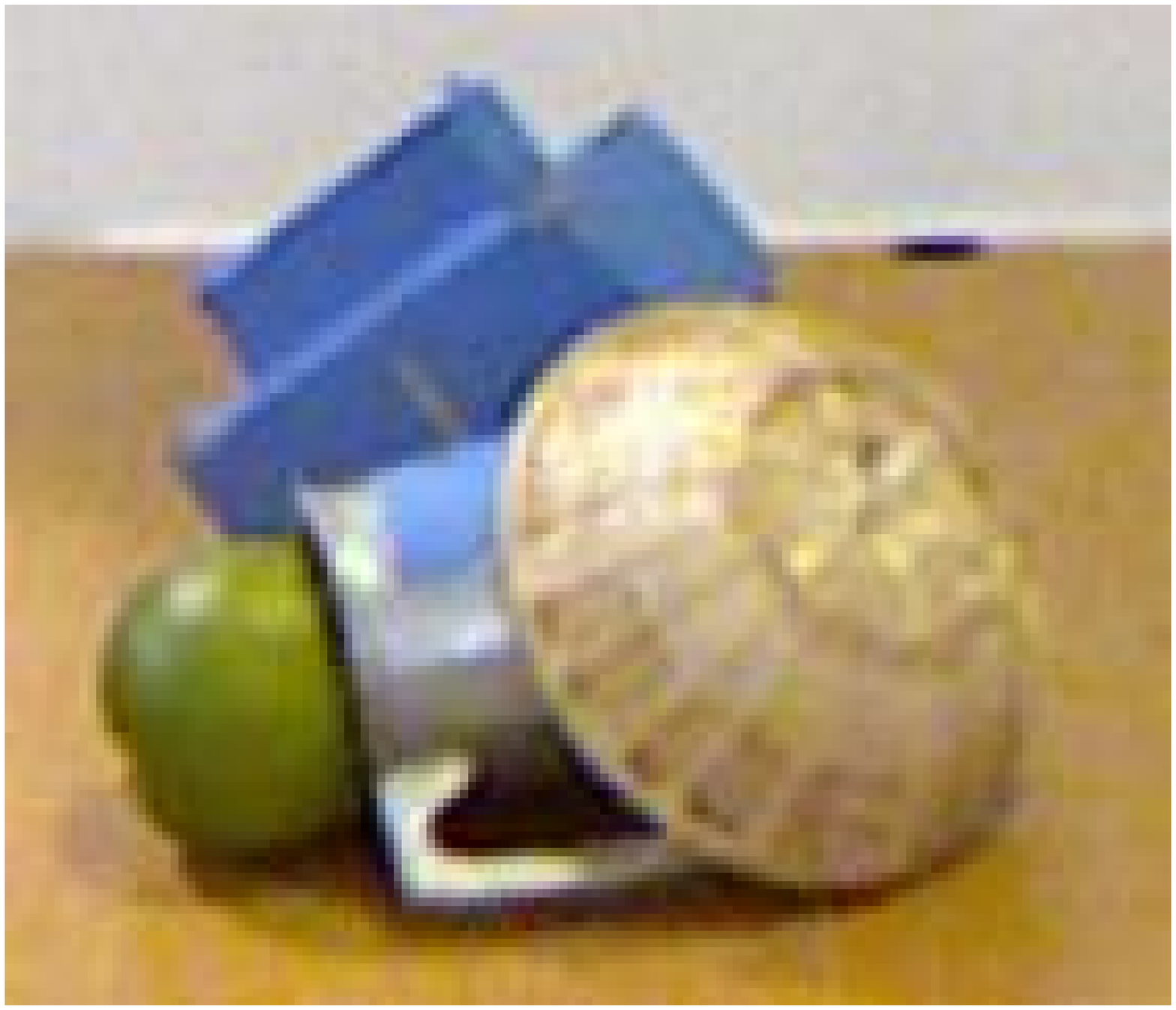}%
  \includegraphics[width=.049\textwidth]{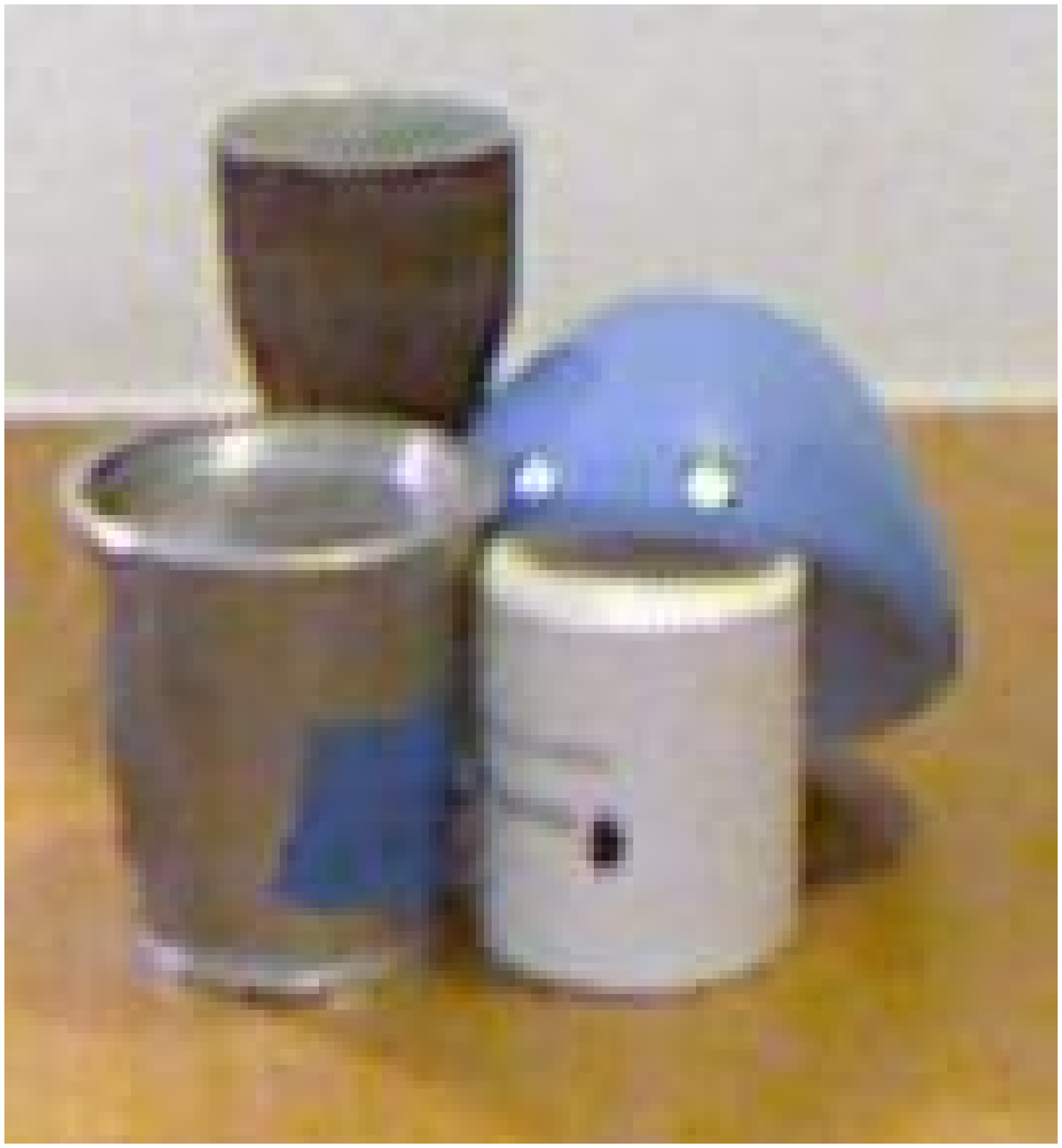}%
  \includegraphics[width=.049\textwidth]{figs/sope_samples/cloud_4_9.ps}%
  \includegraphics[width=.045\textwidth]{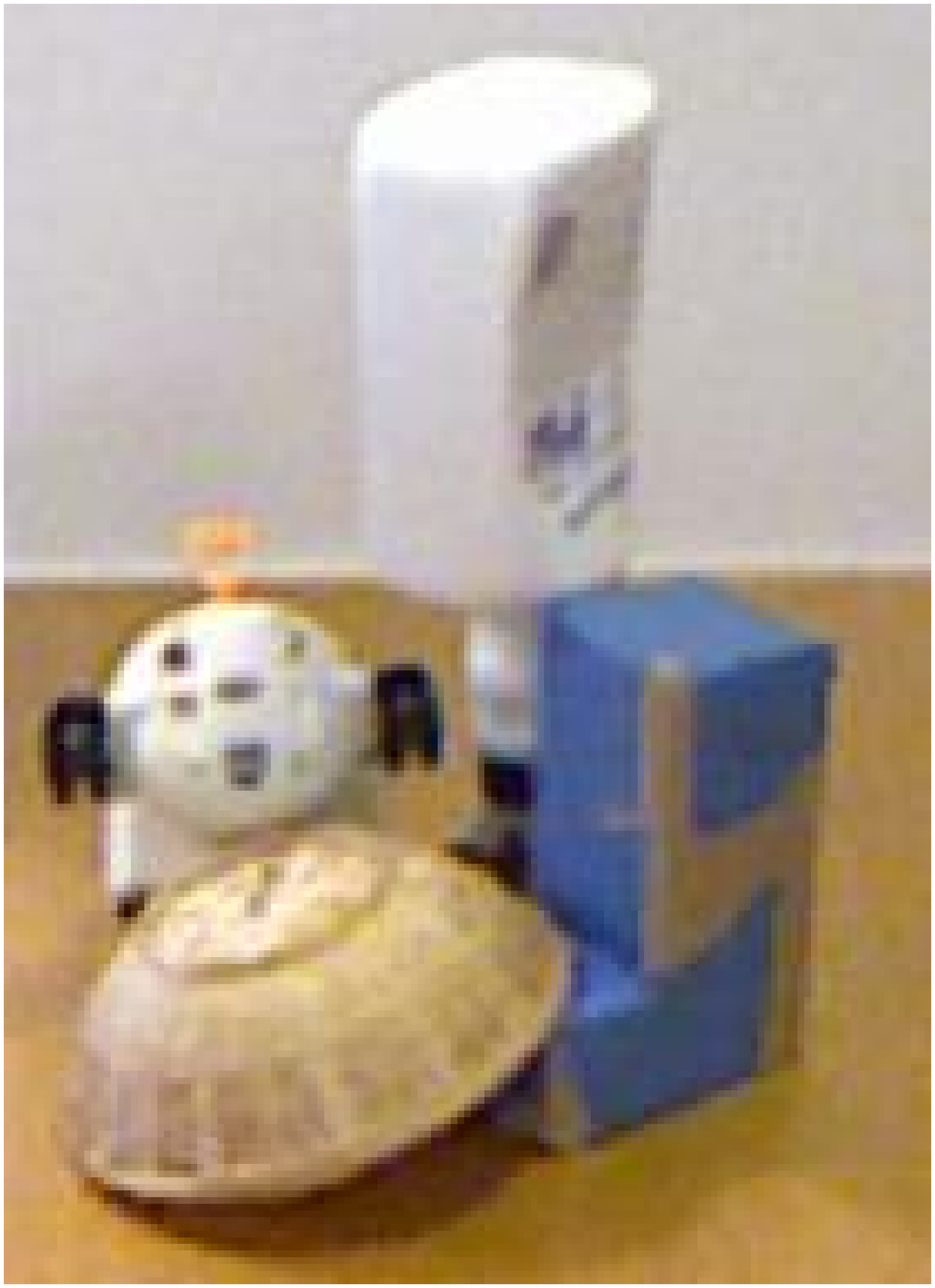}%
  \\
  \includegraphics[width=.049\textwidth]{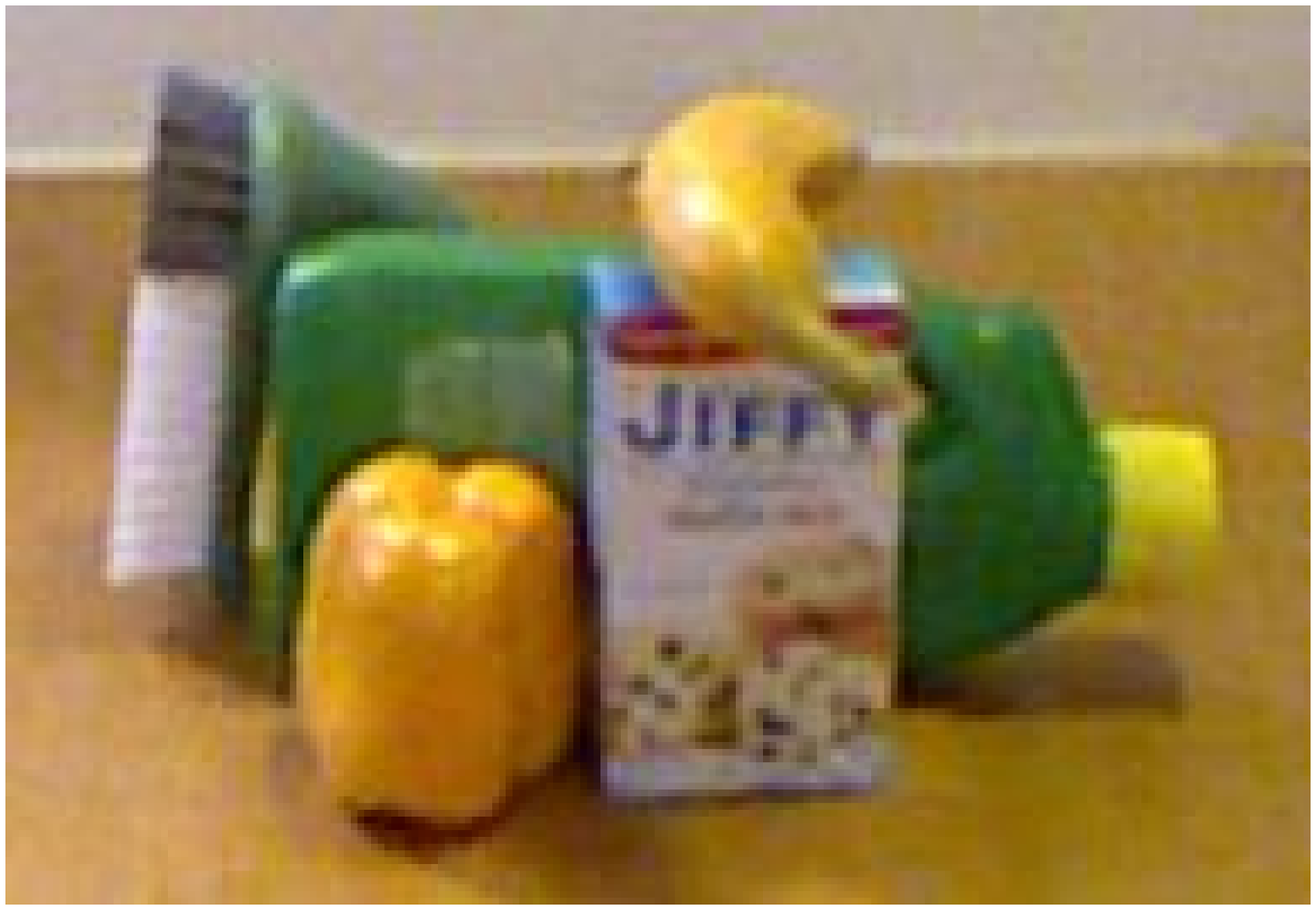}%
  \includegraphics[width=.045\textwidth]{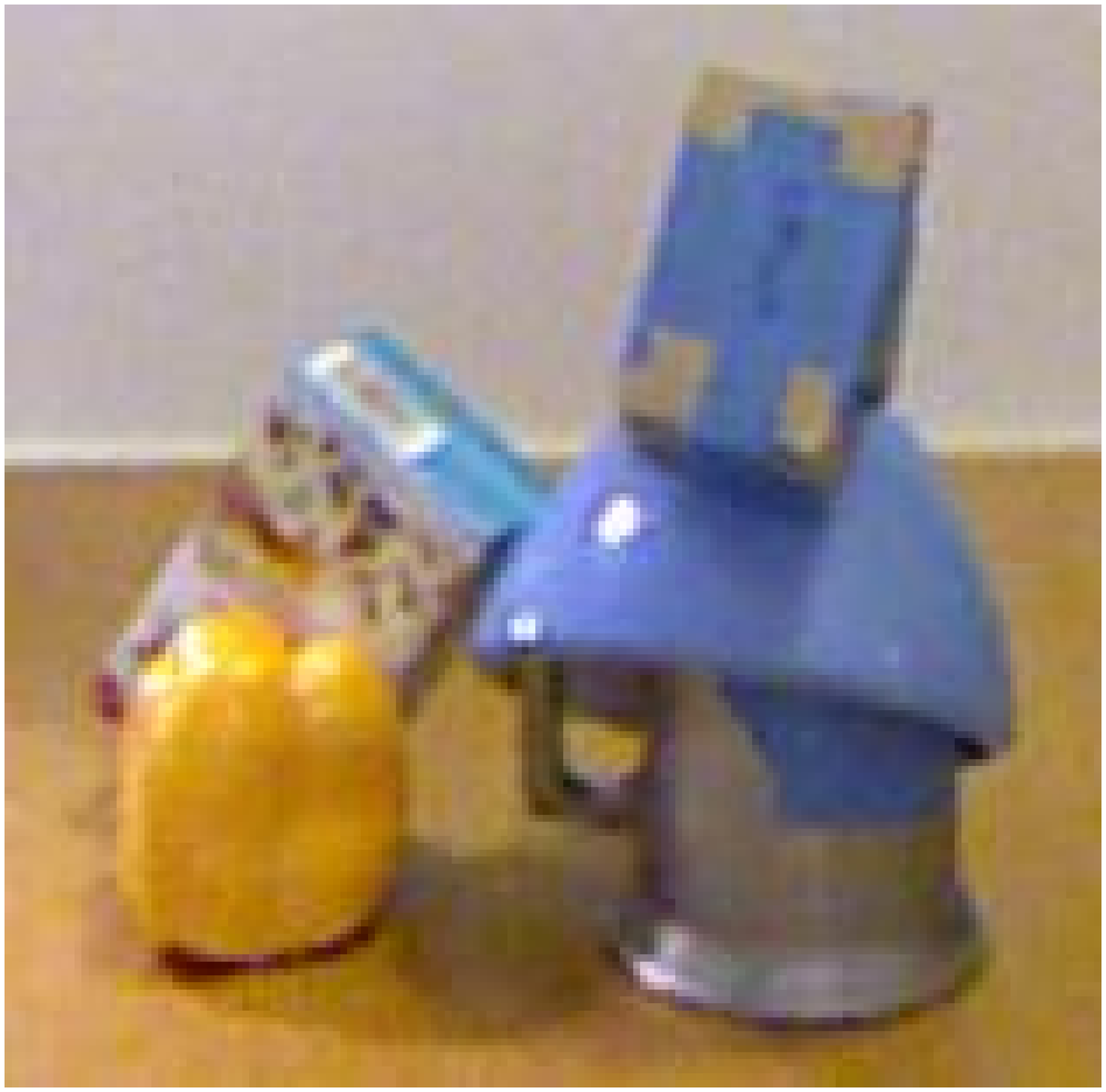}%
  \includegraphics[width=.049\textwidth]{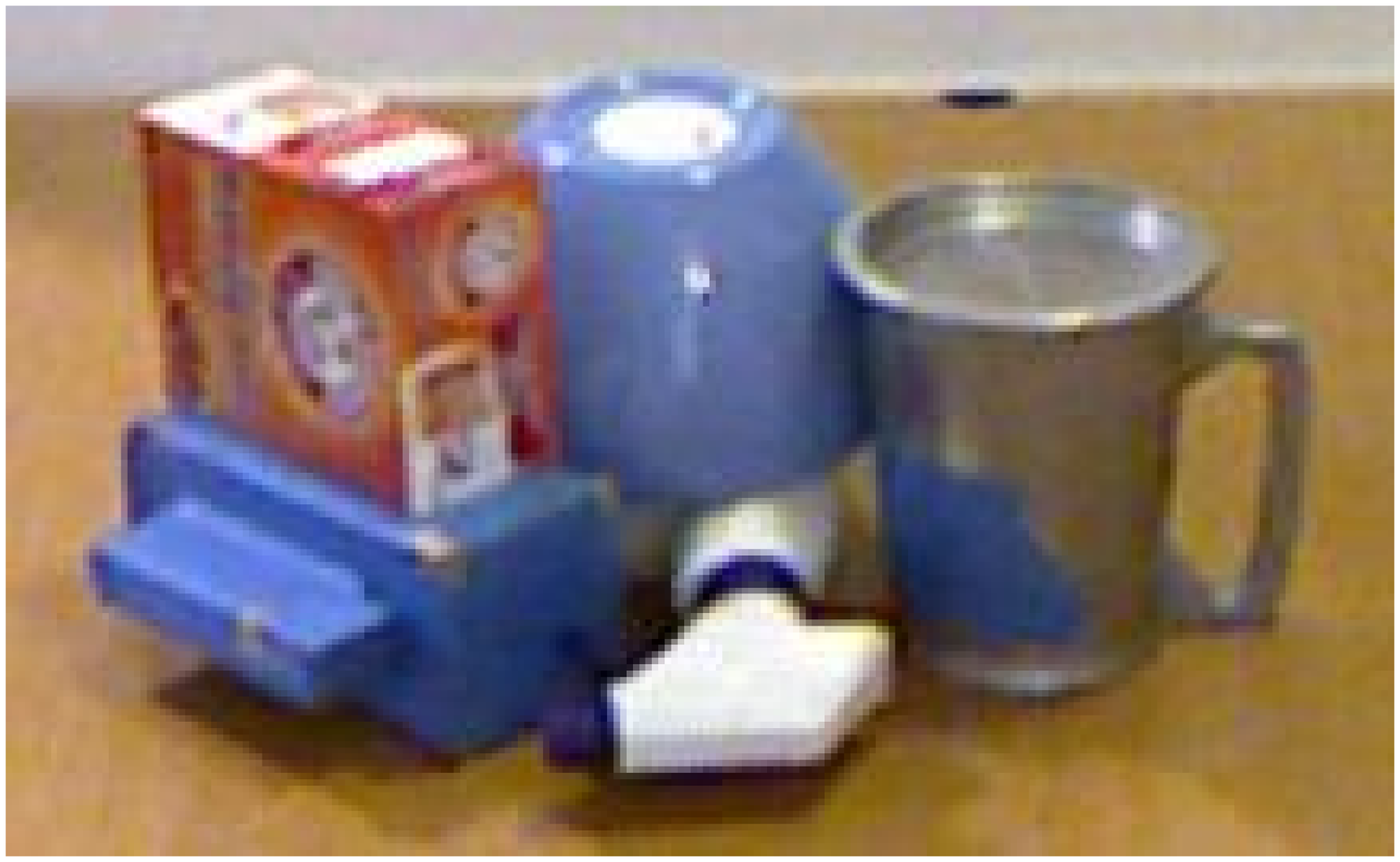}%
  \includegraphics[width=.049\textwidth]{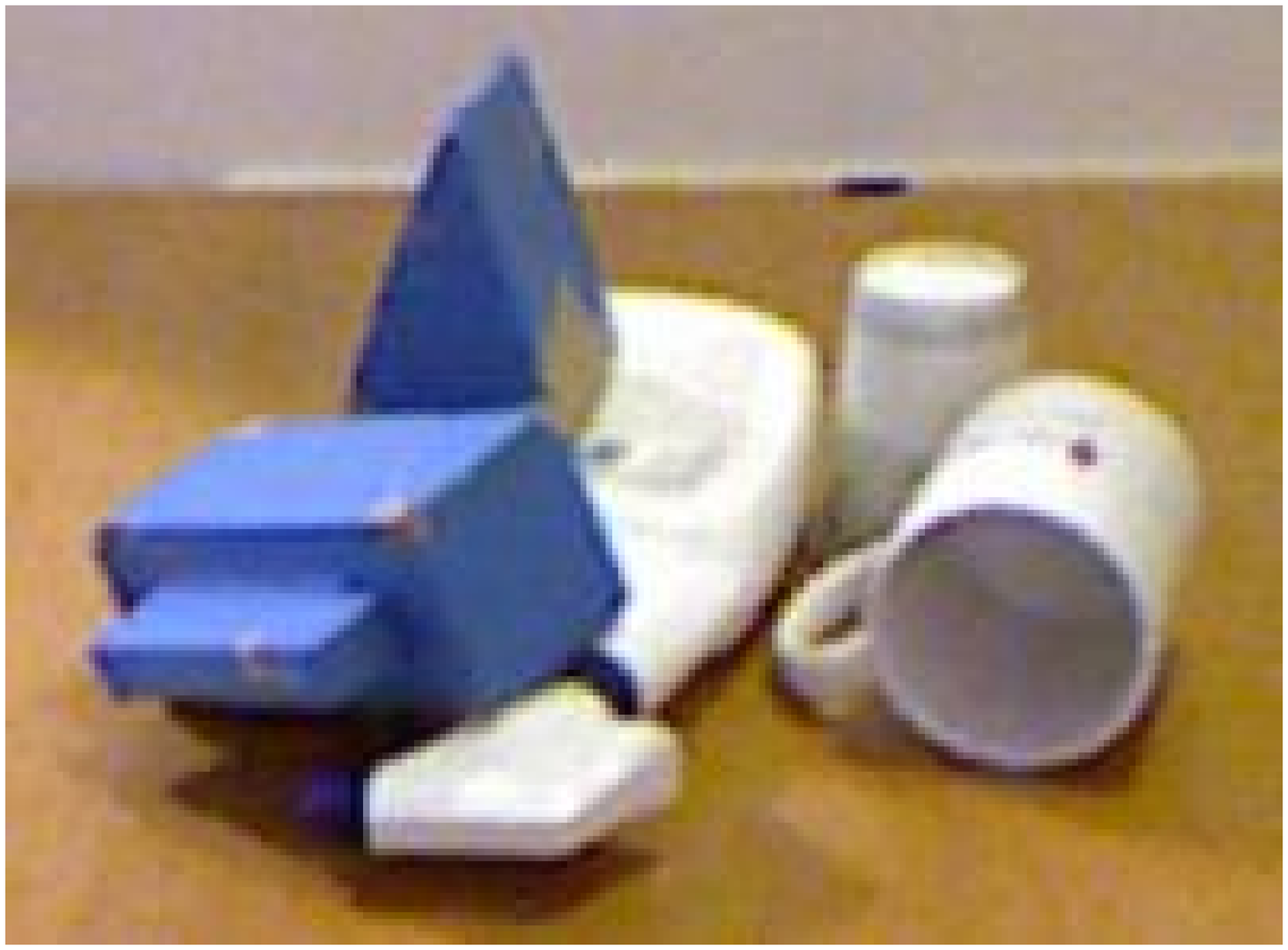}%
  \includegraphics[width=.049\textwidth]{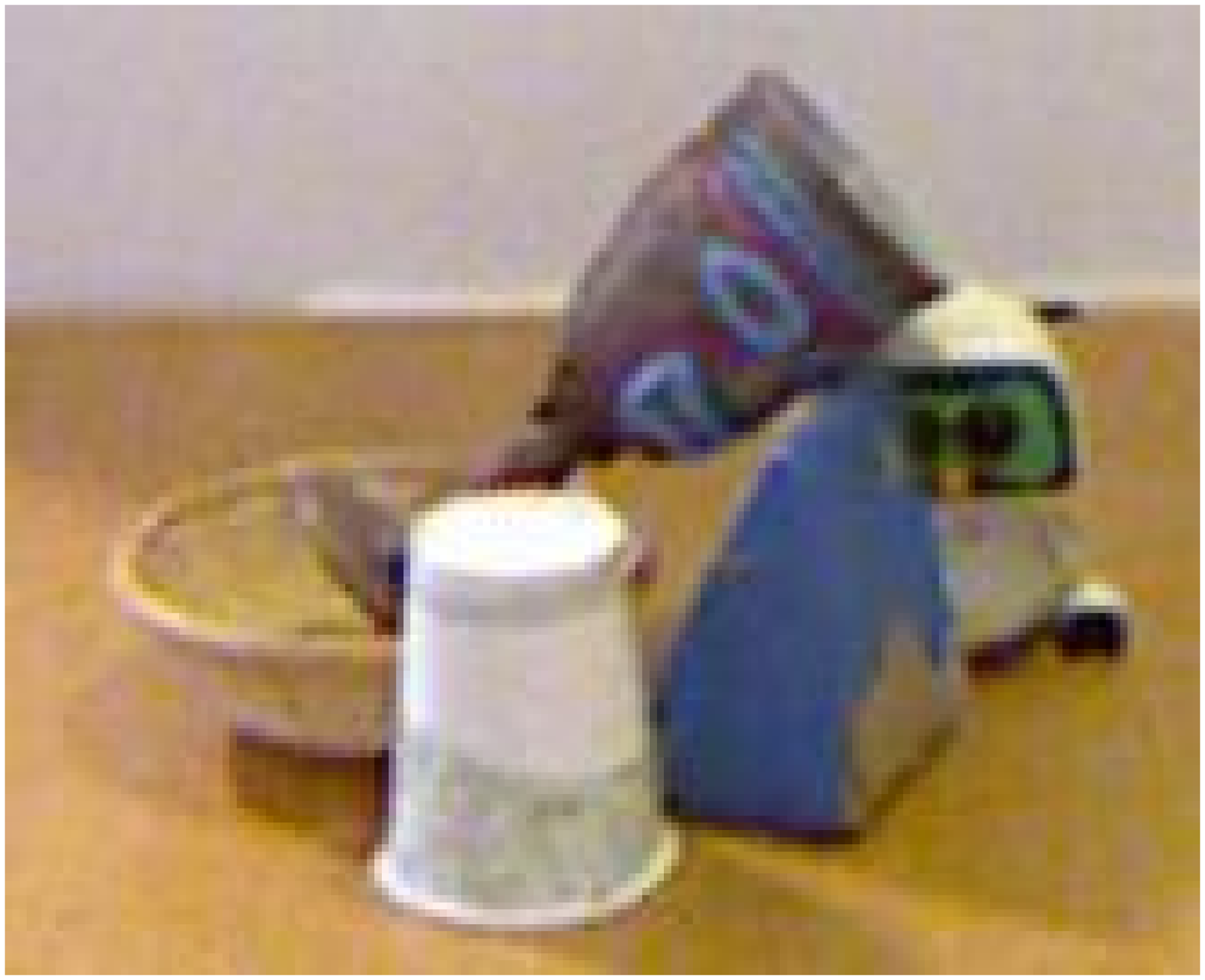}%
  \includegraphics[width=.049\textwidth]{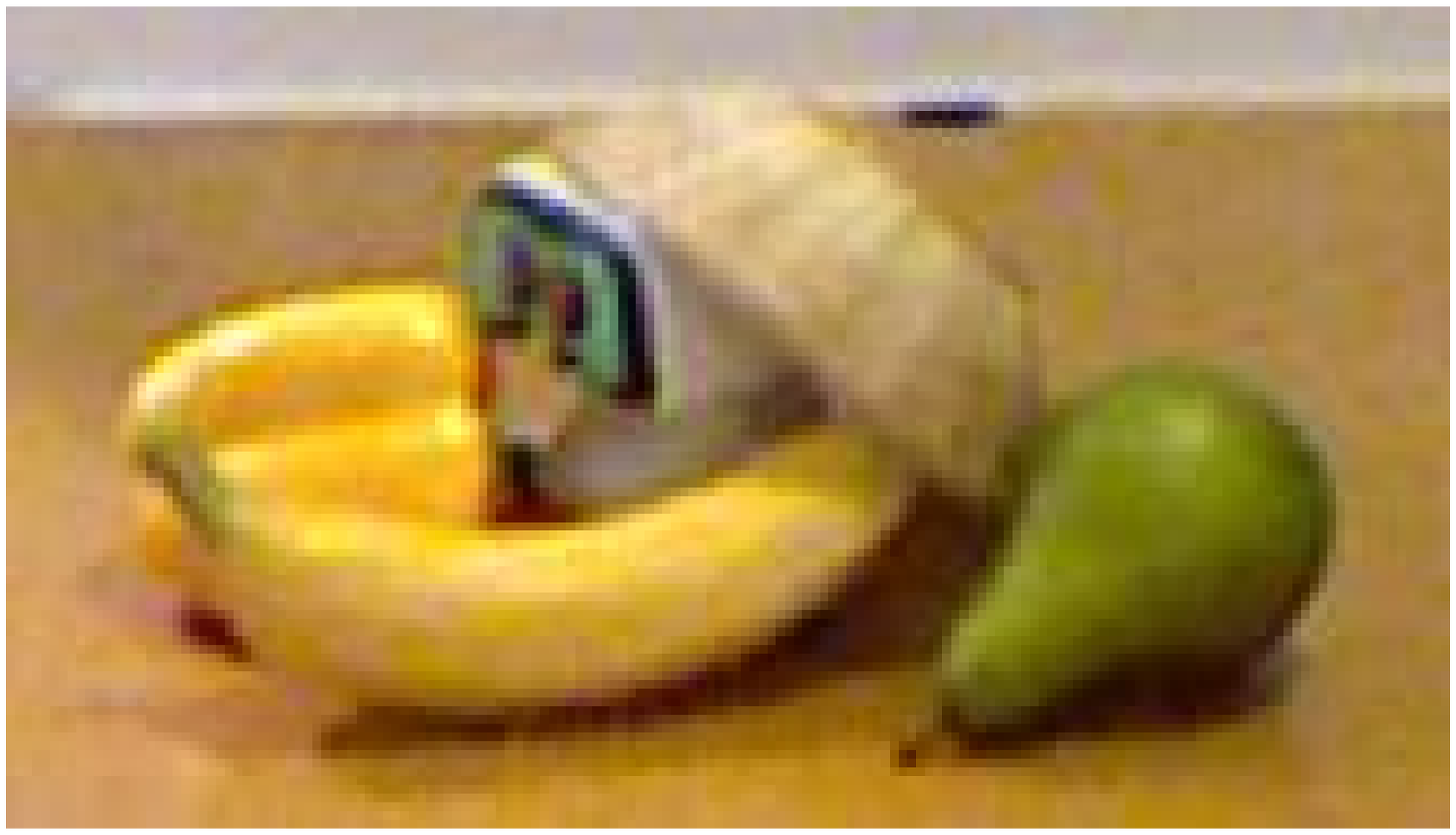}%
  \includegraphics[width=.045\textwidth]{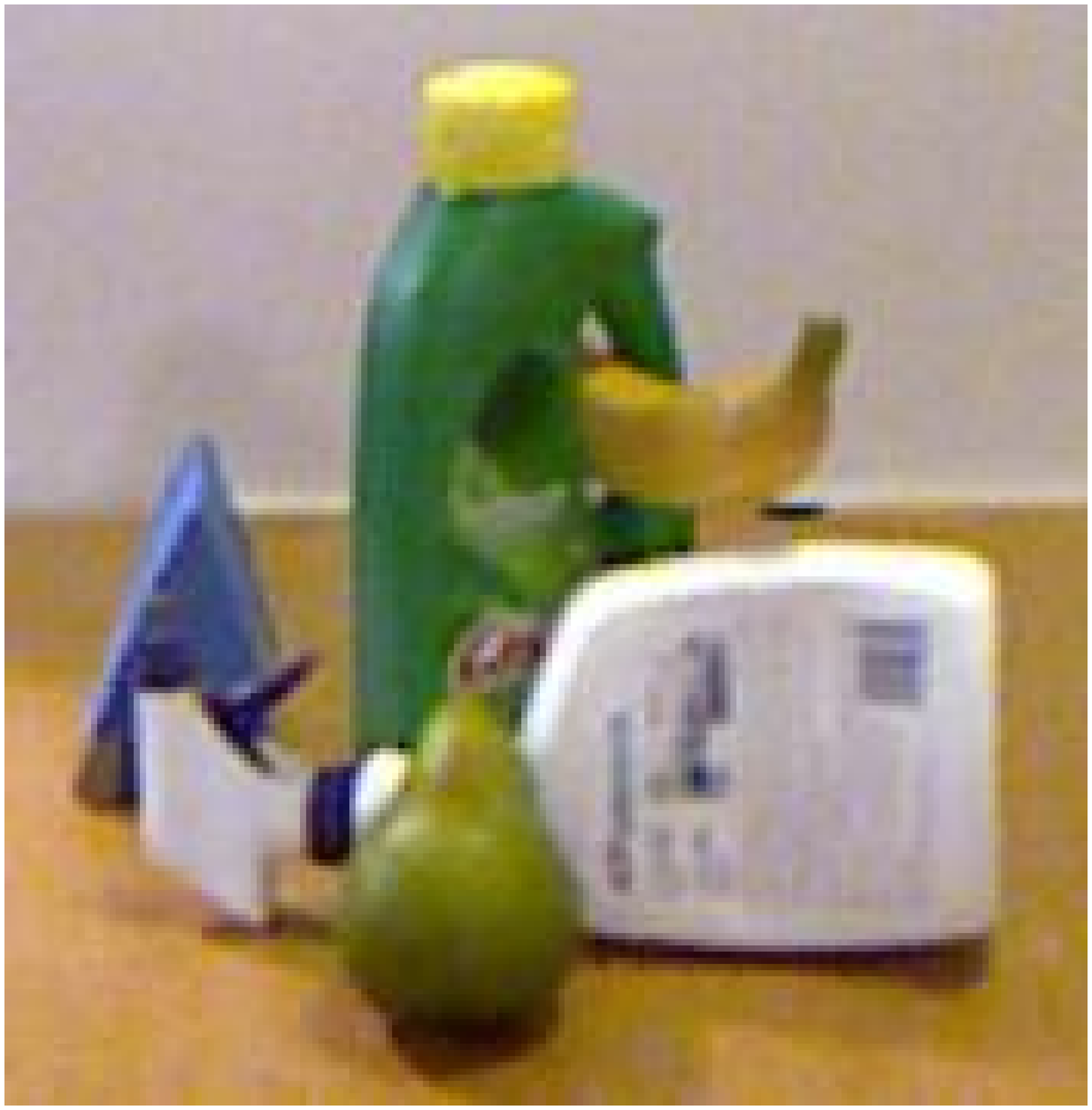}%
  \includegraphics[width=.049\textwidth]{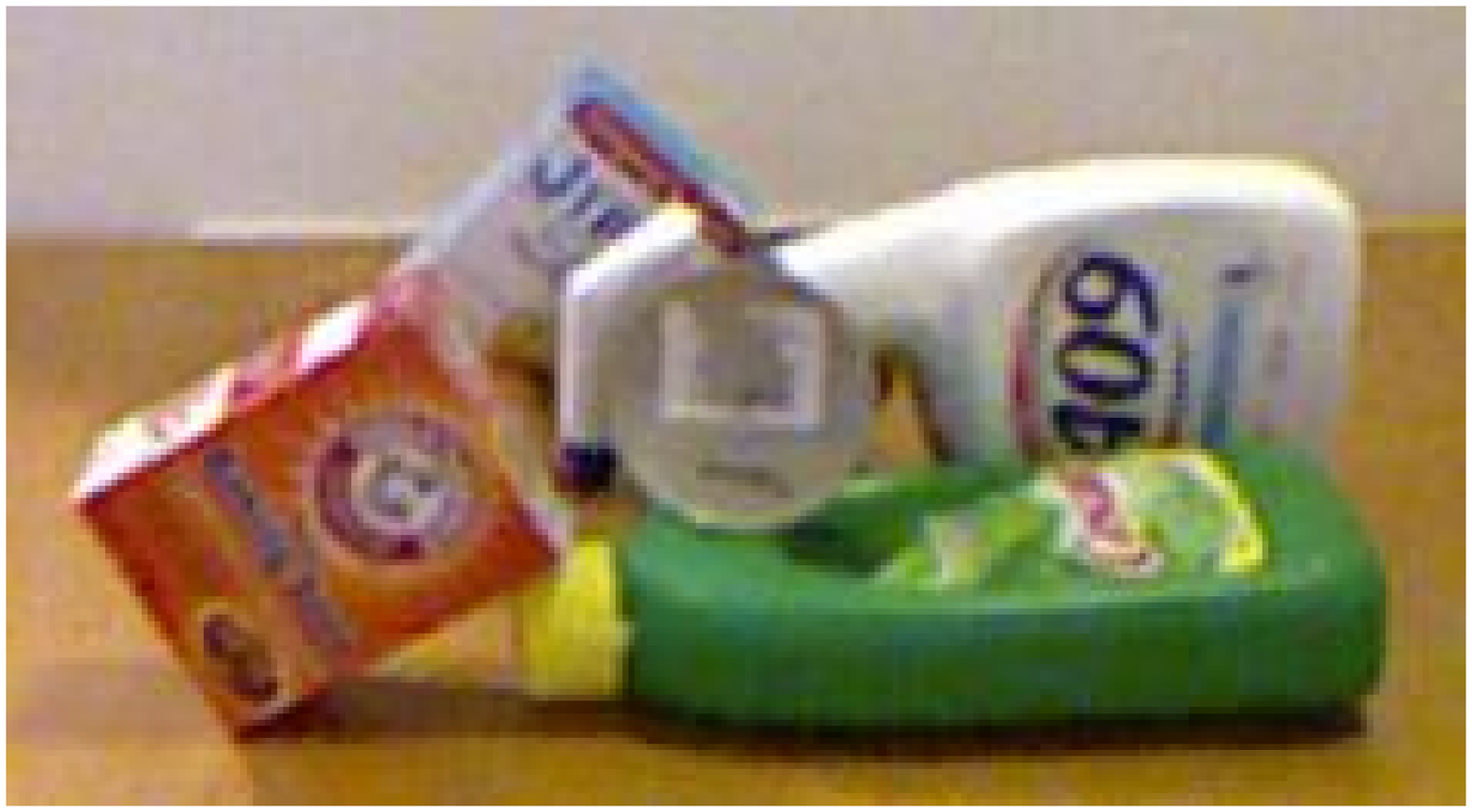}%
  \includegraphics[width=.045\textwidth]{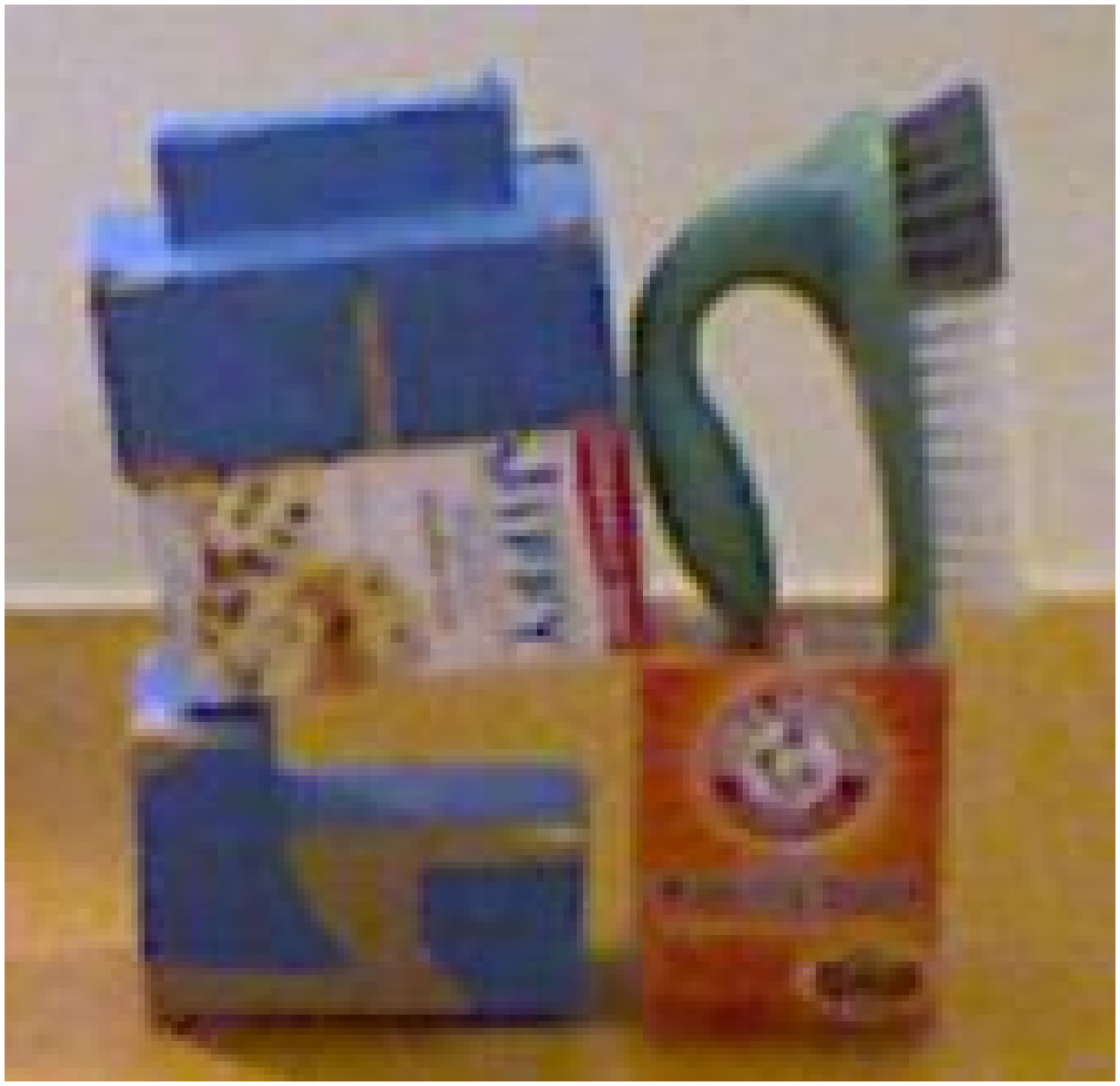}%
  \includegraphics[width=.049\textwidth]{figs/sope_samples/cloud_5_10.ps}%
  \caption{The \textit{Clutter} testing data set.}
  \label{fig:clutter_dataset}
\end{figure}

\section{Experimental Results}

We tested our object detection system on two Kinect-based data sets---the \textit{Kinect} data set from
Aldoma et. al~\cite{aldoma2012global} containing 35 models and 50 scenes, and a new, more difficult data
set with many more occlusions and object pose variations that we collected for this paper which we will
refer to as \textit{Clutter}, which contains 18 models and 30 scenes (Figure~\ref{fig:clutter_dataset}).
We used the same parameters (score component weights, number of samples, etc.) on both data sets.
In table~\ref{table:precision_recall},
we compare the precision and recall of the top scene interpretations (multi-object-placement samples) of our method
against Aldoma et. al on both data sets\footnote{
Since neither ours nor the baseline method uses colors in their object models, we considered a model placement ``correct''
for the \textit{Clutter} data set if it was within a threshold of the correct pose ($2.5cm$, $pi/16$ radians) with respect
to the model's symmetry group.  For example, we don't penalize flipping boxes front-to-back or top-to-bottom, since the
resulting difference in object appearance is purely non-geometric.  For the \textit{Kinect} data set, we used the same correctness
measure as the baseline method (RMSE between model in true pose and estimated pose), with a threshold of $1cm$.}.

\begin{table}[h]
\centering
\begin{tabular}{|c|c|c|c|c|c|c|}
\hline &
\multicolumn{2}{|c|}{\textit{\textbf{this paper (BPA)}}} &
\multicolumn{2}{|c|}{\textit{this paper (ICP)}} &
\multicolumn{2}{|c|}{\textit{Aldoma et. al~\cite{aldoma2012global}}} \\
\cline{2-7}
& precision & recall & precision & recall & precision & recall \\
\hline
\textit{Kinect} & 89.4 & \textbf{86.4} & 71.8 & 71.0 & \textbf{90.9} & 79.5 \\
\textit{Clutter} & \textbf{83.8} & \textbf{73.3} & 73.8 & 63.3 & 82.9 & 64.2 \\
\hline
\end{tabular}
\caption{A comparison of precision and recall.}
\label{table:precision_recall}
\end{table}

\begin{table}[h]
\centering
\begin{tabular}{c|c|c|c|c|c|c|}
\cline{2-7}
\# samples & 1 & 2 & 3 & 5 & 10 & 20 \\
\cline{2-7}
recall & 73.3 & 77.5 & 80.0 & 80.8 & 83.3 & 84.2 \\
\cline{2-7}
\end{tabular}
\caption{Recall on the \textit{Clutter} data set as a function of the number of scene interpretation samples.}
\label{table:multi_mope_recall}
\end{table}

Our algorithm (with BPA) achieves state-of-the art recall performance on both data sets.
When multiple scene interpretations are considered, we achieve even higher recall rates (Table~\ref{table:multi_mope_recall}).
Our precisions are similar to the baseline method (slightly higher on \textit{Clutter}, slightly lower on \textit{Kinect}).
We were unable to train discriminative feature models on the \textit{Kinect} data set, because the original training
scans were not provided.  Training on scenes that are more similar to the cluttered test scenes is also likely to improve
precision on the \textit{Clutter} data set, since each training scan contained only one, fully-visible object.

\subsection{BPA vs. ICP}

We evaluated the benefits of our new alignment method, BPA, in two ways.  First, we compared it to
ICP by replacing the BPA alignment step in round 2 with an ICP alignment step\footnote{Both ICP and BPA
used the same point correspondences; the only difference was that BPA incorporated point feature orientations,
while ICP used only their positions.}.  This resulted in a drop of $10\%$ in both precision and recall on the
\textit{Clutter} data.

For a second test of BPA, we initialized 50 model placements by adding random Gaussian noise to the ground truth poses
for each object in each scene of the \textit{Clutter} data set.  Then, we ran BPA and ICP for 20 iterations on each of
the model placements\footnote{In other words, we repeated the alignment step of round 2 twenty times, regardless of whether the
total score improved.}.
We then computed the average of the minimum pose errors in each alignment trial, where the minimum
at time $t$ in a given trial is computed as the minimum pose error from step $1$ to step $t$.  (The justification for this
measure is that this is approximately what the ``accept if score improves'' step of round 2 is doing.)  As shown in
figure~\ref{fig:bpa_vs_icp}, the pose errors decrease much faster in BPA.

\begin{figure}[h]
  \includegraphics[width=.24\textwidth]{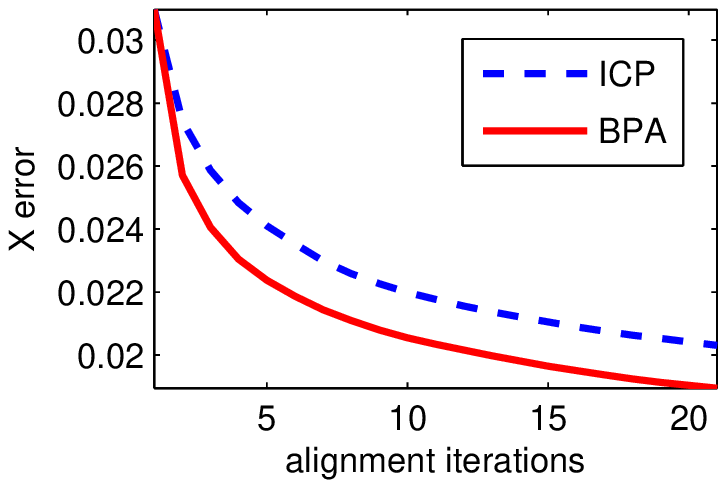}%
  \hfill
  \includegraphics[width=.24\textwidth]{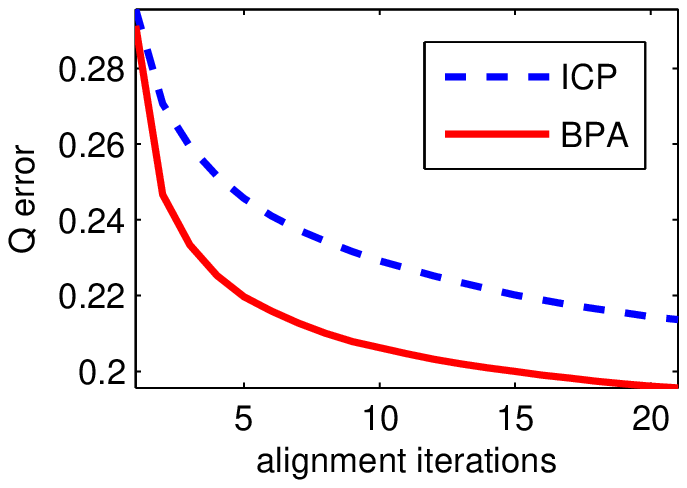}
  \caption{Comparing BPA with ICP.  (Left) The average of the minimum position errors in each alignment trial.
    (Right) The average of the minimum orientation errors in each alignment trial.}
  \label{fig:bpa_vs_icp}
\end{figure}

\section{Related Work}

Since the release of the Kinect in 2010, much progress has been made on 3-D object detection in
cluttered RGB-D scenes.  The two most succesful systems to date are Aldoma et. al~\cite{aldoma2012global}
and Tang et. al~\cite{tang2012textured}.  Aldoma's system is purely geometric, and uses SHOT
features~\cite{tombari2010unique} for model-scene correspondences. It relies heavily on pose
clustering of feature correspondences to suggest model placements\footnote{This is essentially
a sparse version of the Hough transform~\cite{ballard1981generalizing}, which is limited by the
number of visible features on an object, and is why their recall rates tend to be lower than in
our system for objects that are heavily occluded.}.  The main contribution of Aldoma's system is
that they jointly optimize multiple model placements for consistency, which inspired our own
multiple object detection system.

Tang's detection system uses both geometry and image features, and placed first in the
ICRA 2011 Solutions in Perception instance recognition challenge.  Their system relies
heavily on being able to segment objects in the scene from one another, and most of the
effort is spent on combining geometry and image features for classification of
scene segments.  It is unclear how well the system would perform if such segmentations are
not easy to obtain, as is the case in our new \textit{Clutter} data set.

The Bingham distribution was first used for 3-D cluttered object detection in Glover et.
al~\cite{Glover11}.  However, that system was incomplete in that it lacked any alignment step,
and differs greatly from this work because it did not use feature correspondences.

\section{Conclusion and Future Work}

We have presented a system for 3-D cluttered object detection which uses a new
alignment method called Bingham Procrustean Alignment (BPA) to improve detections in highly
cluttered scenes, along with a new RGB-D data set which contains much more clutter and
pose variability than existing data sets.  Our system relies heavily on geometry, and
will clearly benefit from image and color models, such as in Tang et. al~\cite{tang2012textured}.
Our \textit{Clutter} data set, while challenging, contains zero ambiguity, in that a human
could easily detect all of the objects in their correct poses, given enough time to study the models.
An important direction of future work is to handle ambiguous scenes, where the parts of objects that are
visible are insufficient to perform unique alignments, and instead one ought to
return distributions over possible model poses.  In early experiments we have performed
on this problem, the Bingham distribution has been a useful tool for representing orientation ambiguity.




\section*{APPENDIX}

\textbf{The Bingham Distribution.}
The Bingham distribution is commonly used to represent uncertainty on 3-D rotations
(in unit quaternion form)~\cite{antone2001robust, bingham_antipodally_1974, Glover11}.  For quaternions,
its density function (PDF) is given by
\begin{equation} \label{eq:bingham_pdf}
p(\vec{q}; \Lambda,V) = \frac{1}{F} \exp \left\{ \sum_{i=1}^3 \lambda_i (\vec{v_i}^T \vec{q})^2 \right\}
\end{equation}
where $F$ is a normalizing constant so that the distribution integrates to one over the surface
of the unit hypersphere $\S3$, the $\lambda$'s are non-positive ($\leq 0$) concentration parameters,
and the $\vec{v_i}$'s are orthogonal direction vectors.

\textbf{Product of Bingham PDFs.}
The product of two Bingham PDFs is given by adding their exponents:
\begin{equation} \label{eq:bingham_mult}
\begin{split}
f(\vec{q}; &\Lambda_1,V_1) f(\vec{q};\Lambda_2,V_2) \\
&= \frac{1}{F_1 F_2} \exp \left\{ \vec{q}^T ( \sum_{i=1}^3 \lambda_{1i} \vec{v_{1i}} \vec{v_{1i}}^T +
\lambda_{2i} \vec{v_{2i}} \vec{v_{2i}}^T ) \vec{q} \right\} \\
&= \frac{1}{F_1 F_2} \exp \left\{ \vec{q}^T (C_1 + C_2) \vec{q} \right\}
\end{split}
\end{equation}
After computing the sum $C = C_1 + C_2$ in the exponent of
equation~\ref{eq:bingham_mult}, we compute the eigenvectors and eigenvalues of $C$, and then
subtract off the lowest magnitude eigenvalue from each spectral
component, so that only the eigenvectors corresponding to the largest
$3$ eigenvalues (in magnitude) are kept, and $\lambda_i \leq 0 \; \forall i$ (as in
equation~\ref{eq:bingham_pdf}).  We use the open-source Bingham Statistics
Library\footnote{http://code.google.com/p/bingham} to look up the normalization constant.

\textbf{Estimating the Uncertainty on Feature Orientations.}
When we extract surface features from depth images, we estimate their 3-D orientations
from their normals and principal curvature directions by computing the rotation matrix
$R = \left[ \vec{n} \;\; \vec{p} \;\; \vec{p'} \right]$, where $\vec{n}$ is the normal vector,
$\vec{p}$ is the principal curvature vector, and $\vec{p'}$ is the cross product of 
$\vec{n}$ and $\vec{p}$.  We take the quaternion associated with this rotation matrix
to be the feature's estimated orientation.

These orientation estimates may be incredibly noisy, not only due to typical sensing noise, but
because on a flat surface patch the principal curvature direction is undefined and will be chosen completely
at random.  Therefore it is extremely useful to have an estimate of the uncertainty on each feature orientation
that allows for the uncertainty on the normal direction to differ from the uncertainty on the principal curvature
direction.  Luckily, the Bingham distribution is well suited for this task.

To form such a Bingham distribution, we take the quaternion associated with $R$ to be the mode of the distribution,
which is orthogonal to all the $\vec{v_i}$ vectors.  Then, we set $\vec{v_3}$ to be the quaternion associated with
$R' = \left[ \vec{n} \;\; -\vec{p} \;\; -\vec{p'} \right]$, which has the same normal as the mode, but reversed principal
curvature direction.  In quaternion form, reversing the principal curvature is equivalent to the mapping:
\[ (q_1, q_2, q_3, q_4) \;\; \rightarrow \;\; (-q_2, q_1, q_4, -q_3) \;\;.\]
We then take $\vec{v_1}$ and $\vec{v_2}$ to be unit vectors orthogonal to the mode and $\vec{v_3}$ (and each other).
Given these $\vec{v_i}$'s, the concentration parameters $\lambda_1$ and $\lambda_2$ penalize deviations in the normal vector,
while $\lambda_3$ penalizes deviations in the principal curvature direction.  Therefore,
we set $\lambda_1 = \lambda_2 = \kappa$ (we use $\kappa = -100$ in all our experiments in this paper), and we use the heuristic
$\lambda_3 = \max \{ 10(1 - c_1/c_2), \kappa \}$, where $c_1/c_2$ is the ratio of the principal curvature eigenvalues\footnote{
The principal curvature direction is computed with an eigen-decomposition of the covariance of normal vectors in a neighborhood
about the feature.}.  When the surface is completely flat, $c_1 = c_2$ and $\lambda_3 = 0$, so the resulting Bingham
distribution will be completely uniform in the principal curvature direction.  When the surface is highly curved, $c_1 \gg c_2$,
so $\lambda_3$ will equal $\kappa$, and deviations in the principal curvature will be penalized just as much as deviations in
the normal.





\bibliographystyle{plain}
\bibliography{paper}

\end{document}